\documentclass[twoside,11pt]{article}
\usepackage{jmlr2e}
\pdfoutput=1
\usepackage{amsmath}
\usepackage{cases,bm}
\usepackage{amsfonts}
\usepackage{amssymb}
\usepackage{algorithm,algorithmic}
\usepackage{multirow}
\usepackage{array}
\usepackage{subfigure}
\usepackage{colortbl,dcolumn}  % rowcolor
\usepackage{epstopdf}
\usepackage{textcomp,booktabs}

\usepackage{hyperref}
\hypersetup{colorlinks, linkcolor=blue, anchorcolor=blue, citecolor=blue}

\usepackage{makecell} % for Xline
\newcommand{\tabincell}[2]{\begin{tabular}{@{}#1@{}}#2\end{tabular}}
\newcommand\mgape[1]{\gape{$\vcenter{\hbox{#1}}$}}

\newtheorem{assumption}{\hspace{0pt}\bf Assumption}

\definecolor{mygray}{gray}{.9}
\definecolor{mypink}{rgb}{.99,.91,.95}
\definecolor{mycyan}{cmyk}{.3,0,0,0}
\definecolor{blush}{rgb}{0.87, 0.36, 0.51}
\definecolor{bluebell}{rgb}{0.64, 0.64, 0.82}
\definecolor{frenchlilac}{rgb}{0.53, 0.38, 0.56}

\begin{document}

\title{Machine Discovery of Partial Differential Equations from Spatiotemporal Data}

\author{\name Ye Yuan \\
        \addr School of Artificial Intelligence and Automation, State Key Laboratory of Digital Manufacturing Equipments and Technology, Huazhong University of Science and Technology, Wuhan 430074, P.R.~China\\
        \name Junlin Li \\
            \addr School of Artificial Intelligence and Automation,
         Key Laboratory of Image Processing and Intelligent Control,
        Huazhong University of Science and Technology, Wuhan 430074, P.R.~China\\
        \name Liang Li \\
            \addr School of Artificial Intelligence and Automation,
         Key Laboratory of Image Processing and Intelligent Control,
        Huazhong University of Science and Technology, Wuhan 430074, P.R.~China\\
         \name Frank Jiang  \\
        \addr School of Electrical Engineering and Computer Science,
        KTH Royal Institute of Technology, Stockholm 11428, Sweden \\
        \name Xiuchuan Tang \\
        \addr  School of Mechanical Science and Engineering, Huazhong University of Science and Technology, Wuhan 430074, P.R.~China  \\
        \name Fumin Zhang  \\
        \addr School of Electrical and Computer Engineering,
        Georgia Institute of Technology, Atlanta, 30332, USA \\
        \name Sheng Liu \\
        \addr School of Power and Mechanical Engineering, Wuhan University,
        Wuhan 430074, P.R. China \\
        \name Jorge Goncalves \\
        \addr Department of Engineering, University of Cambridge, Cambridge, CB2 1PZ, UK and the Luxembourg Centre for Systems Biomedicine, University of Luxembourg, Belvaux, L-4367, Luxembourg \\
        \name Henning U.Voss \\
        \addr Citigroup Biomedical Imaging Center, Weill Cornell Medical College, New York, 10021, USA \\
                       \name Xiuting Li \email xtingli@hust.edu.cn \\
            \addr School of Artificial Intelligence and Automation,
         Key Laboratory of Image Processing and Intelligent Control,
        Huazhong University of Science and Technology, Wuhan 430074, P.R.~China\\
        \name J\"{u}rgen Kurths \email kurths@pik-potsdam.de \\
        \addr Research Domain IV - Transdisciplinary Concepts \& Methods, Potsdam Institute for Climate Impact Research, Potsdam, 14473, Germany\\
         \name Han Ding  \\
        \addr School of Mechanical Science and Engineering, State Key Laboratory of Digital Manufacturing Equipments and Technology, Huazhong University of Science and Technology, Wuhan 430074, P.R.~China }

\editor{~}

\hyphenation{dictionary}
\maketitle
\clearpage
\begin{abstract}%
%Conventionally, researchers in physics discover partial differential equations (PDE) from  observed phenomena by combining existing physical laws that match collected measurement data. The history shows that this process of discovering and establishing physical laws in form of PDEs is arduous and often requires decades or even centuries of trial and error.
The study presents a general framework for discovering underlying Partial Differential Equations (PDEs) using measured spatiotemporal data. The method, called Sparse Spatiotemporal System Discovery ($\text{S}^3\text{d}$), decides which physical terms are necessary and which can be removed (because they are physically negligible in the sense that they do not affect the dynamics too much) from a pool of candidate functions. The method is built on the recent development of Sparse Bayesian Learning; which enforces the sparsity in the to-be-identified PDEs, and therefore can balance the model complexity and fitting error with theoretical guarantees. Without leveraging prior knowledge or assumptions in the discovery process, we use an automated approach to discover ten types of PDEs, including the famous Navier-Stokes and sine-Gordon equations, from simulation data alone. Moreover, we demonstrate our data-driven discovery process with the Complex Ginzburg-Landau Equation (CGLE) using data measured from a traveling-wave convection experiment. Our machine discovery approach presents solutions that has the potential to inspire, support and assist physicists for the establishment of physical laws from measured spatiotemporal data, especially in notorious fields that are often too complex to allow a straightforward establishment of physical law, such as biophysics, fluid dynamics, neuroscience or nonlinear optics.
\end{abstract}


\section{Introduction}
The typical process of discovering physical laws involves collecting measurement data on physical phenomena and then using them to reconstruct and validate the synthesis of related, previously discovered physical laws. This foundational discovery approach is fueled by and leverages the ingenuity of physicists researching the phenomena. The physics community's careful and rigorous proposal and validation of physical law has successfully led to great achievements and progress towards understanding the physics of the universe \citep{Ara2002,KdV,SL, FN-1}. However, this process typically needs a rather long time. This work proposes a discovery tool that may help accelerate this discovery process, especially in the presence of complex dynamics that would otherwise take needlessly significant effort to begin uncovering and understanding \citep{Aqu2002, Aqu2009, Gre1982, Lai2011, Aih1990}. The tool automatically synthesizes the available physical features which dominate complex dynamics from data, and
incorporate them into model equations which enable physicists to study them both analytically and numerically.

Data-driven discovery of physical laws allows computational methods to learn representations from data. There have been reports of many successes using data-driven methods such as learning a nonlinear mapping from time series and predict  chaotic dynamics \citep{Far1987}, predicting catastrophes in nonlinear systems \citep{Wan2011}, equation-free modeling multi-scale/complex systems \citep{Kevrekidis2003},
inferring models of nonlinear coupled dynamical systems  \citep{Bon2007},
reconstructing networks with stochastic dynamical processes \citep{She2014}, extracting Koopman modes from data \citep{mezic2013analysis,WilliamsEDMD2015, tu2014dynamic,giannakis2017data}. In addition to these studies, other methods, such as symbolic regression \citep{Schmidt}, neural-network based tools \citep{Raissi1,Raissi3,raissi2017physics,raissi2018multistep}, make an effective attempt to data-driven discovery, though they have good prediction power, yet the learned models are hard to offer physical insights.  Thus, it is desirable to develop a general framework that specifically discovers a reasonably interpretable model directly from data.
{In the system identification community, a standard way is to transform the model discovery problem into a regression problem and design a good optimization algorithm to solve the formulated problem \citep{Ljung:1986:SIT:21413,lindquist2015linear}. More precisely, suppose the dynamics of the to-be-discovered system can be described by a partial differential equation (PDE) model of the general form:
\begin{eqnarray}\label{eq:PDE}
u_t(x,t) &=& F(x,t, u,u_x,u_{xx},u_{xxx},\ldots; \Theta),
\end{eqnarray}
where $u(x,t)$ is a state variable with respect to time $t$ and space $x\in\mathbb{R}$, $\Theta$ is
a parameter vector and F is an unknown function,  which depends on  $u$ and its partial derivatives.  To infer an accurate and interpretable representation of $F$ using measurement data of the state variable $u$,
a (large) set of dictionary with many  potential candidate terms; for instance, $\{u, u_x, u_{xx}, u^2,uu_x,uu_{xx},u_xu_{xx}, \ldots\}\triangleq \{\Phi_1,\ldots,\Phi_m\}$ (with cardinality $m$), will be created to represent $F$ via
$
F=\sum\limits_{i=1}^{m}\theta_i\Phi_i
$
and the values of the parameters/weights $\{\theta_i\}_{i=1}^{m}$ will be estimated by solving the following regression problem
\begin{eqnarray}\label{eq:linearregress}
u_t&=& \Phi\theta+\nu,
\end{eqnarray}
where $\Phi\in \mathbb{R}^{n\times m}$ is the dictionary matrix, $u_t\in \mathbb{R}^{n}$ is the time  derivatives of $u$, $\theta\in \mathbb{R}^{m}$ is the unknown vector to be learned and $\nu\in \mathbb{R}^{n}$ is the modelling error.

There are several numerical algorithms to find solutions to problems like Eq.~\eqref{eq:linearregress}. The first kind of such methods builds on the ordinary least squares regression. Examples include orthogonal reduction \citep{Xu,khanmohamadi2009spatiotemporal},
the backward elimination scheme \citep{Bar, Bill1}, least squares estimator \citep{muller2004parameter} and the alternating conditional expectation algorithm \citep{breiman1985estimating, Voss, voss1998identification}. These approaches do have an good fitting accuracy as described by \citep{tibshirani2015statistical} which, however, have two major issues: prediction accuracy and interpretability in modeling a wider range of data.
The second type of methods is the sparse regression approach, where a sparse priori assumption of underlying governing equations is incorporated  to tackle the outstanding issue, overfitting. The work of
\citep{brunton2016discovering} apply a $L_0$-norm penalty into the estimation problem of Eq.~\eqref{eq:linearregress} that reflects sparse priori. The resulting sparse identification of nonlinear dynamics (SINDy) algorithm selects the key feature terms
based on the value of least squares regression coefficients by a thresholding step \citep{mangan2016inferring,zhang2018convergence}.  This SINDy framework is recently extended to recover PDEs from spatiotemporal data \citep{Rudy} where not only time/spatial derivatives are estimated from noisy data but also
a modified algorithm, named by a Sequential Threshold Ridge regression (STRidge) algorithm, is developed to improve the robustness of the discovery. Instead of penalizing the $L_0$ norm of $\theta$, \citep{Schae2017} used lasso algorithm, a method that combines the least-squares loss with an $L_1$-constraint regularization of the coefficients, to learn active features and estimate coefficients in the underlying PDEs. In both $L_0$ and $L_1$ schemes,  the identified PDEs model is parsimonious and accurate, as the penalty term explicitly encourages the solution with few nonzero components. However, it is well known that solving the $L_0$ problem is NP-hard and hardly tractable with the number of features large due to the nonconvexity of the $L_0$ norm. The $L_1$ norm can provide a good approximation (convex relaxation) to the $L_0$ norm, but the resulting Lasso tends to over-shrink large coefficients, which leads to biased estimates \citep{fan2001variable}.

In this paper, the estimation problem in Eq. \eqref{eq:linearregress} with observations generated by spatiotemporal systems is solved in a Sparse Bayesian Learning (SBL) framework. Proposed in \citep{tipping2001sparse} and further investigated in \citep{faul2002analysis,rao2003subset,palmer2006variational}, the SBL framework treats the parameter $\theta$ as a random variable with sparse-inducing  priori distributions  determined by a set of hyperparameters. It has been shown to be more general and powerful in finding maximally sparse solution. For instance, the maximum a posteriori (MAP) estimation with a fixed weight priori is its special case \citep{rao2003subset}; the Lasso problem can be interpreted as the explicit MAP with Laplace priori \citep{figueiredo2002adaptive}; the $L_0$-norm  can be approached by a priori of the form: $p(\theta)\sim \exp(-\sum\limits_{i=1}^{m}|\theta_i|^p)$ with $p\rightarrow 0$ \citep{wipf2004sparse}. Specially, \citep{wipf2011latent}  establishes the equivalence of the SBL cost function and the objectives in canonical MAP-based sparse methods (e.g., the Lasso) via a dual-space analysis and shows that the SBL method maintains many desirable advantages over all possible MAP methods, such as using non-factorial coefficient priors  \citep{wipf2008new}, treating the  reconstruction problem with unfavorable restricted isometry
properties (RIP) \citep{wu2012dual}, implementing non-separable alternative via either $L_2$ or $L_1$ reweighting \citep{wipf2010iterative,pan2012reconstruction,chang2012data,wei,pan2018identification}. Recently, \citep{zhang2018robust} explores the advantage of Bayesian inference to provide error bars and robust discovery of governing physical laws from noisy data. However, this relatively new paradigm is still hindered by lack of the theoretical analysis and an effective optimization algorithm  which addresses issues brought by the SBL cost function.

Our main purpose in this paper is to present an effective SBL framework for locally recovering unknown PDEs from
spatiotemporal data and conduct some theoretical analysis on the sparse discovery for spatiotemporal systems.
For brevity, we refer to the proposed discovery approach as  the Sparse
Spatiotemporal System Discovery ($\text{S}^3\text{d}$) method.
The key contributions of the $\text{S}^3\text{d}$ framework include the following. First of all, we take a new perspective of the SBL algorithm. Specially, the existing methods optimizing the SBL cost function, such as the Expectation-Maximization \citep{wipf2004sparse} and MacKay update rules \citep{tipping2001sparse}, are expressed in either the coefficient space or  the latent space. Here we take the cost function expressed concurrently in the two space. The resulting difference of convex function (DC) programming is then solved by an iterative convex weighted lasso algorithm with concave-convex procedure (CCP). Secondly, many prototypical PDEs including the Navier-Stokes equation in fluid mechanics~\citep{NS}, the FitzHugh-Nagumo for nerve conduction~\citep{FN-1}, and the Schr$\ddot{\text{o}}$dinger equation in quantum mechanics~\citep{SL} are rediscovered from only spatiotemporal data, which cast off a great deal of trial and error.
Finally, we apply the proposed $\text{S}^3\text{d}$ method to
real measurements from a traveling-wave convection experiment with the discovery of the normal Complex Ginzburg-Landau Equation (CGLE) model. The numerical solution to the CGLE commendably reconstructs the observed experimental phenomenon.

Theoretically, the primary concern is the convergence and consistency of the sparse estimators with measurement data generated by spatiotemporal systems. The convergence and consistency analysis are conducted under the SBL cost function expressed jointly in the coefficient space and the latent variable space. We show that under reasonable assumptions on the optimization path in joint space, with the CCP, the SBL cost function has a local minimum and can produce the maximally sparse solution. When the cost function is separable,  we show that the DC programming formulation succeeds in recovering sparsity pattern and reaching small error with high probability.

The rest of this paper is organized as follows:
In the next subsection we summarize the notation and conventions used in this paper.
In Section~\ref{sec:S3d}, we introduce a bayesian framework for the problem of discovering PDEs from spatiotemporal data and estimation of coefficients. Our main results on the convergence and consistency of the parameter estimator are presented in Section~\ref{sec3},
with the proof included in the Appendix.
We then evaluate our method experimentally in Section~\ref{sec:experiment} with both synthetic data and real-life
data, which confirms the exact sparse recovery  of the theoretical results.
The comparison with the state-of-the-art algorithms are also provided in this section.
We conclude this paper with a discussion in Section~\ref{sec:conclusion}.}

\subsection{Notation}
Throughout this paper, we denote vectors by lowercase letters, e.g., $x$, and matrices by uppercase letters, e.g., $X$. For  a vector $x$ and a matrix $X$, we
denote its transpose by $x^{T}$ and $X^{T}$, respectively. Furthermore, $x_i$ refers to its $i$-th element  of $x$.
$\|x\|_2=\sqrt{x^T x}$ and
$\| x\|_1=\sum_{\ell} |x_{\ell}|$,
denotes its Euclidean norm and
$\ell_1$-norm.
Let $u:\Omega \rightarrow \mathbb{R}, x\in \Omega \subset \mathbb{R}^n$,
$\frac{\partial u}{\partial x_i}=\lim_{h\rightarrow 0}\frac{u(x_i+h)-u(x_i)}{h}$, provide this limit exists.
We usually write $u_t$ for $\partial u / \partial t$, and similarly $u_{x_i} \triangleq \partial u /\partial x_i$, $u_{x_ix_i} \triangleq \partial^2 u /\partial x_i^2,$
  $u_{x_ix_j} \triangleq \partial^2 u /\partial x_i \partial x_j$, etc..

\section{$\text{S}^3\text{d}$: A Bayesian Framework to Discover PDEs from Spatiotemporal Data}\label{sec:S3d}
{This section proposes a framework of the $\text{S}^3\text{d}$ method for data-driven discovery of PDEs. The proposed $\text{S}^3\text{d}$ framework
%is illustrated in Fig.~\ref{fig2}, which
includes the following three steps: Data collection, Model specification and SBL. We will consider the dynamics of a spatiotemporal system which is governed by a PDE of the
general form:
\begin{eqnarray}\label{eq:PDE}
\frac{\partial u(x,t)}{\partial t}= F\left(x,t,u(x,t),\frac{\partial u(x,t)}{\partial x_1},\cdots,\frac{\partial u(x,t)}{\partial x_n},\frac{\partial ^2 u(x,t)}{\partial x_1^2},\cdots,\frac{\partial ^2 u(x,t)}{\partial x_1\partial x_n},\cdots, \Theta\right),
\end{eqnarray}
where the dynamical variable, $u$, is $N$-component with temporal and  multi-dimensional spatial variables, $t$ and $x=(x_1,\ldots,x_n)^T \in \mathbb{R}^n$, respectively, defining the state. $\Theta$ is a parameter vector and $F$ is an unknown, $N$-dimensional function, which depends on the dynamical variable $u$ and its derivatives.
%\begin{figure}[!h]
%\includegraphics[scale=0.45]{test.pdf}
%\caption{{\bf Schematics of the $\text{S}^3\text{d}$ method used to infer the CGLE equation from data.}
%\textbf{a}  Data collection. Data is collected from experiments of a real system (or from a numerical solution to a PDE). The collected data is organized into a snapshot matrix;
%\textbf{b}  Model specification. A large library of candidate terms including the potential features for the to-be-identified PDE is constructed. The candidate terms include the state $u$, its spatial derivatives up to 2-order and nonlinear coupling terms between them. The time derivative $u_t$ is regarded as the model output.  The collected data is used  to approximate the time derivative $u_t$ of the state $u$ and all the candidate terms.
%A set of linear equations (i.e., $y=\Phi \theta$) is then defined to represent the underlying PDE;
%\textbf{c}  Sparse Bayesian Learning. Sparse regression problem is solved using a iterative reweighted $\ell_1$-minimization algorithm. } \label{fig2}
%\end{figure}
}

\subsection{Model Specification}\label{sec:sparseregress}
In this section, we shall restrict the discussion to the case $N=1$ for notational simplicity, yet the proposed framework can be applied to general cases when $N>1$. We construct a large library of candidate terms that may appear in the function $F$ . We mainly take the model structure to be a $q^{th}$-order PDE:
\begin{eqnarray*}
u_t=\sum\limits_{i=1}^{m}\theta_i f_i\left(u^{p_1}\left(\frac{\partial u}{\partial x}\right)^{p_2}\cdots\left(\frac{\partial^q u}{\partial x^q}\right)^{p_{n+1}}\right),
\end{eqnarray*}
where $q\in \mathbb{N}^{+}$ represents the $n^{th}$-order PDE, $p=\sum\limits_{l=1
}^{n+1}p_l$ represents the $p$-th operator with $p_l\in \mathbb{Z}^{+}, \text{where } l=1,2,\ldots,n+1$ is the degree of nonlinearity, $f_i(\cdot)(i=1,2,\ldots,m)$ is an identical mapping and the term $u_t$ represents the output. The chosen model structure can pack many nonlinearities (e.g., polynomial $u^2, ~uu^2$, etc.), partial derivatives (e.g., $u_x,~u_{xx},~u_{xxx}$, etc.), their combinations (e.g., $u^2u_x,~u^3u_{xx}$, etc.). Some physics-based candidate terms (such as the reaction term $u_x-u_{xxx}$, the diffusion term $u(u-u_{xx})_x$, the convection term $u_x(u-u_{xx})$, etc.) also are included in our function library if needed. The PDE with the dictionary can be rewritten into the following regression form:
\begin{eqnarray}\label{eq-EQ}
u_{t}&=&
\left[ \begin{array}{cccccccccccc}
1&u&u_x&u_{xx}&u_{xxx}&u^2&uu_x&\cdots&u^2u_{xxx}\\
 \end{array} \right]\cdot \theta,
\end{eqnarray}
where we chose the order $q=3$ and the degree $p=3$ for clarity.  Eq.~\eqref{eq-EQ} then holds at all points in the domain of interest.

Now, we construct a feature matrix $\Phi$. Each column of $\Phi$ is a compilation of all the values of a specific candidate function on the right side of Eq.~\eqref{eq-EQ} over all grid points, and the output vector, $y$, contains the values of $u_t$ in the left side of Eq.~\eqref{eq-EQ} over all grid points. We state the set of linear equations representing our PDE as follows:
\begin{eqnarray}\label{regre}
\underset{y}{\underbrace{\left[ \begin{array}{cccccccc}
u_{t}(x_1,t_1)\\
u_{t}(x_2,t_1)\\
\vdots\\
%u_{t}(x_{n-1},t_m)\\
u_{t}(x_{n_x},t_{n_t})\\
 \end{array}
 \right]}}&=&
\underset{\text{Dictionary~matrix}~ \Phi}{\underbrace{ \left[ \begin{array}{cccccccc}
1&u(x_1,t_1)&u_x(x_1,t_1)&\cdots&u^2u_{xxx}(x_1,t_1)\\
1&u(x_2,t_1)&u_x(x_2,t_1)&\cdots&u^2u_{xxx}(x_2,t_1)\\
\vdots&\vdots&\vdots&\cdots&\vdots\\
%1&u(x_{n-1},t_m)&u_x(x_{n-1},t_m)&\cdots&u^5u_{xxx}(x_{n-1},t_m)&u|u|^2(x_{n-1},t_m)\\
1&u(x_{n_x},t_{n_t})&u_x(x_{n_x},t_{n_t})&\cdots&u^2u_{xxx}(x_{n_x},t_{n_t})\\
 \end{array} \right]}}
\underset{\theta}{\underbrace{ \left[ \begin{array}{cccccccc}
\theta_1\\
\theta_2\\
\vdots\\
%\\
\vdots\\
 \end{array}
 \right]}}.
\end{eqnarray}
To find the unknown coefficients $\theta$, we need to compute the time (and spatial) derivatives of $u$ and solve Eq.~\eqref{regre}.
Various numerical methods including the numerical differentiation \citep{FDM, Lichen, spec}, the symbolic or automatic differentiation \citep{baydin2018automatic} can be used for
the estimation of the derivatives in the dictionary matrix $\Phi$. In our framerwork, the polynomial approximation method is selected to estimate the derivatives from the noise-contained measurements. Indeed, the polynomial approximation method has been proved to be capable of effectively estimating the derivatives from the noise-contained data in \citep{Rudy}.  In Appendix~\ref{sec:Polyapp}, we explain the polynomial approximation method in detail.

A general solution to Eq.~\eqref{regre} can be obtained by using the least-squares method. For example, $\hat{\theta}$ is the minimizer of the following optimization problem:
\begin{eqnarray}\label{eq0}
\hat{\theta}=\arg \min_{\theta}\|y-\Phi \theta\|_2^2.
\end{eqnarray}
The solution, $\hat{\theta}$, is uniquely identified only if the number of grid points, $n=n_xn_t$, is larger than the number of columns in $\Phi$ and the matrix, $\Phi$, has full column rank. However, these two conditions, in general, are not always satisfied. In addition, if the dictionary is over-determined, the resulting optimization easily leads to overfitting.

\subsection{Sparse Bayesian Learning for $\text{S}^3\text{d}$}\label{subsec:SBL}
{This section proposes a SBL algorithm based CCP to solve Eq.~\eqref{regre} defined by $\text{S}^3\text{d}$. The SBL, initially introduced to solve regression and classification problems in seminal works  \citep{tipping2001sparse,wipf2004sparse},
performs parameter estimation via evidence maximization or type-$\Pi$ maximum likelihood \citep{tipping2000relevance}. From the SBL procedure, it can be seen that the SBL cost function can lead to sparse representations.
\cite{wipf2004sparse} employed the Expectation-Maximization (EM) algorithm to minimize the SBL cost function.
However, EM update rule has non-ideal convergence properties. Further, \cite{wipf2010iterative} proposed a re-weighted $L_1$ algorithm derived from the duality theory. Here,  a similar algorithm will be derived from a new perspective such that the convergence of the algorithm can be analyzed in a simpler way. We will use this algorithm to solve Eq.~\eqref{regre}.
Note that all the candidate terms in the dictionary matrix $\Phi$ and the output $y$ are approximated using the data by numerical methods.
We consider the following matrix form with the model error $\nu$,
\begin{eqnarray}\label{eq:sparserecovery}
y&=& \Phi\theta+\nu,
\end{eqnarray}
where the noise vector  $\nu$ is assumed to be i.i.d. Gaussian distributed with zero mean: $\mathcal{N}(0,\sigma^2I_n)$
with $I_n$ an $n\times n$ identity matrix, and $\theta$ is assumed to be sparse. Then the likelihood of the output $y$ given the coefficients $\theta$ is
\begin{equation}\label{eq1}
  p(y|\theta)=(2\pi\sigma^2)^{-\frac {n}{2}}\exp\left(-\frac {1}{2\sigma^2}\|y-\Phi\theta\|^2\right).
\end{equation}
To enforce the sparsity, we assume the following form prior distribution
\begin{equation}\label{eq2}
 p(\theta;\gamma)=\prod_{i=1}^{m}(2\pi\gamma_i)^{-\frac {1}{2}}\exp\left( -\frac {\theta_i^2}{2\gamma_i}\right),
\end{equation}
where $\gamma:=[\gamma_1,\gamma_2,\cdots,\gamma_m]^T$ is a vector of $m$ hyperparameters. Although the hyperparametric vector $\gamma$ is unknown, it determines the variance of each entry of $\theta$ and ultimately produces sparsity properties. In particular, if $\gamma_i=0$, one can get $\mathbb{P}(\theta_i=0)=1$ immediately. Moreover,
 the hyperparametric vector can be estimated by type-$\Pi$
 %from the data  marginalizing over the coefficients and then performing
maximum likelihood. To this end, we first compute the marginalized probability distribution
\begin{align}\notag
% \nonumber to remove numbering (before each equation)
  p(y;\gamma)&=\displaystyle{\int p(y|\theta)p(\theta;\gamma)d\theta} =(2\pi)^{-\frac{n}{2}}|\Sigma_y|^{-\frac {1}{2}}\exp\left(-\frac {1}{2}y^T\Sigma_y^{-1}y\right),\label{eq3}
\end{align}
where $\Sigma_y:=\sigma^2I_n+\Phi\Gamma\Phi^T$ and $\Gamma:=\mathrm{diag}(\gamma)$. Then we estimate $\gamma$ by maximizing the marginalized probability distribution function. This is equivalent to minimizing $-\log p(y;\gamma)$, giving the cost function
\begin{equation}\label{eq6}
  \mathcal{L}_\gamma(\gamma)=\log|\Sigma_y|+y^T\Sigma_y^{-1}y,
\end{equation}
that is,
\begin{equation}\label{eq7}
\mathcal{L}_\gamma(\gamma)=\log|\sigma^2I_n+\Phi\Gamma\Phi^T|+y^T(\sigma^2I_n+\Phi\Gamma\Phi^T)^{-1}y.
\end{equation}

On the other hand, for fixed hyperparameters, one can compute the posterior distribution of $\theta$ from the Bayes formula. As a result, one has
\begin{equation}\label{eq4}
  p(\theta|y;\gamma)=\mathcal{N}(\mu_\theta,\Sigma_\theta),
\end{equation}
with $\mu_\theta=\sigma^{-2}\Sigma_\theta\Phi^Ty$ and $\Sigma_\theta=(\sigma^{-2}\Phi^T\Phi+\Gamma^{-1})^{-1}$.
After obtaining the estimate $\widehat{\gamma}$ of the hyperparameter vector, the optimal Bayes estimation of coefficient vector $\theta$ is
\begin{equation}\label{eq5}
 \widehat{\theta}=\widehat{\mu}_\theta=(\Phi^T\Phi+\sigma^{2}\widehat{\Gamma}^{-1})^{-1}\Phi^Ty
 =\widehat{\Gamma}\Phi^T(\sigma^2I_n+\Phi\widehat{\Gamma}\Phi^T)y,
\end{equation}
where $\widehat{\Gamma}=\mathrm{diag}(\widehat{\gamma})$.

Consider that
\begin{align*}
   y^T(\sigma^2I_n+\Phi\Gamma\Phi^T)^{-1}y =& \frac{1}{\sigma^2}y^Ty-\frac{1}{\sigma^4}y^T\Phi\Sigma_\theta\Phi^Ty= \frac{1}{\sigma^2}\|y-\Phi\mu_\theta\|^2+\mu_\theta^T\Gamma^{-1}\mu_\theta\\
            =& \min \limits_{\theta}\Big{\{} \frac{1}{\sigma^2}\|y-\Phi\theta\|^2+\theta^T\Gamma^{-1}\theta \Big{\}}.
\end{align*}
Thereby, we define the cost function with respect to $\theta$ and $\gamma$ as follows,
\begin{align}
  \mathcal{L}(\theta,\gamma) & =\frac{1}{\sigma^2}\|y-\Phi\theta\|^2+\theta^T\Gamma^{-1}\theta+\log|\sigma^2I_n+\Phi\Gamma\Phi^T|\overset{\Delta}{=}f(\theta,\gamma)+g(\theta,\gamma),\label{eq8}
\end{align}
where
\begin{equation*}
f(\theta,\gamma)=\frac{1}{\sigma^2}\|y-\Phi\theta\|^2+\theta^T\Gamma^{-1}\theta,~~
g(\theta,\gamma)=\log|\sigma^2I_n+\Phi\Gamma\Phi^T|.
\end{equation*}
It is easy to verify that $f(\theta,\gamma)$ is convex for $(\theta,\gamma)$ and $g(\theta,\gamma)$ is concave for $(\theta,\gamma)$. Define $\Omega:=\mathbb{R}^m\times [0,\infty)^m$. Obviously, $ \mathcal{L}(\theta,\gamma)$ is defined on the inner set $\mathrm{int}(\Omega):=\mathbb{R}^m\times (0,\infty)^m$, and not defined on the boundary set $\partial \Omega :=\Omega \backslash\mathrm{int}(\Omega)$. In such case, it is difficult to ensure that $\mathcal{L}$ has local minimum on the open set $\mathrm{int}(\Omega)$. To alleviate this issue, some supplementary definitions are imposed on the boundary set $\partial\Omega$.

It is noting that from \eqref{eq2}, if the hyperparameters are driven to zero, the associated coefficients will obey a degenerate distribution. In other words, if $\gamma_i=0$, then $p(\theta_i;\gamma_i=0)=\delta(\theta_i)$, which will lead to the posterior probability to satisfy
\begin{equation}\label{2.9}
  \mathbb{P}(\theta_i=0|y;\gamma_i=0)=1.
\end{equation}
It is clear that the sparsity of $\theta$ and $\gamma$ is closely related. More concretely, $\theta$ and $\gamma$ have the same sparsity pattern, i.e., when $\theta_i=0$, the optimal $\gamma_i$ equals to zero as well, and vice versa. Note though that, such information is lost if only one variable is considered for optimization \citep{wipf2011latent}. Thereout, we consider the following assumption.
\begin{assumption}\label{ass0}
If $\theta_i$ and $\gamma_i$ are driven to zero simultaneously, we only consider the following form of a path in $(\theta,\gamma)$-space
\begin{equation}\label{eq10}
  \theta_i=\mathcal{O}(\gamma_i),\ as \ \gamma_i\rightarrow 0^+.
\end{equation}
\end{assumption}
Such an assumption is reasonable according to Eq.~\eqref{eq5} and the KKT conditions of \eqref{eq8}. More significantly, the limit of $\mathcal{L}(\theta,\gamma)$ at each point $(\theta,\gamma)\in \Omega $ exits under Assumption \ref{ass0}. In fact, $\lim  \limits_{\theta_i\rightarrow 0,\gamma_i \rightarrow 0}\theta_i^2/\gamma_i$ exists when along the path satisfying \eqref{eq10} in $(\theta,\gamma)$-space. Based on this analysis, we supplement the definition of $\mathcal{L}(\theta,\gamma)$ to the boundary.

\begin{definition}
(i) For any point $(\theta,\gamma)\in \partial\Omega$, if $\theta_i$ and $\gamma_i$ are zero simultaneously or $\gamma_i\neq 0$ for any index $i\in \{1,2,\cdots,m\}$, the limit of the function $\mathcal{L}$ at point $(\theta,\gamma)$ exists and is unique when along the path satisfying \eqref{eq10}, and hence we let this limit be the function value of $\mathcal{L}$ at $(\theta,\gamma)$.\\
(ii) If there exists an index $i$ such that $\theta_i\neq0$, $\gamma_i=0$, the generalized limit of the function $\mathcal{L}$ at point $(\theta,\gamma)$ is positive infinity when along the path satisfying \eqref{eq10}, and hence we define $\mathcal{L}(\theta,\gamma)=M$, where $M$ is a sufficiently large number.
\end{definition}

%By this means, $L$ has a local minimum.
 Obviously, the boundary set can be divided into two disjoint sets, i.e., $\partial\Omega=\Omega_1\cup \Omega_2$, where $$\Omega_1:=\{(\theta,\gamma)|\mathrm{there\ exists\ at\ least\ an\ index}\  i\  \mathrm{such\  that}\ \theta_i\neq0,\ \gamma_i=0\}$$ and $\Omega_2:=\partial\Omega \backslash \Omega_1$. Thereby, the SBL problem can equivalently be viewed as solving
\begin{equation}\label{eq11}
(\widehat{\theta},\widehat{\gamma})\in \arg \min \limits_{(\theta,\gamma)\in \Omega} \mathcal{L}(\theta,\gamma).
\end{equation}
 This is a DC programming problem that can be solved by CCP. The CCP is a iterative algorithm that solves the following sequence of convex programs,
\begin{equation}\label{eq12}
  (\theta^{(k+1)},\gamma^{(k+1)})\in \arg \min \limits_{(\theta,\gamma)\in \Omega} \widehat{\mathcal{L}}(\theta,\gamma;\theta^{(k)},\gamma^{(k)}),
\end{equation}
or equivalently,
\begin{equation*}
  (\theta^{(k+1)},\gamma^{(k+1)}) \in \arg \min \limits_{(\theta,\gamma)\in \Omega} \left\{ \frac{1}{\sigma^2}\|y-\Phi\theta\|^2+\sum \limits _{i}\left(\frac {\theta_i^2}{\gamma_i}+c_i^{(k)}\gamma_i\right) \right\},
\end{equation*}
where
\begin{align*}
  \widehat{\mathcal{L}}(\theta,\gamma;\theta^{(k)},\gamma^{(k)}):=& f(\theta,\gamma)+g(\theta^{(k)},\gamma^{(k)}) + <\partial_\gamma g(\theta^{(k)},\gamma^{(k)}),\gamma-\gamma^{(k)}>,\\
c^{(k)}:=&\nabla_{\gamma}\log|\sigma^2I_n+\Phi\Gamma^{(k)}\Phi^T|,\\ \Gamma^{(k)}:=&\mathrm{diag}(\gamma^{(k)}_1,\gamma^{(k)}_2,\cdots,\gamma^{(k)}_m).
 \end{align*}
Note that, the supplementary definition for $\mathcal{L}$ can be extended to the function $\widehat{\mathcal{L}}$. It follows that
\begin{subequations}\label{eq16}
\begin{equation}
  \theta^{(k+1)}\in \arg \min \limits_{\theta} \Big{\{}\|y-\Phi\theta\|^2+2\sigma^2\sum \limits _{i}\sqrt{c_i^{(k)}}|\theta_i| \Big{\}},
\end{equation}
\begin{equation}
\gamma^{(k+1)}_i=\frac {|\theta_i^{(k+1)}|} {\sqrt{c_i^{(k)}}},\ i=1,2,\cdots,m.
\end{equation}
\end{subequations}
Given an initial point $(\theta^{(0)},\gamma^{(0)})\in \mathrm{int}(\Omega)$, we can obtain an iterative sequence $\{(\theta^{(k)},\gamma^{(k)})\}_{k=0}^{\infty}$. Moreover, $c^{(k)}$ can be analytically calculated by
\begin{align}
  c^{(k)}&=\nabla_\gamma \log|\sigma^2I_n+\Phi\Gamma^{(k)}\Phi^T|
           =\mathrm{diag}[\Phi^T(\sigma^2I_n+\Phi\Gamma^{(k)}\Phi^T)^{-1}\Phi].\label{eq13}
\end{align}
By the property of concave function, one has
\begin{equation*}
  \mathcal{L}(\theta,\gamma)\leq  \widehat{\mathcal{L}}(\theta,\gamma;\theta^{(k)},\gamma^{(k)}).
\end{equation*}
Combining with \eqref{eq12}, one has
\begin{align}\notag
   \mathcal{L}(\theta^{(k+1)},\gamma^{(k+1)})&\leq  \widehat{\mathcal{L}}(\theta^{(k+1)},\gamma^{(k+1)};\theta^{(k)},\gamma^{(k)})\\ \notag
                                   &\leq \widehat{\mathcal{L}}(\theta^{(k)},\gamma^{(k)};\theta^{(k)},\gamma^{(k)})\\ \label{eq14}
                                   &=\mathcal{L}(\theta^{(k)},\gamma^{(k)}).
\end{align}
Next, easy computations give
\[
\mathcal{L}(\theta^{(k)},\gamma^{(k)})\geq \log|\sigma^2I_n|=2n\log\sigma>-\infty.
\]
Therefore, $\{\mathcal{L}(\theta^{(k)},\gamma^{(k)})\}_{k=0}^{\infty}$ is a bounded monotonic non-increasing sequence. In fact, it is strictly monotonic before the iteration \eqref{eq16} reaches a local minimum (See Lemma \ref{lem1}). Therefore, the iterative process should be terminated when
\begin{equation}\label{stop}
  \mathcal{L}(\theta^{(k)},\gamma^{(k)})-\mathcal{L}(\theta^{(k+1)},\gamma^{(k+1)})\leq  \tau,
\end{equation}
where $\tau$ is a tolerance.

 In summary, the iterative process is showed in Algorithm~\ref{alg::conjugateGradient}.
\begin{algorithm}[h]\label{Algorithm}
  \caption{SBL algorithm based on CCP}
  \label{alg::conjugateGradient}
  \begin{algorithmic}[1]
    \REQUIRE
      $\Phi\in \mathbb{R}^{ n\times m}$: design matrix;
      $y\in \mathbb{R}^{n}$: observation vector;
      $\tau$: tolerance;
    \ENSURE
         $\theta^{(k+1)}$
    \STATE initial $\Gamma^{(0)}=I_m$, $k=0$;
    \REPEAT
      \STATE  $c^{(k)}=\mathrm{diag}[\Phi^T(\sigma^2I_n+\Phi\Gamma^{(k)}\Phi^T)^{-1}\Phi]$;
      \STATE $\theta^{(k+1)}\in \arg \min \limits_{\theta} \Big{\{}\|y-\Phi\theta\|^2+2\sigma^2\sum \limits _{i}\sqrt{c_i^{(k)}}|\theta_i|\Big{\}}$;
      \STATE $\gamma^{(k+1)}_j=\frac {|\theta_j^{(k+1)}|} {\sqrt{c_j^{(k)}}}$, $\Gamma^{(k+1)}=\mathrm{diag}(\gamma^{(k+1)})$;
      \STATE $k=k+1$
    %  \STATE compute the step size $\alpha_k=s/\parallel d_k \parallel_{2}$;
    \UNTIL{\big{(}$\mathcal{L}(\theta^{(k)},\gamma^{(k)})-\mathcal{L}(\theta^{(k+1)},\gamma^{(k+1)}) \leq \tau$\big{)}}
  \end{algorithmic}
\end{algorithm}

\begin{remark}
The constructed dictionary matrix from spatiotemporal data is generally high dimensional, resulting in a computationally expensive optimization problem \eqref{eq:sparserecovery}. In Appendix \ref{appendix:reduce}, we propose two strategies to reduce the the computational complexity.
\end{remark}

\section{Theoretical Guarantees}\label{sec3}
In this section, we first prove the convergence of the proposed SBL algorithm based on CCP. Then, we give the properties of the SBL cost function. Finally, we prove the selection consistency and provide error bounds of the SBL algorithm.

\subsection{Algorithmic Convergence}
Based on the analysis in Section~\ref{subsec:SBL}, the sequence $\{\mathcal{L}(\theta^{(k)},\gamma^{(k)})\}_{k=0}^{\infty}$ is monotonic bounded. Thereout, $\{f(\theta^{(k)},\gamma^{(k)})\}_{k=0}^{\infty}$ and $\{g(\theta^{(k)},\gamma^{(k)})\}_{k=0}^{\infty}$ are bounded.
To ensure the feasibility of the identification, we need the following assumption.
%avoid $ c^{(k)}_i$ being $0$
\begin{assumption}\label{ass1}
All columns of the design matrix $\Phi$ are nonzero, i.e., $\Phi_i\neq 0$ ($i=1,\ldots,m$).
\end{assumption}
Under Assumption \ref{ass1}, the sequences $\{\theta^{(k)}\}_{k=0}^{\infty}$ and $\{\gamma^{(k)}\}_{k=0}^{\infty}$ must be bounded (otherwise it violates the boundedness for sequences $\{f(\theta^{(k)},\gamma^{(k)})\}_{k=0}^{\infty}$ and $\{g(\theta^{(k)},\gamma^{(k)})\}_{k=0}^{\infty}$). Let $M_\gamma$ be an upper bound of sequence $\{\gamma^{(k)}\}_{k=0}^{\infty}$, i.e., $\gamma^{(k)}_i\leq M_\gamma$, for $\forall\ k\geq 0$, $i\in \{1,2,\cdots, m\}$. Then, by Eq.~\eqref{eq13}, one has
\begin{align}\notag
  \sigma^{-2} \|\Phi_i\|^2&\geq c^{(k)}_i=\Phi^T_i(\sigma^2I_n+\Phi\Gamma^{(k)}\Phi^T)^{-1}\Phi_i \\ \label{eq18}
                          &\geq \Phi^T_i(\sigma^2I_n+M_\gamma\Phi\Phi^T)^{-1}\Phi_i.
\end{align}
%\begin{remark}
Combining with \eqref{eq16} yields
\begin{equation}\label{eq17}
0< A \leq \frac {|\theta_i^{(k)}|}{\gamma_i^{(k)}}\leq B,
\end{equation}
for $\gamma_i^{(k)}\neq 0$, $k\geq0$, $i\in \{1,\cdots,m\}$, where $A:= \sqrt{\min \limits_{i}\Phi^T_i(\sigma^2I_n+M_\gamma\Phi\Phi^T)^{-1}\Phi_i}$ and $B:=\sigma^{-1} \max \limits_{i}\|\Phi_i\|$. This relation is consistent with Assumption \ref{ass0}.
%\end{remark}

By \eqref{eq12}, define a point-to-set map:
\begin{align}\notag
\mathcal{A}:\ \ \ \ \Omega  &\longrightarrow \sigma(\Omega)\\ \label{eq15}
              (\theta,\gamma) &\longmapsto \arg \min \limits_{(\bar{\theta},\bar{\gamma})\in \Omega} \widehat{\mathcal{L}}(\bar{\theta},\bar{\gamma};\theta,\gamma),
\end{align}
where $\sigma(\Omega)$ stands for the $\sigma$-algebra generated by $\Omega$. Obviously, $(\theta^{(k+1)},\gamma^{(k+1)})\in \mathcal{A}(\theta^{(k)},\gamma^{(k)})$. In particular, if the dictionary matrix $\Phi$ has full column rank, $\mathcal{A}(\theta,\gamma)$ is a single point set in $\Omega$, and hence the point-to-set map comes down to a point-point map. A fixed point of the map $\mathcal{A}$ is a point $(\theta,\gamma)$ that satisfies $\{(\theta,\gamma)\}=\mathcal{A}(\theta,\gamma)$, whereas a generalized fixed point of map $\mathcal{A}$ is a point $(\theta,\gamma)$ that satisfies $(\theta,\gamma)\in \mathcal{A}(\theta,\gamma)$. Let $S$ be the generalized fixed point set of $\mathcal{A}$. Then, $S$ is a stationary point set of \eqref{eq8}, as shown in Lemma \ref{lem0}.
\begin{lemma}\label{lem0}
If $(\theta^{\ast},\gamma^{\ast})\in S$, then $(\theta^{\ast},\gamma^{\ast})$ is a stationary point of the program \eqref{eq8}.
\end{lemma}
\begin{proof}
If $(\theta^{\ast},\gamma^{\ast})\in S$, by the iteration \eqref{eq16} generated from map $\mathcal{A}$, there exists $z^*\in \partial \|\theta^*\|_1$ such that
\begin{subequations}\label{st1}
\begin{numcases}{}
\Phi_i^T(\Phi\theta^{\ast}-y)+\sigma^2\sqrt{c_i^\ast}z_i^\ast=0,\\
 \frac {|\theta_i^{\ast}|}{\sqrt{c_i^\ast}}=\gamma_i^\ast, i=1,\cdots,m,
\end{numcases}
\end{subequations}
where $c^\ast:=\nabla_{\gamma}\log|\sigma^2I_n+\Phi\Gamma\Phi^T|\big{|}_{\gamma=\gamma^*}$. Moreover, it is easy to show that \eqref{st1} is the KKT conditions of \eqref{eq8}. Hence, $(\theta^{\ast},\gamma^{\ast})$ is a stationary point of \eqref{eq8}.
\end{proof}

To show the convergence of the iterative sequence $\{(\theta^{(k)},\gamma^{(k)})\}_{k=0}^{\infty}$, we introduce the following lemmas.
\begin{lemma}\label{lem1}
Let $\{(\theta^{(k)},\gamma^{(k)})\}_{k=0}^{\infty}$ be an iterative sequence generated by the point-to-set map $\mathcal{A}$ with an initial point $(\theta^{(0)},\gamma^{(0)})\in \mathrm{int}(\Omega)$. Then, one has
\begin{itemize}
  \item[(1)] All points $(\theta^{(k)},\gamma^{(k)})$ are in a compact set $F\subseteq \Omega$.
  \item[(2)] $\mathcal{A}$ is monotonically decreasing with respect to $\mathcal{L}$, i.e., $\mathcal{L}(\bar{\theta},\bar{\gamma})\leq \mathcal{L}(\theta,\gamma) $ for any $(\bar{\theta},\bar{\gamma})\in \mathcal{A}(\theta,\gamma)$, and
  \begin{itemize}
    \item[(i)] if $(\theta,\gamma)\not \in S$, $\mathcal{L}(\bar{\theta},\bar{\gamma})< \mathcal{L}(\theta,\gamma)$, $\forall (\bar{\theta},\bar{\gamma})\in \mathcal{A}(\theta,\gamma)$,
    \item[(ii)] if $(\theta,\gamma) \in S$, then either the algorithm terminates or $\mathcal{L}(\bar{\theta},\bar{\gamma})\leq \mathcal{L}(\theta,\gamma)$, $\forall (\bar{\theta},\bar{\gamma})\in \mathcal{A}(\theta,\gamma)$.
  \end{itemize}
\end{itemize}
\end{lemma}
\begin{proof}
We first verify claim (1). According to the former analysis, we see that the sequence $\{\theta^{(k)}\}_{k=0}^{\infty}$ is bounded. Let $m_\theta$ and $M_\theta$ be the lower and upper bounds of sequence $\{\theta^{(k)}\}_{k=0}^{\infty}$ respectively. Then, there exists a compact set $F=[m_\theta,M_\theta]^m\times [0,M_\gamma]^m\subseteq \Omega$ such that $(\theta^{(k)},\gamma^{(k)})\in F$ for $k=0,1,\ldots$.

Next, we verify claim (2). By the definition of $\mathcal{A}$, one has
\begin{equation}\label{eq_lem2}
   \mathcal{L}(\bar{\theta},\bar{\gamma})\leq  \widehat{\mathcal{L}}(\bar{\theta},\bar{\gamma};\theta,\gamma)\leq \widehat{\mathcal{L}}(\theta,\gamma;\theta,\gamma)=\mathcal{L}(\theta,\gamma),
\end{equation}
for any $(\bar{\theta},\bar{\gamma})\in \mathcal{A}(\theta,\gamma)$, i.e., $\mathcal{A}$ is monotonically decreasing with respect to $\mathcal{L}$. Moreover, if $(\theta,\gamma)\not\in S$, then $\mathcal{L}(\bar{\theta},\bar{\gamma})< \mathcal{L}(\theta,\gamma)$. In fact, if $\mathcal{L}(\bar{\theta},\bar{\gamma})=\mathcal{L}(\theta,\gamma)$, by \eqref{eq_lem2}, $\widehat{\mathcal{L}}(\bar{\theta},\bar{\gamma};\theta,\gamma)= \widehat{\mathcal{L}}(\theta,\gamma;\theta,\gamma)$, and therefore $(\theta,\gamma)\in S$.
\end{proof}

\begin{lemma}\label{lem2}
Let $\{(\theta^{(k)},\gamma^{(k)})\}_{k=0}^{\infty}$ and $\{(\bar{\theta}^{(k)},\bar{\gamma}^{(k)})\}_{k=0}^{\infty}$ be two sequences in $\mathrm{int}(\Omega)\bigcup \Omega_2$. If the following conditions:
\begin{itemize}
  \item[(1)]$(\theta^{(k)},\gamma^{(k)})\rightarrow (\theta^{\ast},\gamma^{\ast})\in \Omega$,
  \item[(2)]$(\bar{\theta}^{(k)},\bar{\gamma}^{(k)})\rightarrow (\theta^{\ast\ast},\gamma^{\ast\ast})\in \mathrm{int}(\Omega)\bigcup\Omega_2$,
   \item[(3)]  $|\bar{\theta}_i^{(k)}| / \bar{\gamma}_i^{(k)}$ is bounded on $\{\bar{\gamma}_i^{(k)}\neq 0\}$,
  \item[(4)] $\bar{\theta}_i^{(k)}$ and $\bar{\gamma}_i^{(k)}$ are both zero and not zero at the same time,
  \item[(5)]$(\bar{\theta}^{(k)},\bar{\gamma}^{(k)})\in \mathcal{A}(\theta^{(k)},\gamma^{(k)})$
\end{itemize}
hold. Then, one has
\begin{equation}\label{l1}
(\theta^{\ast\ast},\gamma^{\ast\ast})\in \mathcal{A}(\theta^{\ast},\gamma^{\ast}).
\end{equation}
\end{lemma}
\begin{proof}
See Appendix.
\end{proof}

\begin{theorem}\label{th1}
%Suppose an initial point $(\theta^{(0)},\gamma^{(0)})$ is given in $\mathrm{int}(\Omega)$.
Let $\{(\theta^{(k)},\gamma^{(k)})\}_{k=0}^{\infty}$ be an iterative sequence generated by the point-to-set map $\mathcal{A}$ with an initial point $(\theta^{(0)},\gamma^{(0)})\in \mathrm{int}(\Omega)$. Then, the sequence $\{(\theta^{(k)},\gamma^{(k)})\}_{k=0}^{\infty}$ converges to a stationary point of \eqref{eq8}. Moreover,
\begin{equation}\label{station}
\lim \limits_{k\rightarrow \infty}\mathcal{L}(\theta^{(k)},\gamma^{(k)})=\mathcal{L}(\theta^{\ast},\gamma^{\ast}),
\end{equation}
where $(\theta^{\ast},\gamma^{\ast})$ is a stationary point of \eqref{eq8}.
\end{theorem}
\begin{proof}
By Lemma \ref{lem1}, one has that all points $(\theta^{(k)},\gamma^{(k)})$ are in a compact set. Then, there exists a convergent subsequence
\begin{equation}\label{t1}
  (\theta^{(k_s)},\gamma^{(k_s)})\longrightarrow (\theta^{\ast},\gamma^{\ast}),\ as\ s\rightarrow \infty.
\end{equation}
Combining with the monotonicity of $\mathcal{A}$ with respect to $\mathcal{L}$, one has
\begin{equation}\label{t2}
   \mathcal{L}(\theta^{(k_{s+1})},\gamma^{(k_{s+1})}) \leq  \mathcal{L}(\theta^{(k_s)},\gamma^{(k_s)})< +\infty.
\end{equation}
By \eqref{eq17}, one has $(\theta^{\ast},\gamma^{\ast})\in \mathrm{int}(\Omega)\cup \Omega_2$. In fact, if $(\theta^{\ast},\gamma^{\ast})\in \Omega_1$, then by \eqref{t1}, one has
\begin{equation*}
  \lim \limits_{s\rightarrow \infty} \mathcal{L}(\theta^{(k_{s})},\gamma^{(k_{s})})=+\infty.
\end{equation*}
This is a contradiction with \eqref{t2}. Next, if $(\theta^\ast,\gamma^\ast)\in \mathrm{int}(\Omega)$, then by the continuity of $\mathcal{L}$ on $\mathrm{int}(\Omega)$, one has
\begin{equation}\label{t3}
  \lim \limits _{s\rightarrow \infty} \mathcal{L}(\theta^{(k_s)},\gamma^{(k_s)})=\mathcal{L}(\theta^{\ast},\gamma^{\ast});
\end{equation}
If $(\theta^\ast,\gamma^\ast)\in \Omega_2$, then by \eqref{eq17} and the supplementary definition of $\mathcal{L}$ on $\Omega_2$, Eq. \eqref{t3} holds.

On the other hand, for the sequence $\{(\theta^{(k_s+1)},\gamma^{(k_s+1)})\}_{s=1}^{\infty}$, there exists a convergent subsequence
\begin{equation}\label{t4}
  (\theta^{(k_{s_j}+1)},\gamma^{(k_{s_j}+1)})\longrightarrow (\theta^{\ast\ast},\gamma^{\ast\ast}),\ as\ j\rightarrow \infty.
\end{equation}
Similar with the proof of Eq. \eqref{t3}, we obtain that
\begin{equation}\label{t5}
    \lim \limits _{j\rightarrow \infty} \mathcal{L}(\theta^{(k_{s_j}+1)},\gamma^{(k_{s_j}+1)})=\mathcal{L}(\theta^{\ast\ast},\gamma^{\ast\ast}).
\end{equation}
By Lemma \ref{lem1}, we can conclude that the sequence $\{\mathcal{L}(\theta^{(k)},\gamma^{(k)})\}_{k=0}^{\infty}$ is  monotonic bounded, and therefore is convergent. Combining with Eqs.~\eqref{t3} and \eqref{t5}, one has
\begin{equation}\label{t8}
  \mathcal{L}(\theta^{\ast},\gamma^{\ast})=\mathcal{L}(\theta^{\ast\ast},\gamma^{\ast\ast}).
\end{equation}
Furthermore, by \eqref{t1} and Eq.~\eqref{t3}, one has
\begin{equation}\label{t6}
   (\theta^{(k_{s_j})},\gamma^{(k_{s_j})})\longrightarrow (\theta^{\ast},\gamma^{\ast}),\ as\ j\rightarrow \infty
\end{equation}
and
\begin{equation}\label{t7}
    \lim \limits _{j\rightarrow \infty} \mathcal{L}(\theta^{(k_{s_j})},\gamma^{(k_{s_j})})=\mathcal{L}(\theta^{\ast},\gamma^{\ast}).
\end{equation}

Consider that $(\theta^{(k_{s_j}+1)},\gamma^{(k_{s_j}+1)}) \in  \mathcal{A}(\theta^{(k_{s_j})},\gamma^{(k_{s_j})})$. Combining with \eqref{eq16} and \eqref{eq17}, the sequences $\{(\theta^{(k_{s_j})},\gamma^{(k_{s_j})})\}$ and $\{(\bar{\theta}^{(k_{s_j}+1)},\bar{\gamma}^{(k_{s_j}+1)})\}$ satisfy the conditions of Lemma \ref{lem2}. Applying Lemma \ref{lem2} yields
\begin{equation}\label{t9}
  (\theta^{\ast\ast},\gamma^{\ast\ast})\in \mathcal{A}(\theta^{\ast},\gamma^{\ast}).
\end{equation}
Then, by \eqref{t8}, \eqref{t9}, and Lemma \ref{lem1}, one has $$(\theta^{\ast},\gamma^{\ast})\in S. $$
Hence, by Lemma \ref{lem0}, $(\theta^{\ast},\gamma^{\ast})$ is a stationary point of \eqref{eq8}. Moreover, the Eq.~\eqref{station} follows from the convergence of $\{\mathcal{L}(\theta^{(k)},\gamma^{(k)})\}_{k=0}^{\infty}$ and Eq.~\eqref{t3}.
%Since $L(\theta,\gamma)$ is nonconvex and differentiable on $\mathrm{int}(\Omega)$, then an local minimum point of $L(\theta,\gamma)$ must be an stationary point of \eqref{eq8}.
\end{proof}
\subsection{Properties of the Cost Function in $(\theta,\gamma)$-space} \label{sparse}
%For the sparse recovery problem \jl{\eqref{eq:sparserecovery}},
%The cost function either in $\theta$-space or $\gamma$-space have been researched \citep{wipf2011latent}. However,
This section mainly discusses the properties of cost function $\mathcal{L}(\theta,\gamma)$ whose minimization corresponds with maximally sparse solutions. In particular, we will focus on establishing the relationship between  the cost function in $(\theta,\gamma)$-space and the one in $\theta$-space. This relationship facilitates us to understand how the underlying cost function promotes sparsity.
 %is function in $\theta$-space can be interpreted as a least squares function with concave penalty.
%Moreover, we shall show that given certain conditions the cost function in $(\theta,\gamma)$-space is characterized by a local minimum that can produce the maximally sparse solution at the posterior mean.

%\begin{subequations}\label{st4}
%\begin{numcases}{}
%\Phi_i^T(\Phi\theta^{\ast}-y)+\sigma^2\sqrt{c_i^\ast}z_i^\ast=0,\\
% \frac {|\theta_i^{\ast}|}{\sqrt{c_i^\ast}}=\gamma_i^\ast, i=1,\cdots,m,
%\end{numcases}
%\end{subequations}
%where $c^\ast:=\nabla_{\gamma^\ast}\log|\sigma^2I_n+\Phi\Gamma\Phi^T|$ and $z^{\ast}:=[z^{\ast}_1,\cdots,z^{\ast}_n]^T\in \partial\|\theta^{\ast}\|_1$.
Let $(\theta^{\ast},\gamma^{\ast})$ be a local minimum point of $\mathcal{L}$. It is noting that a local minimum point of $\mathcal{L}$ must be a stationary point of $\mathcal{L}$.
Then, $(\theta^{\ast},\gamma^{\ast})$ satisfies the KKT conditions \eqref{st1}. Thereout, $\theta^\ast$ and $\gamma^\ast$ have the same sparsity pattern. Consider that the log-det penalty term in $\mathcal{L}(\theta,\gamma)$ is concave for $\gamma$, which is frequently used for promoting sparse solution. Consequently, $\mathcal{L}(\theta,\gamma)$ can reach a local minimum at a sparse solution. To illustrate this result, we present a theorem that connects local minimum of $\mathcal{L}$ with local minimum of the cost function in $\theta$-space.

 %Under Assumption \ref{ass0}, the function $\mathcal{L}(\theta,\gamma)$ is defined on $\Omega$, and then we have the following theorem.
\begin{theorem}\label{th3}
Define the $\theta$-space cost function
\begin{equation}\label{loc1}
  \mathcal{L}_\theta(\theta):=\|y-\Phi\theta\|^2+\sigma^2h(\theta),
\end{equation}
with penalty
\begin{equation}\label{loc2}
h(\theta):=\min \limits_{\gamma\succeq  0} \Big{\{}\sum \limits_{i} \frac {\theta_i^2}{\gamma_i}+\log|\sigma^2I_n+\Phi\Gamma\Phi^T| \Big{\}}.
\end{equation}
Then, under Assumption \ref{ass0} and Assumption \ref{ass1}, $(\theta^\ast,\gamma^\ast)$ is a global (local) minimum  of \eqref{eq8} iff $\theta^\ast$ is a global (local) minimum of \eqref{loc1} with $\gamma_i^\ast=\theta_i^\ast/\sqrt{c_i^\ast}$, where $c_i^\ast$ is the same as in \eqref{st1}.
\end{theorem}
 \begin{proof}
Obviously, under Assumption \ref{ass0}, the term  $\theta_i^2/\gamma_i$ is definable. If $(\theta^*,\gamma^*)$ is a global minimum of \eqref{eq8}, then it satisfies the KKT conditions \eqref{st1}, and hence $\theta^*$ is a global minimum of \eqref{loc1} with $\gamma_i^\ast=\theta_i^\ast/\sqrt{c_i^\ast}$. In turn, let  $\theta^*$ be a global minimum of \eqref{loc1}. With $\theta$ fixed, global minimum $\gamma$ of the optimization problem from \eqref{loc2} satisfies the equation: $\gamma_i=\theta_i/\sqrt{c_i}$, where $c:=\nabla_\gamma\log|\sigma^2I_n+\Phi\Gamma\Phi^T|$. Then, there must exist some $\gamma^*$ that minimizes $\mathcal{L}(\theta^*,\gamma)$ such that $\gamma_i^*=\theta_i^*/\sqrt{c_i^*}$. It leads to that $(\theta^*,\gamma^*)$ satisfies the KKT condition of \eqref{eq8}. Note that $\mathcal{L}_\theta(\theta)$ is a strict upper bound on $\mathcal{L}(\theta,\gamma)$ with $\mathcal{L}_\theta(\theta)=\min \limits_{\gamma\succeq 0}\mathcal{L}(\theta,\gamma)$. Hence, $(\theta^*,\gamma^*)$ is a global minimum of \eqref{loc1}.

The relationship between global solutions to \eqref{eq8} and \eqref{loc1} can be extended to local solutions as well. In fact, since the optimization problem from \eqref{loc2} is convex with respect to the reparameterization of $\gamma$ given by $\upsilon_i:=\log\gamma_i$, all of its minima (if multiple exist) are connected. In other words, the minimization problem
\begin{equation*}
h(\theta):=\min \limits_{\upsilon} \Big{\{}\sum \limits_{i} e^{-\upsilon_i}\theta_i^2+\log|\sigma^2I_n+\Phi e^\Upsilon\Phi^T| \Big{\}},
\end{equation*}
where $\Upsilon :=\mathrm{diag}(\upsilon)$, is convex for $\upsilon$, and therefore all local minima are connected. Hence, the above result for global minima is also hold for local minima.
\end{proof}
%
%The proof of Theorem \ref{th3} can be briefly summarized using the relationships between multiple limit and iterative limit.

 Theorem \ref{th3} shows that global (local) minimum of \eqref{eq8} can be obtained by minimizing a least squares with concave penalty. Note that concave, non-decreasing regularization functions are well-known to enforce sparsity. Since $h(\theta)$ is such a function, it can induce sparsity to some extent. Furthermore, $h(\theta)$ provides a tighter approximation of $\ell_0$ norm than $\ell_1$ norm while
generating many fewer local minimum than using $\ell_0$ norm. However, usually $h(\theta)$ is non-factorable, i.e., $h(\theta)\neq\sum \limits_{j}h_i(\theta_i)$, unlike traditional penalty, for example, $\ell_1$-norm penalty and $\ell_2$-norm penalty. It leads to the difficulty of studying the behaviors of the optimal solution of \eqref{loc1}.

\subsection{Selection Consistency and Error Bounds} \label{sec:selection}% (Separable Case)
In this subsection, we will compare the sparsity pattern of the local minimum of \eqref{eq8} with that of real coefficient vector in \eqref{regre}. For convenience, we denote the real coefficient vector $\theta$ in \eqref{eq:sparserecovery} by $\theta^{\mathrm{real}}$. Theorem \ref{th1} showed that the iterative sequence $\{(\theta^{(k)},\gamma^{(k)})\}_{k=0}^{\infty}$ generated from Algorithm \ref{alg::conjugateGradient} converges to a stationary point $(\theta^\ast,\gamma^\ast)$ of \eqref{eq8}. And, each local minimum of \eqref{eq8} must be stationary point. Therefore, it is necessary to explore the connections between $\theta^\ast$ and $\theta^{\mathrm{real}}$.
  For the iterative
sequence $\{(\theta^{(k)},\gamma^{(k)})\}_{k=0}^{\infty}$ generated from Algorithm \ref{alg::conjugateGradient}, one has
\begin{subequations}\label{st2}
\begin{numcases}{}
\Phi_i^T(\Phi\theta^{(k+1)}-y)+\sigma^2\sqrt{c_i^{(k)}}z_i^{(k+1)}=0,\\
 \frac {|\theta_i^{(k+1)}|}{\sqrt{c_i^{(k)}}}=\gamma_i^{(k+1)}, i=1,\cdots,m,
\end{numcases}
\end{subequations}
where $c^{(k)}:=\mathrm{diag}[\Phi^T(\sigma^2I_n+\Phi\Gamma^{(k)}\Phi^T)^{-1}\Phi]$ and $z^{(k+1)}:=[z^{(k+1)}_1,\cdots,z^{(k+1)}_n]^T\in \partial\|\theta^{{k+1}}\|_1$.

To  simplify the Algorithm \ref{alg::conjugateGradient}, we consider a special case (i.e., an orthonormal design $\Phi^T\Phi=I_m$).
In such case, one has
\begin{align*}
   &(\sigma^2I_n+\Phi\Gamma^{(k+1)}\Phi^T)^{-1}  \\
   =& \sigma^{-2}I_n-\sigma^{-4}\Phi\Gamma^{(k+1)}\left[I_m+\sigma^{-2}\Gamma^{(k+1)}\Phi^T\Phi\right]^{-1}\Phi^T \\
   = & \sigma^{-2}I_n-\sigma^{-2}\Phi\Gamma^{(k+1)}\left[\sigma^2 I_m+\Gamma^{(k+1)}\right]^{-1}\Phi^T,
\end{align*}
and hence,
\begin{align}\notag
   c^{(k+1)}_i&= \sigma^{-2}\Phi_i^T\Phi_i- \sigma^{-2}\Phi_i^T\Phi\Gamma^{(k+1)}\left[\sigma^2  I_m+\Gamma^{(k+1)}\right]^{-1}\Phi^T\Phi_i\\  \notag
   & =\sigma^{-2}-\sigma^{-2}e_i^T\Gamma^{(k+1)}\left[\sigma^2 I_m+\Gamma^{(k+1)}\right]^{-1}e_i\\  \notag
   &=\sigma^{-2}-\sigma^{-2}\frac{\gamma_i^{(k+1)}}{\sigma^{2}+\gamma_i^{(k+1)}}\\ \label{hence1}
   &=\frac {1}{\sigma^2+\gamma_i^{(k+1)}}=\frac {\sqrt{c_i^{(k)}}}{\sigma^2\sqrt{c_i^{(k)}}+|\theta_i^{(k+1)}|},
\end{align}
where the first equality follows from the Woodbury equation. By Theorem 1, we see that $\{(\theta^{(k)},\gamma^{(k)})\}$ converges to the stationary
point $(\theta^{\ast},\gamma^{\ast})$ of  \eqref{eq8}. And, all stationary points of \eqref{eq8} satisfy the KKT conditions \eqref{st1}. Hence, $c^{(k)}$ converges to
$c^\ast$ as $k$ tends to infinity. Next, take the limit of Eq.~\eqref{hence1}, we
obtain an implicit equation:
\begin{equation}\label{st3}
\sigma^2c_i^\ast+|\theta_i^\ast|\sqrt{c_i^\ast}-1=0,
\end{equation}
which can be viewed as a quadratic equation with respect to $\sqrt{c_i^\ast}$. It can be easily calculated
\begin{equation}\label{st4}
\sqrt{c_i^\ast}=\frac {-|\theta_i^\ast|+\sqrt{|\theta_i^\ast|^2+4\sigma^2}}{2\sigma^2}=\frac {2}{|\theta_i^\ast|+\sqrt{|\theta_i^\ast|^2+4\sigma^2}}.
\end{equation}
At this time, the weighted update rule is similar with the one in \citep{candes2008enhancing}. Similarly, one has from \eqref{st2}
\begin{equation}\label{st5}
  \Phi_i^T(\Phi\theta^{\ast}-y)+\sigma^2\sqrt{c_i^{\ast}}z_i^{\ast}=0,
\end{equation}
where $z^\ast:=[z_1^\ast,z_2^\ast,\cdots,z_m^\ast]^T$. Substituting Eqs.~\eqref{eq:sparserecovery} and \eqref{st4} into Eq.~\eqref{st5} yields
\begin{equation}\label{st6}
  \theta^\ast_i-\theta_i^{\mathrm{real}}+\frac {-|\theta_i^\ast|+\sqrt{|\theta_i^\ast|^2+4\sigma^2}}{2}z_i^\ast=\Phi_i^Tv.
\end{equation}
Obviously, $\theta_i^\ast=\theta_i^{\mathrm{real}}$ if $\sigma^2=0$ and $v=0$. To compare the true coefficient vector $\theta^{\mathrm{real}}$ and the estimated one $\theta^\ast$, we introduce the support as a performance measure. Here, the support of a given vector $\theta$ is defined as
\[
\mathbb{S}(\theta):=\{i: |\theta_i|>\xi\},
\]
where $\xi>0$ is a fixed threshold that allows us to neglect some very small nonzero coefficients compared to those of coefficients. Let $\theta_{\min}$ be the minimum absolute value of all nonzero entries of $\theta^{\mathrm{real}}$. Then, when $\xi< \theta_{\min}$, one has $\mathbb{S}(\theta^{\mathrm{real}})=\{i :\theta^{\mathrm{real}}_i\neq 0\}$, and hence $\theta_{\min}:=\min \limits _{i\in \mathbb{S}(\theta^{\mathrm{real}})}|\theta^{\mathrm{real}}_i|$. Moreover, we tackle the related issues with noise, demonstrating that the estimator $\theta^\ast$ is selection consistent in recovering sparsity pattern.
\begin{definition}
 A estimator $\theta^\ast$ is selection consistent if $\mathbb{P}(\mathbb{S}(\theta^\ast)=\mathbb{S}(\theta^{\mathrm{real}}))\rightarrow1$, as $\sigma^2\rightarrow 0$.
 \end{definition}
\begin{theorem}\label{th2}
Assume that $ \Phi^T\Phi= I_m$. If $2\sigma\sqrt{\log 2m}\leq \xi < \frac {1}{2}\theta_{\min}-\sigma$, then $\mathbb{S}(\theta^{\mathrm{real}})=\mathbb{S}(\theta^\ast)$
with probability at least $1-e^{-\frac {\xi^2}{4\sigma^2}}$, i.e., $\theta^\ast$ is selection consistent. Moreover,
\[
|\theta_i^{\mathrm{real}}-\theta^\ast_i|\leq \xi+\frac {\sigma^2}{\xi},\ \forall i\in \mathbb{S}(\theta^{\mathrm{real}}).
\]
\end{theorem}
\begin{proof}
For any $i \in \mathbb{S}(\theta^{\mathrm{real}})$, one has $i\in \mathbb{S}(\theta^\ast)$ if $|\Phi_i^Tv|<\frac {1}{2}\theta_{\min}$. In fact, if $i\not \in \mathbb{S}(\theta^\ast)$, then by Eq.~\eqref{st6}, one has
\begin{equation}\label{st7}
\theta_i^\ast+\frac {-|\theta_i^\ast|+\sqrt{|\theta_i^\ast|^2+4\sigma^2}}{2}z_i^\ast=\theta^{\mathrm{real}}_i+\Phi_i^Tv.
\end{equation}
When $|\Phi_i^Tv|<\frac {1}{2}\theta_{\min}$, one has
\begin{equation*}
  |\theta^{\mathrm{real}}_i+\Phi_i^Tv|\geq |\theta^{\mathrm{real}}_i|-|\Phi_i^Tv|\geq\frac{1}{2}\theta_{\min},
\end{equation*}
and
\[
\left|\theta_i^\ast+\frac {-|\theta_i^\ast|+\sqrt{|\theta_i^\ast|^2+4\sigma^2}}{2}z_i^\ast\right|\leq \xi+\sigma.
\]
 By assumption, it is contradictory to Eq.~\eqref{st7}. For any $i \not \in \mathbb{S}(\theta^{\mathrm{real}})$, one has $i\not \in \mathbb{S}(\theta^\ast)$ if $|\Phi_i^Tv|<\xi$. In fact, if $i\in \mathbb{S}(\theta^\ast)$, then from Eq.~\eqref{st6} and assumption
\begin{equation}\label{st9}
|\Phi_i^Tv|=\frac { \left|\theta_i^\ast+\sqrt{|\theta_i^\ast|^2+4\sigma^2}z_i^\ast\right|}{2}>\xi,
\end{equation}
which is a contradiction. Hence, $\mathbb{S}(\theta^{\mathrm{real}})=\mathbb{S}(\theta^\ast)$ holds on $\{\max \limits_{i}|\Phi_i^Tv|   <\xi\}$. Since $\Phi^T\Phi=I_m$, then $\Phi_i^Tv$, $i=1,\cdots,m$, are independent Gaussian random variables with zero mean and variance $\sigma^2$. By the Hoeffding inequality, one has
\[
\mathbb{P}\left[| \Phi_i^Tv| \geq\xi\right]\leq 2e^{-\frac {\xi^2}{2\sigma^2}}.
\]
Thereby,
\begin{align*}
\mathbb{P}\left[\max \limits_{i}| \Phi_i^Tv| <\xi\right]&= \prod _i\mathbb{P}\left[| \Phi_i^Tv| <\xi\right]\\
                                                        & \geq \left(1-2e^{-\frac {\xi^2}{2\sigma^2}}\right)^m\\
                                                        &\geq 1-e^{-\frac {\xi^2}{4\sigma^2}},
\end{align*}
where the last inequality follows from the fact that $(1-x)^m\geq 1-mx$, $\forall x\in [0,1]$. Next, we establish the bounds of error $\theta_i^{\mathrm{real}}-\theta^\ast_i$. For any $i\in \mathbb{S}(\theta^{\mathrm{real}})$, by Taylor formula, there exists an real number $\eta\in (0,4\sigma^2)$ such that
\[
\sqrt{|\theta_i^\ast|^2+4\sigma^2}=|\theta_i^\ast|+\frac {2\sigma^2}{\sqrt{|\theta_i^\ast|^2+\eta}}.
\]
Substituting it into Eq.~\eqref{st6} yields
\begin{align*}
 |\theta_i^{\mathrm{real}}-\theta^\ast_i|=\left|\Phi_i^Tv-\frac {\sigma^2}{\sqrt{|\theta_i^\ast|^2+\eta}}\right|
                                         \leq \xi+\frac {\sigma^2}{\xi}.
\end{align*}
This completes the proof. \end{proof}

\begin{remark}
Under the scaling $\sigma^2\rightarrow 0$, the condition of Theorem \ref{th2} holds, and moreover that the probability of success converges to one. Additionally, from the proof process, we see that the estimator $\theta^\ast$ is sign-consistent, i.e., $\mathbb{P}(\mathrm{sign}(\theta^\ast)=\mathrm{sign}(\theta^{\mathrm{real}}))\rightarrow1$, as $\sigma^2\rightarrow 0$, with the convention: $\mathrm{sign}(\theta_i)=0$, $|\theta_i|\leq \xi$.
\end{remark}

The above analysis implies that the log-det term in \eqref{eq8} is factorable in the case of orthogonal design and can be expressed in the form independently of $\Phi$. More generally, if $\Phi^T\Phi$ is diagonal, the log-det term is still factorable and has similar results.
%Furthermore, \eqref{eq8} has no local minimum in the limit as $\sigma^2\rightarrow 0$ due to only a single feasible solution with $\theta=\Phi^ Ty$.
}

\section{Experiments}\label{sec:experiment}
This section performs several experiments to demonstrate and validate our presented method. All codes to reproduce the results are publicly available at {\it https://github.com/HAIRLAB/S3d}. First, in $\S$\ref{sec:simuexample}, we illustrate the discovery process with $\text{S}^3\text{d}$ on eight prototypical PDEs used in various scientific fields and had several different mathematical forms. We seek
to show that the proposed $\text{S}^3\text{d}$ method handles model diversity and typical errors or difficulties that arise in modeling. Secondly, in $\S$\ref{sec:subexperiment}, we demonstrate the applicability of our $\text{S}^3\text{d}$ method to real data collected from an experiment on convection in an ethanol-water mixture in a long, narrow, annular container heated from below~\citep{Voss}. We reconstruct the underlying PDE known as the  Complex Ginzburg-Landau Equation (CGLE) without any \textit{a priori} information.
Finally in $\S$\ref{sec:comparison}, we compare the performance of $\text{S}^3\text{d}$, STRidge and DR in a identification task for the original dataset generated in \citep{Rudy}. It has been observed that $\text{S}^3\text{d}$ achieves better performance in terms of smaller parametric error. In addition, the number of samples that is required by $\text{S}^3\text{d}$ for training is smaller than the STRidge method.

\subsection{Discovery of Prototypical PDEs}\label{sec:simuexample}
In this section, we collect our synthetic data by generating datasets from numerical solutions to canonical PDEs. The
data collected on each PDE is composed into snapshot matrix $U\in \mathbb{R}^{n_x\times n_t}$. We denote  the $i$th row and $j$th column of $U$ by $u(x_i,t_j)$.
To illustrate the robustness of our method in more realistic settings, we add 1\% Gaussian noise into the synthetic data. We write our synthetic measurements with added Gaussian noise as
$$
y(x_j,t_i)=u(x_j,t_i)+\eta_{j,i},
$$
where $\eta_{j,i}$ takes a normal distribution $\mathcal{N}(\mu, \sigma_u^2)$ with mean $\mu=0$ and standard deviation $\sigma_u$. The error from the noise-contained data is  quantified by means of the root mean square error (RMSE),
\begin{eqnarray*}
	Err(y)=\sqrt{\frac{1}{n_xn_t}\sum\limits_{j=1}^{n_x}\sum\limits_{i=1}^{n_t}(u(x_j,t_i)-y(x_j,t_i))^2}.
\end{eqnarray*}
We also use different error norms for assessing our results. We compute the RMSE for evaluating the impact of noise and the mean-square error (MSE) and Standard Deviation (STD) for assessing our discovery results.

\subsubsection{Sine-Gordon equaiton}
We begin our examples with the $\text{S}^3\text{d}$ method by considering the propagation of a slip
in an infinite chain of elastically bound atoms lying over a fixed lower chain of similar atoms~\citep{Baro}, whose dynamics are governed by the sine-Gordon equation~\citep{Ablo},
\begin{eqnarray}\label{SGE-1}
u_{tt}=u_{xx}-\sin(u),
\end{eqnarray}
where $u_{xx}$ represents the elastic interaction energy between neighboring
atoms, $u_{tt}$ represents neighboring atoms' kinetic energy, and $\sin(u)$ represents the
potential energy due to the fixed lower chain. Eq.~\eqref{SGE-1} admits various types of special solutions, such as small-amplitude solutions, traveling-wave solutions, and envelope wave solutions \citep{Baro}. Our objective is to extract the dynamical regime from a breather solution to Eq.~\eqref{SGE-1} in the following form:
\begin{eqnarray*}
	u(x,t)=4\arctan\left(\frac{\sin(t/\sqrt{2})}{\cosh(x/\sqrt{2})}\right).
\end{eqnarray*}
For our analysis using $\text{S}^3\text{d}$, we use a dataset consisting of $256$ measurements of the breather solution. The measurements are equally sampled in time with an interval of $\triangle t=0.0784$. We uniformly discretize the spatial domain $[-12.4,12.4]$ of interest into $511$ equal spaces (so that $m=512$) using a grid step of $\triangle x=0.0485$. We show a visualization of the discrete breather solution in Fig.~\ref{SGE-2}. It is seen that the solution is not only symmetric but also periodic (a solitary wave propagates with periodicity $2\sqrt{2}\pi$). We organize the data into snapshot matrix, $U\in \mathbb{R}^{512\times 256}$, and focus on discovering the governing equation with the nonlinear term $\sin(u)$ from our small periodic dataset.
\begin{figure}[h]
	\centering \subfigure{\includegraphics[scale=.3]{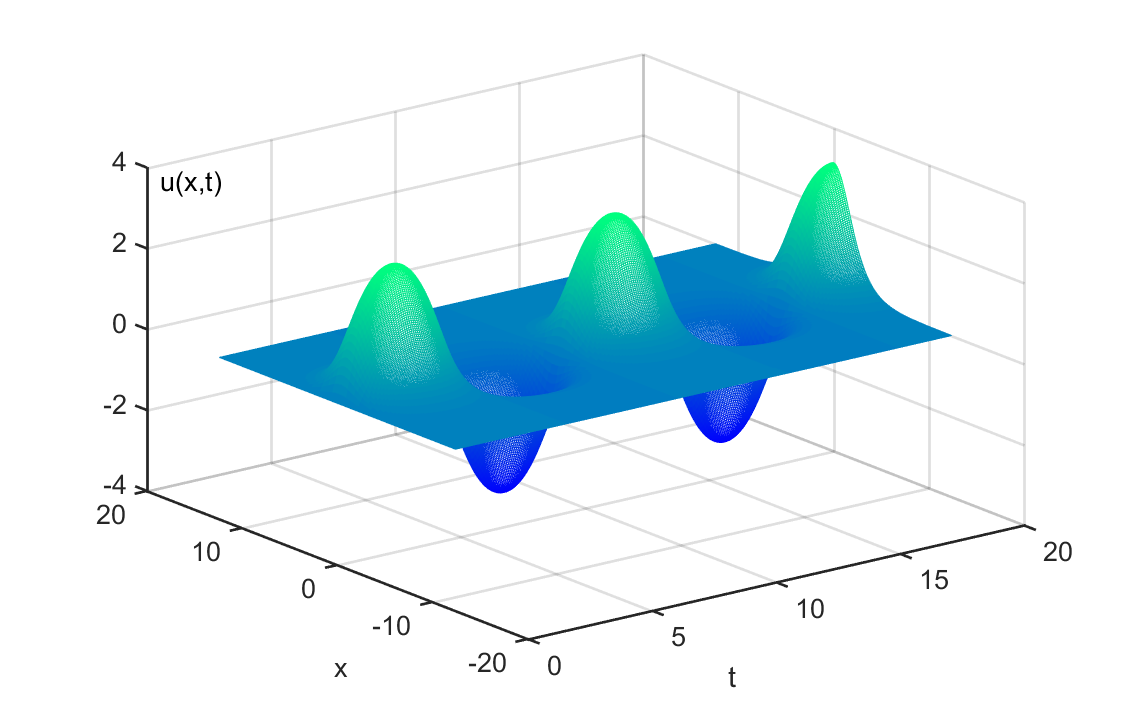}}
	%\captionsetup{justification=centering}
	\caption{Numerical
		solution for Sine-Gordon equation \eqref{SGE-1}.}\label{SGE-2}
\end{figure}

First, we test the $\text{S}^3\text{d}$ method on a small dataset selected randomly from the synthetic data. We construct a dictionary consisting of $m=25$ basis functions. The basis functions contain any combination of the state, $u(x,t)$, and the corresponding derivatives up to the third-order. In addition, we include periodic trigonometric functions such as, $\sin(u), ~\cos(u),~ \sin(u)\sin(u), \sin(u)\cos(u),~\cos(u)\cos(u)$ in our dictionary. From the full dataset, the derivatives we include among the candidate terms can be approximated using the second-order finite difference method. Then, we randomly sub-sample 10 data points from the domain $[-2.6935, 2.1112]$  $\times$ $[0.3922, 4.2353]$, in which a wave package forms. Finally, when applying $\text{S}^3\text{d}$, we correctly discover the Sine-Gordon equation with an accuracy of MSE=0.0706\%, STD = 0.0895\%, which is shown in Table~\ref{SGE-3}. To check the criticality of the subsample domain for our discovery results, we conduct the same experiments on other domains where the wave package does not form. Our results show that, indeed, the subsampled domain is essential for Eq.~\eqref{SGE-1}.

Next, we perform a sensitivity analysis of $\text{S}^3\text{d}$ on noisy data. We choose to add 1\% Gaussian
noise into our synthetic data and compare this noisy data with the original one, and find a RMSE of $Err= 0.0084$. It is observed from Table \ref{SGE-3} that more data points are required to correctly identify the underlying dynamics. In contrast with the noise-free case, we consider a domain in which the second wave package forms and randomly sub-sample 50 data points out of the $131072$ ($512\times 256=131072$) data points. Our results show that disturbance from noise results in a small loss of precision.

This example with the Sine-Gordon equation highlights the $\text{S}^3\text{d}$ method's ability to discover PDEs with only a small sample of the system's states. This example also shows that small samples including essential information regarding the system dynamics play key roles in detecting the underlying trigonometric function ($\sin(u)$).

\begin{table}[ht]
	\centering
	\caption{$\text{S}^3\text{d}$ Method for Sine-Gordon equation.}\label{SGE-3}
	\begin{tabular}{!{\vrule width1.0pt}c|!{\vrule width1.0pt}c|!{\vrule width1.0pt}c|!{\vrule width1.0pt}c|!{\vrule width1.0pt}c|}
		\Xhline{1.0pt}
		& points &\tabincell{c}{$u_{tt}=u_{xx}-\sin(u)$}&mean(err)$\pm$std(err)\\
		\Xhline{1.0pt}
		\tabincell{c}{Identified PDE\\(no noise)}&10 &\tabincell{c}{$u_{tt}=0.9999u_{xx}-0.9987\sin(u)$}&  0.0706\%$\pm$0.0895\%\\
		\Xhline{1.0pt}
		\tabincell{c}{Identified PDE\\(with noise)}&50 &\tabincell{c}{$u_{tt}=0.9920u_{xx}-0.9982\sin(u)$}& 0.4874\%$\pm$0.4407\%\\
		\Xhline{1.0pt}
	\end{tabular}
\end{table}

\subsubsection{Fisher's equation}

For our second example, we consider the Fisher's equation:
\begin{eqnarray}\label{fish}
u_t&=&\alpha u-\beta u^2+d u_{xx},
\end{eqnarray}
which governs bacterial population dynamics \citep{Ozis} as an one-dimensional reaction diffusion model.
The positive constant, $\alpha$, denotes growth rate, $\beta$ is the competition parameter and $d$ is the diffusion coefficient. Fisher's equation has a long-standing history in mathematical modeling of propagation phenomena in
distributed dissipative systems.

We consider the initial condition with a flat roof in the middle \citep{Mit}
\begin{eqnarray*}
	u_0(x)=\begin{cases}
		e^{10(x+1)},~~~x<-1,\\
		1,~~~~~~~~~-1\leq x \leq 1,\\
		e^{-10(x-1)},~~~x>1,
	\end{cases}
\end{eqnarray*}
and set coefficients $\alpha=\beta=1$ and $d=0.1$ in Eq.~\eqref{fish}. The numerical solution, $u$, is computed using a difference scheme
with spatial step length, $\triangle x=0.06$, and time step length, $\triangle t=0.01$.
\begin{figure}[!h]
	\centering \subfigure{\includegraphics[scale=.4]{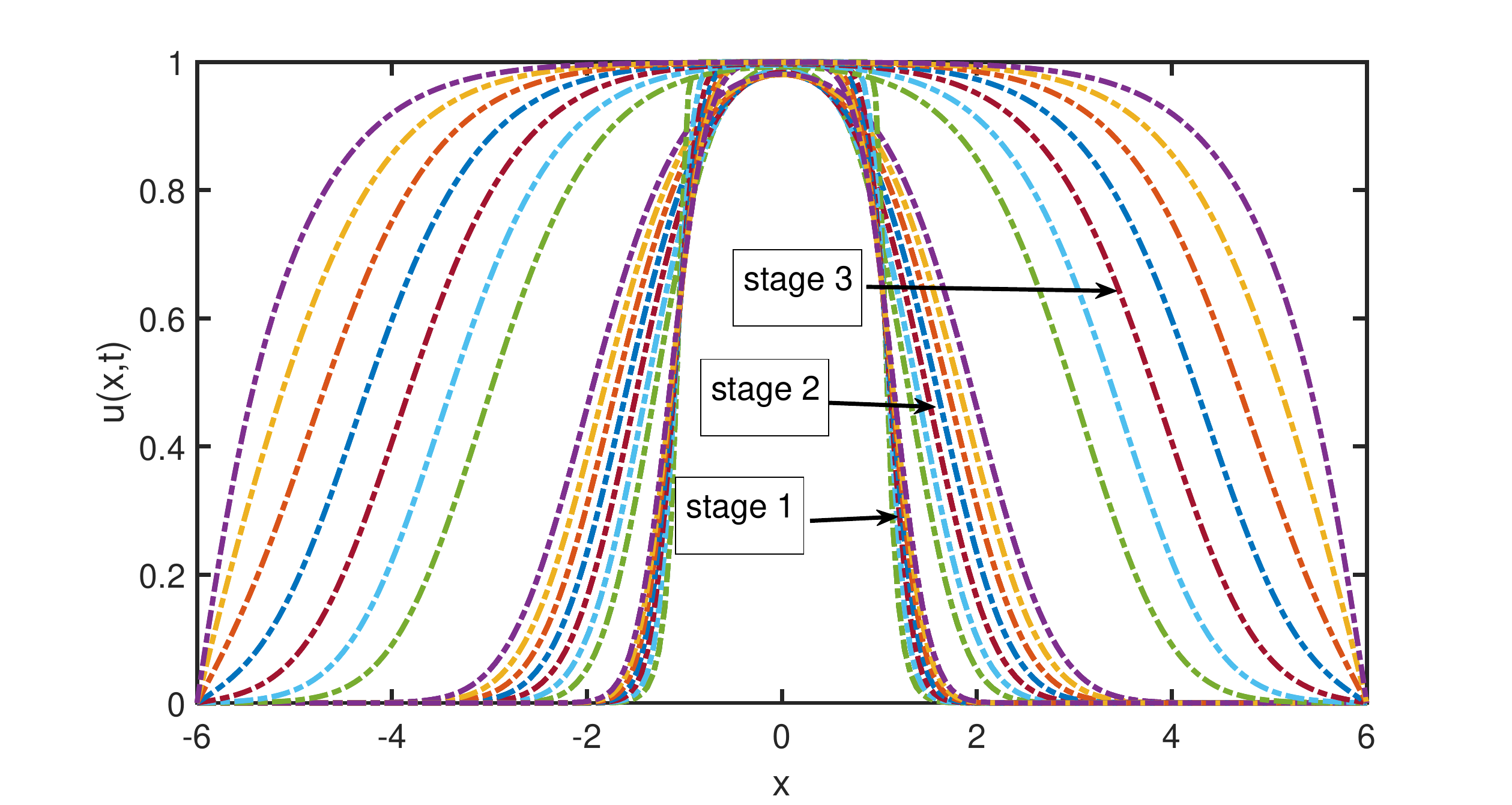}}
	%\captionsetup{justification=centering}
	\caption{Numerical solution to Eq.~\eqref{fish}. Shown is the contour plots of Stage $1$ from $t=0$ to $t=0.4$ with $\triangle t=0.06$; the contour plots of Stage $2$ from $t=1$ to $t=3$ with $\triangle t=0.3$; the contour plots of Stage $3$ from $t=5$ to $t=10$ with $\triangle t=0.8$.}\label{fish1}
\end{figure}
In Fig.~\ref{fish1}, we plot the evolution of the solution, $u$. We observe that, initially, both the reaction term, $u(1-u)$, and the diffusion term, $u_{xx}$, have contours that round around the edges. As time goes on, the whole contour reduces and at the finial stage, the limiting contour appears. The collected data is appropriate for identification since the limiting wave fronts and the limiting speed strictly depend on the system. We store the data  into snapshot matrix $U\in \mathbb{R}^{201\times 1000}$.

With results shown in Table \ref{ifish},
\begin{table}[ht]
	\centering
	\caption{$\text{S}^3\text{d}$ Method for Fisher's equation.}\label{ifish}
	\begin{tabular}{!{\vrule width1.0pt}c|!{\vrule width1.0pt}c|!{\vrule width1.0pt}c|!{\vrule width1.0pt}c|!{\vrule width1.0pt}c|}
		\Xhline{1.0pt}
		& points &\tabincell{c}{$u_t=u-u^2+0.1u_{xx}$}&mean(err)$\pm$std(err)\\
		\Xhline{1.0pt}
		\tabincell{c}{Identified PDE \\(no noise)} & 10000&\tabincell{c}{$u_t=0.9998u -1.0002u^2+0.0997u_{xx}$}&0.0960\%$\pm$0.1342\%\\
		\Xhline{1.0pt}
		\tabincell{c}{Identified PDE \\(with 1 \% noise)}& 10000 &\tabincell{c}{$u_t=1.0053u -1.0079u^2+0.0971u_{xx}$}& 1.4104\%$\pm$1.3007\%\\
		\Xhline{1.0pt}
	\end{tabular}
\end{table}
we chose $m=16$ basis functions as our model class, specifically and attempted to discover the few active terms in the underlying dynamics from our chosen basis functions, $1,	u,	u^2,	u^3,	u_{x},	,u_{xx},$ $	u_{xxx},	uu_{x},	uu_{xx}, uu_{xxx},	u^2u_{x},	u^2u_{xx},	u^2u_{xxx},	$ $u^3u_{x},	u^3u_{xx},	u^3u_{xxx}$. We compute the derivatives using the explicit difference scheme. Furthermore, instead of specifying
the re-sampling region, we randomly take $10000$ out of total $201000$ data points. Our method exactly identifies the main features $u, u^2, u_{xx}$ from the random samples. The computed MSE and STD between the identified parameter and the true parameters is very small, as shown in Table \ref{ifish}.

We continue to test our method with noisy data. In this case, we must handle the error due to the numerical solution, in addition to the measurement error due to noise. We compute the RMSE, $Err=0.0042$ from the noisy data. The RMSE is enhanced when computing the derivatives using the second-order difference method. For the added noise level, we measure the RMSE for the derivatives $u_t, u_{xx}$ as $Err=0.3023$ for term $u_t$, $Err=2.96$ for term $u_{xx}$. To reduce the RMSE, we choose to, instead, approximate the derivatives using the polynomial interpolation method. Compared to the FD method, the measured RMSE for polynomial interpolation method is much smaller, with $Err= 0.0523$ for term $u_t$, $Err=0.0888$ for term $u_{xx}$.
We choose to randomly sub-sample $10000$ data points from the noisy data, just as we did in the noiseless case. As shown in Table \ref{ifish}, with these design choices, we discover Eq.~\eqref{fish} with relatively good accuracy.

This example with Eq.~\eqref{fish} shows that both having an initial condition that stimulates the important features of a dynamical equation and having an accurate method for estimating the derivatives play important roles in the discovery process. In these two experiments, our method permits the use of random samples from the full dataset, similar to the Sine-Gordon equation. However, for Eq.~\eqref{fish}, we randomly sub-sample the data from the entire domain.

\subsubsection{Korteweg-de Vries equation}
Our third example is with the KdV equation with the form,
\begin{eqnarray}\label{KDV-1}
u_t&=&-uu_x-\epsilon u_{xxx},
\end{eqnarray}
where $u$ represents the height of wave at position $x$ and time $t$, the nonlinear term $uu_x$ represents steepening of the wave, the linear dispersive term $u_{xxx}$ represents the spreading of the wave, and $\epsilon$ is a non-zero real constant. The KdV equation describes the unidirectional propagation of shallow water waves over a flat bottom, which theoretically explains the stability of the solution in the experiments of Scott Russell in $1834$ \citep{KdV}. In this example, we show the $\text{S}^3\text{d}$ method on the two interacting soliton case. The interaction between different solitary waves yields particle-like behavior, which is an interesting phenomenon in the shallow water wave problem. We set $\epsilon=0.000484$ and try to validate $\text{S}^3\text{d}$'s ability to discover the KdV equation with small $\epsilon$.

We use the high-order compact difference method in \citep{Lichen} to simulate the collision of a double soliton.
The corresponding initial values are specified as follows:
\begin{eqnarray*}
	u_0(x)=3c_1\text{sech}^2(a_1(x-x_1))+3c_2\text{sech}^2(a_2(x-x_2)),
\end{eqnarray*}
with $c_1=0.3, c_2=0.1, x_1=0.4,x_2=0.8,a_1=0.5\sqrt{c_1/\epsilon}, a_2=0.5\sqrt{c_2/\epsilon}$. Here, the physical domain is  $x\in [0, 2]$, which is uniformly discretized into 255 equal spaces with a spatial grid, $\triangle x=0.0078$. At the outset ($t=0$), we place two solitary waves of different amplitudes along the x-axis, and both of them move in the left direction. We run the simulation until  $t=2$ with time steps of $\triangle t=0.0015$. Fig.~\ref{KDV-2} shows the time evolution of $u(x,t)$ over a range of time values $[0,2]$. We observe that the soliton with higher amplitude travels at a faster speed before interacting with the other soliton. After the two solitons separate, the two solitons still preserve their original amplitude. We organize the data into the snapshot matrix, $U\in \mathbb{R}^{256\times 1301}$.
\begin{figure}[h]
	\centering \subfigure{\includegraphics[scale=.3]{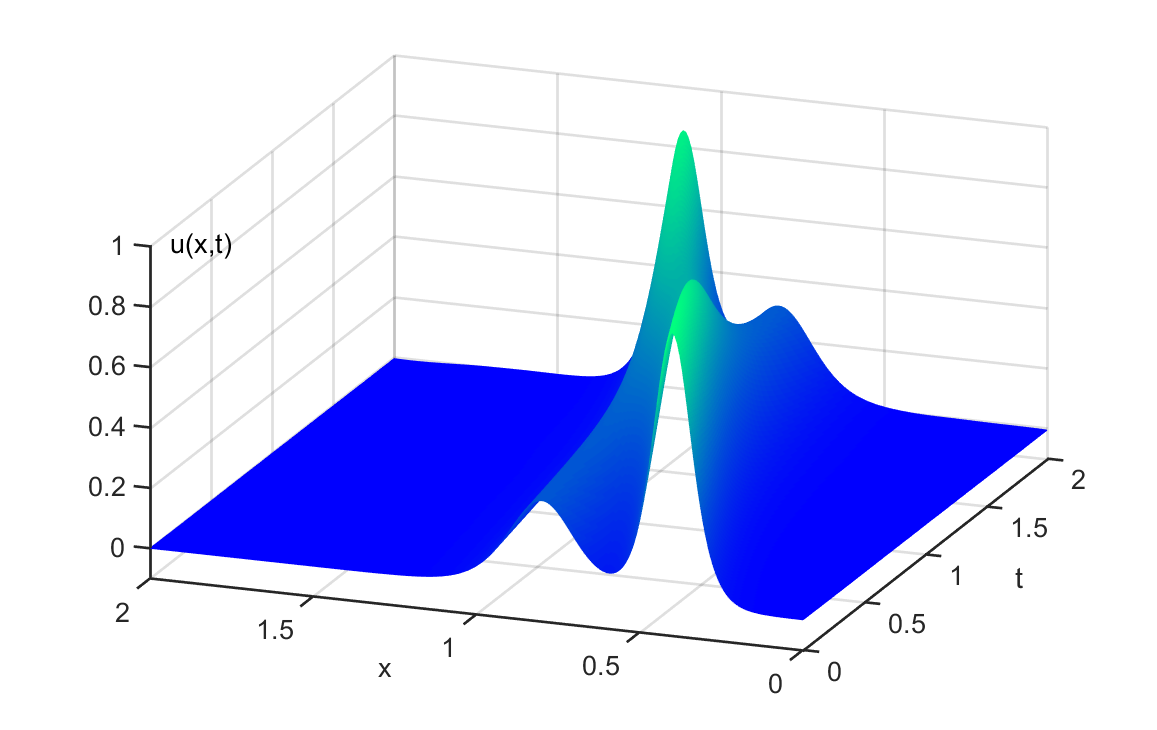}}
	%\captionsetup{justification=centering}
	\caption{Numerical solution for Korteweg-de Vries equation \eqref{KDV-1}.}\label{KDV-2}
\end{figure}

Using snapshot matrix, $U$, we first approximate the time derivatives and the spatial derivatives in a dictionary matrix consisting of $m=12$ basis functions of the form, $1,u,u^2,u_x,u_{xx},u_{xxx}, uu_x,uu_{xx},uu_{xxx},u^2u_{x},u^2u_{xx},u^2u_{xxx}$. We approximate the time derivatives using the second order central difference scheme. To test $\text{S}^3\text{d}$ on a small dataset, we subsample 10000 data points out of the full 333056 data points.
We show the interaction between these two solitons in Fig.~\ref{KDV-2}. To guarantee simultaneous, equal representation of the two solitons in the subsampled data, we subsample the 10000 data points on the domain, $[0.3922,0.7765]\times[0,0.3062]$. Table \ref{KDV-3} lists the discovery results. We see that the $\text{S}^3\text{d}$ method, only using a small portion of data, can effectively identify the KdV equation with $MSE=0.0855\%$ and $STD=0.0145\%$.

In Table \ref{KDV-3}, we also show our results in the noisy case. After adding 1\% Gaussian noise into the original dataset, we compute a $RMSE=0.002$ between the synthetic noisy data and the original one. This error
makes it difficult to approximate the derivatives. To address this error, we use the polynomial approximation with good noise immunity to estimate our derivatives. Then, by increasing the number of data points (18750 of full data points), we can correctly discover the KdV equation and are able to detect parameters that highly coincide with the true values, yielding $MSE=0.9982\%$ and $STD=0.9995\%$.
\begin{table}[ht]
	\centering
	\caption{$\text{S}^3\text{d}$ Method for Korteweg-de Vries equation.}\label{KDV-3}
	\begin{tabular}{!{\vrule width1.0pt}c|!{\vrule width1.0pt}c|!{\vrule width1.0pt}c|!{\vrule width1.0pt}c|!{\vrule width1.0pt}c|}
		\Xhline{1.0pt}
		& points &\tabincell{c}{$u_t=-0.000484u_{xxx}-uu_x$}&mean(err)$\pm$std(err)\\
		\Xhline{1.0pt}
		\tabincell{c}{Identified PDE\\(no noise)}&10000 &\tabincell{c}{$u_{t}=-0.000484u_{xxx}-0.999247uu_x$}& 0.0855\%$\pm$ 0.0145\%\\
		\Xhline{1.0pt}
		\tabincell{c}{Identified PDE\\(with noise)}&18750 &\tabincell{c}{$u_{t}=-0.000483u_{xxx}-0.982950uu_x$}&  0.9982\%$\pm$ 0.9995\%\\
		\Xhline{1.0pt}
	\end{tabular}
\end{table}

\subsubsection{FitzHugh-Nagumo equation}
For our fourth example, we discover the FitzHugh-Nagumo model with two components in the following form:
\begin{eqnarray}\label{FN}
\begin{aligned}
u_t&=d\cdot u_{xx}+f(u,\alpha)-w,\cr
w_t&=\beta u-\gamma w,
\end{aligned}
\end{eqnarray}
where, $u(x,t)$ denotes the transmembrane electrical potential at position $x$ and time $t$, and $w$ denotes the likelihood that a particular class of ion channel is open \citep{Hoffman}. The diffusion coefficient, $d$, represents the electrical conductivity. The function, $f(u,\alpha) \triangleq u(u-\alpha)(1-u)$, represents the reactive properties of the medium with reactive coefficient, $\alpha=0.2$. We are particularly interested in the minimally reactive coefficients, $\gamma=0.001 \ll 1$, $\beta=0.002 \ll 1$, and diffusion coefficient, $d=1$. Eq.~\eqref{FN} is a simplified version to the Hodgkin-Huxley model for nerve conduction, which generally can be used to model activation and deactivation dynamics of a spiking neuron \citep{FN-1, FN-2}.

First, we collect the data from a numerical solution to Eq.~\eqref{FN} with the following initial condition: $u_0=\exp(-x^2),~w_0=0.2\exp(-(x+2)^2)$ and zero boundary conditions. We compute the numerical solution using the finite element method. We create a uniform spatial grid from the interval of interest, $x\in [-15, 15]$, by dividing the interval into $512$ equal subintervals with a fine mesh size of $\triangle x=0.0587$. Then, we discretize the spatial derivative terms in Eq.~\eqref{FN} \citep{Ske}. We integrate the resulting ODE to obtain a finite element solution to Eq.~\eqref{FN}. Fig.~\ref{figure:FN-1} shows the evolution of state variables $u$ and $w$. The sampled data is organized in two data matrices, $U, W\in \mathbb{R}^{512 \times 401}$, with the time step size, $\triangle t=0.05$.
\begin{figure}[h]
	\centering
	\subfigure{\includegraphics[scale=.28]{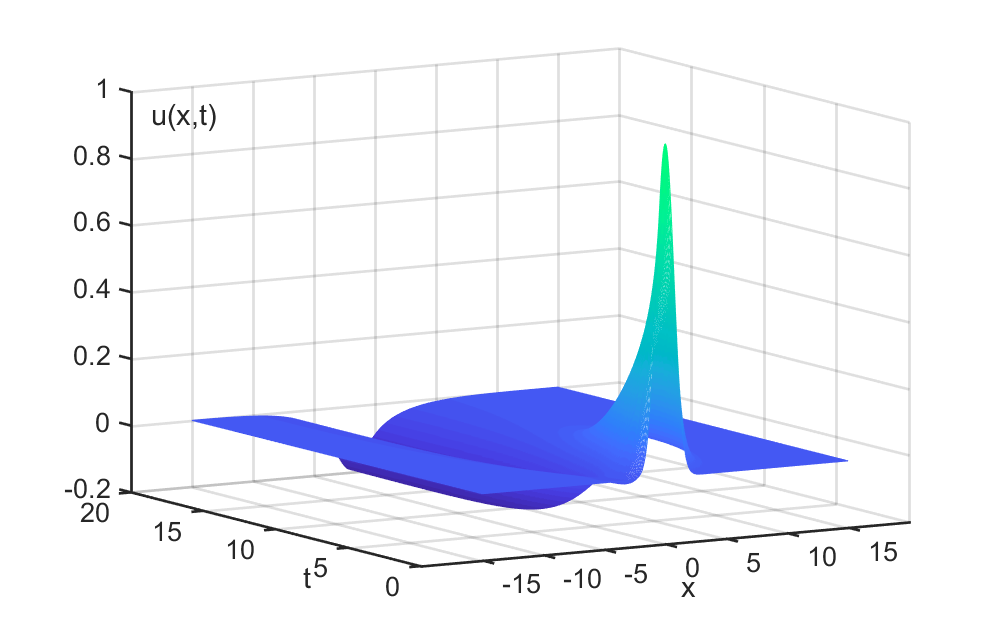}}
	\subfigure{\includegraphics[scale=.28]{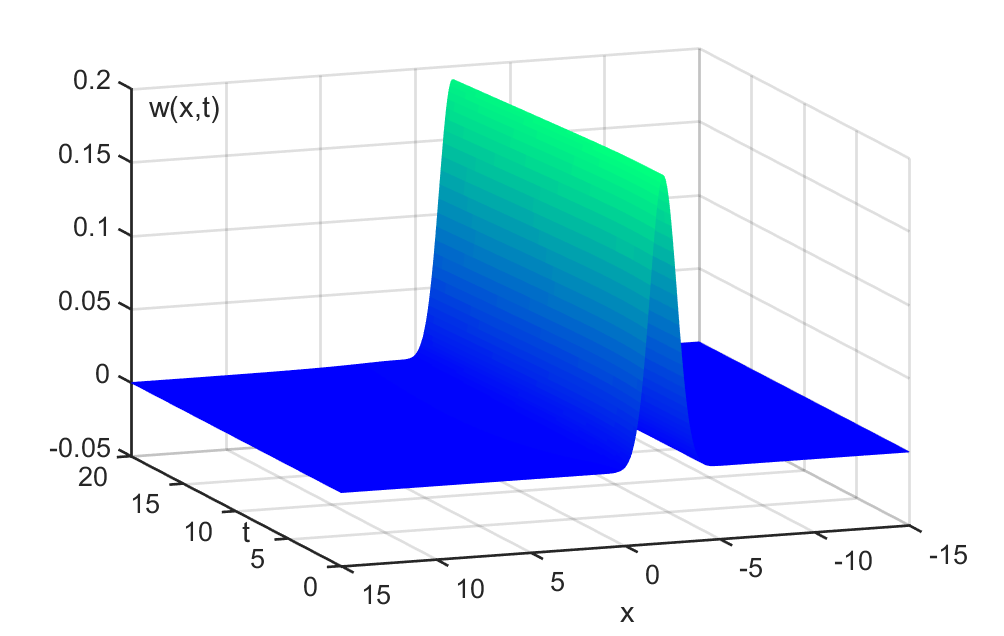}}
	%\captionsetup{justification=centering}
	\caption{The numerical solution to the FitzHugh-Nagumo equation \eqref{FN} with $\gamma=0.001$, $\beta=0.002$.}  \label{figure:FN-1}
\end{figure}

Next, we test the ability of the $\text{S}^3\text{d}$ method to identify  Eq.~\eqref{FN} from small datasets randomly selected from the full synthetic dataset. For this, we sub-sample the collected data from time $t_1=0.25$, to time $t_2=5.2$, with $100$ spatial points. We use a model class consisting of $m=25$ basis functions for our candidate terms, including the derivative of $u$ up to the third order. We approximate the derivative terms in the candidate dictionary using an explicit scheme. By applying the $\text{S}^3\text{d}$ algorithm on the subsampled $10000$ data points, we correctly identify Eq.~\eqref{FN} with a relatively good accuracy.
The results is summarized in Table \ref{iFN}.
Note that the dynamics in the selected region incorporate the interaction between the reactive term and the diffusion term in Eq.~\eqref{FN}, enabling our method to discover the Eq.~\eqref{FN} from the subsampled data points.
\begin{table}[ht]
	\centering
	\caption{$\text{S}^3\text{d}$ discovers the FitzHugh-Nagumo equation.}\label{iFN}
	\begin{tabular}{!{\vrule width1.0pt}c|!{\vrule width1.0pt}c|!{\vrule width1.0pt}c|!{\vrule width1.0pt}c|!{\vrule width1.0pt}c|}
		\Xhline{1.0pt}
		& points &\tabincell{c}{$u_t=d\cdot \bigtriangleup u+ u(u-\alpha)(1-u)-\omega$\\$\omega_t=\beta u-\gamma \omega$}&mean(err)$\pm$std(err)\\
		\Xhline{1.0pt}
		\tabincell{c}{Identified PDE \\(no noise)}& \tabincell{c}{10000~(u)\\10000~(w)} &\tabincell{c}{$u_t=-0.1991u+1.2004u^2$\\$-1.0031u^3+1.0023u_{xx}-1.0005\omega$\\$\omega_t=-0.000993\omega+0.001999u$}&0.2649\%$\pm$0.2462\%\\
		\Xhline{1.0pt}
		\tabincell{c}{Identified PDE\\ (with 1\% noise)} & \tabincell{c}{18000~(u)\\27000~(w)} &\tabincell{c}{$u_t= -0.1981u+1.1762u^2$\\$-0.9499u^3+1.0025u_{xx}-0.9920\omega$\\$\omega_t=-0.001004\omega+0.001952u$}& 1.6795\%$\pm$1.6690\%\\
		\Xhline{1.0pt}
	\end{tabular}
\end{table}

To illustrate the robustness of the $\text{S}^3\text{d}$ method, we add $1\%$ Gaussian noise to the simulated data. The synthetic noisy data is compared with the original simulation data at each of the grid points, and the measured RMSE is $Err(y)=0.000735$ for $u$ and $Err(y)=0.000382$ for $w$. With the noisy data, we perform the same experiment as we did for the data without noise. In the experiment, we sub-sample the noise-contained data from time $t_1=0$ to time $t_2=8.95$ with $100$ spatial points.
%since we need more data to discover Eq.~\eqref{FN} from noisy data.
In addition, we find that the finite difference scheme is sensitive to the noise. Instead, we propose using polynomial interpolation to approximate the derivatives. For different data points and the same dictionary, the proposed $\text{S}^3\text{d}$ method correctly identifies Eq.~\eqref{FN} (see Table \ref{iFN}).

Finally, for an even more intensive study of our method, we further test the proposed $\text{S}^3\text{d}$ method on very small parameters. We take $\gamma=0.0001 $, $\beta=0.0002 $ and
plot the evolution of the computed solution $u$  and $w$ in Fig.~\ref{figure:FN-2}.
\begin{figure}[ht]
	\centering
	\subfigure{\includegraphics[scale=.28]{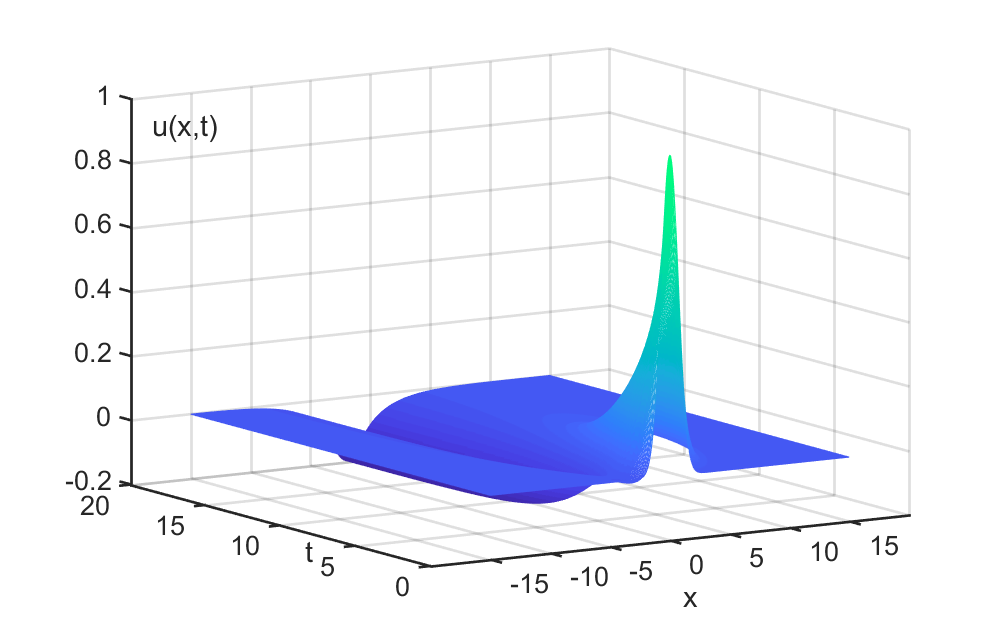}}
	\subfigure{\includegraphics[scale=.28]{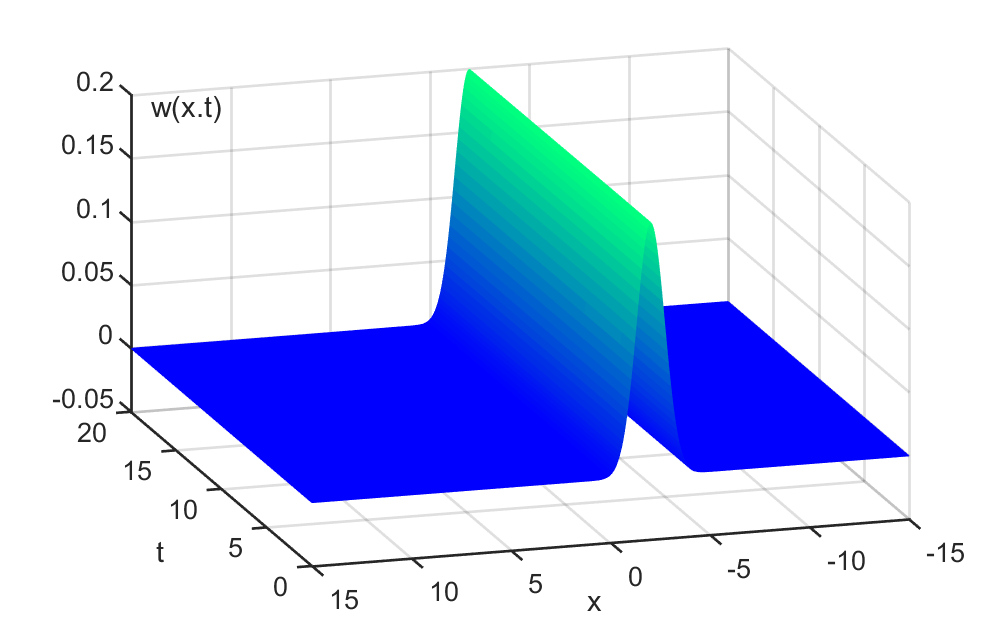}}
	%\captionsetup{justification=centering}
	\caption{The numerical solution to the FitzHugh-Nagumo equation \eqref{FN} with $\gamma=0.0001$, $\beta=0.0002$.}  \label{figure:FN-2}
\end{figure}
We observe that compared with the dynamics in Fig.~\ref{figure:FN-1}, the smaller coefficients produce a slightly different behavior. Despite the different behavior, our method is able to effectively infer the small parameters $\beta$ and $\gamma$ as $\beta=0.0002$ and $\gamma=0.000099$.

\subsubsection{Nonlinear Schr\"{o}dinger's equation}
Now, we showcase the discovery process with S$^3$d on collected data that is in complex domain. To deal with the complex data, we extend the S$^3$d method to the system with complex variables, as introduced in Appendix \ref{appendix:COM}. We consider the following nonlinear Schr\"{o}dinger's equation (also called the Gross-Pitaevskii equation) for condensed neutral atoms in an harmonic trap \citep{Mark}:
\begin{eqnarray}\label{S-1}
iu_t&=&-\frac{1}{\epsilon}(|u|^2-1) u-\epsilon u_{xx},
\end{eqnarray}
where $u(x,t)$ is the Bose-Einstein condensate (BEC) wave function. In mean field theory, Eq.~\eqref{S-1} is a classical approximation that describes a quantum mechanical nonrelativistic many body system with a two body $\delta$-function interaction, which is difficult to observe. However, we show that the $\text{S}^3\text{d}$ method can extract the dynamical regimes from the data collected from the wave function.

First, we obtain the discrete version of the wave function, $u(x,t)$, using the Fourier spectral method. We consider the initial condition, $u(x,0) = 4x^2\exp(-2x^2)\exp(1i)$, and choose parameter $\epsilon=0.3$. We perform the simulation on the domain, $[-\pi, \pi]\times[0, 8]$, with time step, $\triangle t=0.016$, and spatial step, $\triangle x=0.0123$, such that $n_t=501$ is the number of measurements form the snapshot matrices, $U\in \mathbb{R}^{512\times 501}$. Fig.~\ref{S-2} shows the evolution of the (condensate) wave function. Note, the solution has sufficiently variability in both space and time.
\begin{figure}[h]
	\centering \subfigure{\includegraphics[scale=.35]{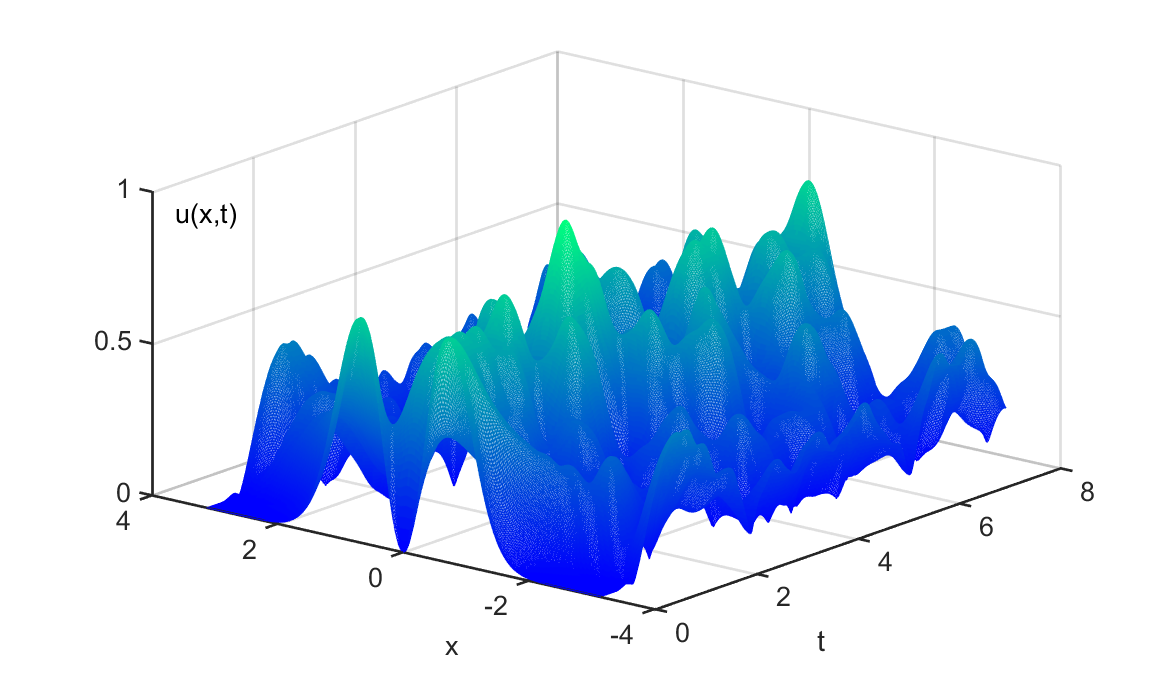}}
	%\captionsetup{justification=centering}
	\caption{Numerical solution for nonlinear  Schr\"{o}dinger's equation \eqref{S-1}.}\label{S-2}
\end{figure}
The identified results from the synthetic data and synthetic noise data is shown in Table \ref{S-3}. We use $m=40$ basis functions for our candidate terms; these terms include $u, |u|, u_x, u_{xx}, u_{xxx}$ and their  combinatorial terms. Since the wave is variable, we choose to use a high precision method (i.e. the Pade scheme \citep{Lichen} for noise free data and the polynomial interplant for the noisy data case) for estimating the derivatives in the example. Additionally, our experiment shows that the method to estimate the derivatives is critical for the success of discovery process.

\begin{table}[ht]
	\centering
	\caption{$\text{S}^3\text{d}$ Method for  Schr\"{o}dinger's equation.}\label{S-3}
	\begin{tabular}{!{\vrule width1.0pt}c|!{\vrule width1.0pt}c|!{\vrule width1.0pt}c|!{\vrule width1.0pt}c|!{\vrule width1.0pt}c|}
		\Xhline{1.0pt}
		& points &\tabincell{c}{$u_{t}=i\frac{3}{10}u_{xx}+i\frac{10}{3}|u|^2u-i\frac{10}{3}u$}&mean(err)$\pm$std(err)\\
		\Xhline{1.0pt}
		\tabincell{c}{Identified PDE\\(no noise)}&10000 &\tabincell{c}{$u_{t}=i0.3u_{xx}+i3.3333u|u|^2$\\$-i3.3333u$}&  0.0011\%$\pm$0.0007\%\\
		\Xhline{1.0pt}
		\tabincell{c}{Identified PDE\\(with noise)}&10000 &\tabincell{c}{$u_{t}=i0.3013u_{xx}+i3.2947u|u|^2$\\$-i3.3187u$}& 0.6775\%$\pm$0.4162\%\\
		\Xhline{1.0pt}
	\end{tabular}
\end{table}

\subsubsection{Klein-Gordon equation}
Consider the Klein-Gordon equation with a cubic nonlinearity
\begin{eqnarray}\label{Kg}
u_{tt}=u_{xx}-u-u^3,
\end{eqnarray}
which occurs as a relativistic wave equation. The state variable, $u(x,t)$, represents the wave displacement at position, $x$, and time, $t$. The Klein-Gordon equation is related to the Schr\"{o}dinger equation and has many applications such as spin waves or nonlinear optics \citep{Klein-Gordon-1}. We rewrite Eq.~\eqref{Kg} as a first-order system,
\begin{eqnarray*}
	\mathbf{u}_t=A\mathbf{u}+F(\mathbf{u}),
\end{eqnarray*}
with $\mathbf{u}=\left[ \begin{array}{cc}u&u_t\\\end{array}\right]^T$, $F(\mathbf{u})=\left[  \begin{array}{cc} 0& -u-u^3\\\end{array}\right]^T$, and
$
A=\left[ \begin{array}{cc}
0& 1\\
\partial_{xx} & 0
\end{array}
\right]
$.
This first-order system is similar to the Fisher's model in Eq.~\eqref{fish}, but with a different nonlinearity. To identify the Klein-Gordon equation from measurement data, we regard the second derivative of $u$ in time as the output, $Y$.

To generate the data, we use the finite difference method to solve Eq.~\eqref{Kg} with the initial condition, $u_0=x^2$, and with zero boundary condition. We divide the domain, $[0,1]\times [0.003,3.003]$, into a $101\times 1001$ mesh with the step size, $\triangle x=0.01$, and time step, $\triangle t=0.003$. In Fig.~\ref{KG},
\begin{figure}[h]
	\centering \subfigure{\includegraphics[scale=.4]{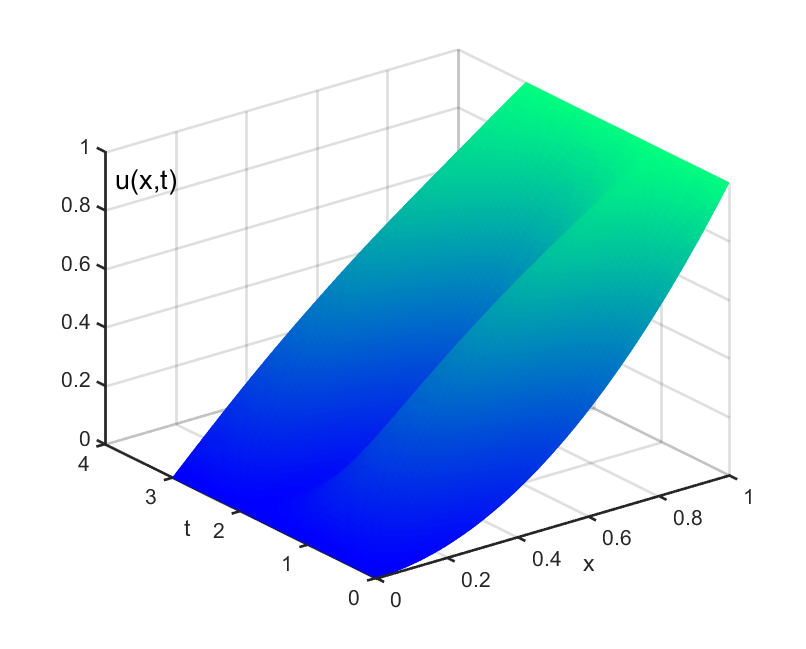}}
	%\captionsetup{justification=centering}
	\caption{Numerical solution for Klein-Gordon equation \eqref{Kg}. }\label{KG}
\end{figure}
we depict the evolution of the numerical solution, $u$. The discrete version of the numerical solution forms the snapshot matrix, $U\in \mathbb{R}^{101 \times 1001}$.

Table \ref{iKG}
\begin{table}[ht]
	\centering
	\caption{$\text{S}^3\text{d}$ Method results for Klein-Gordon equation.}\label{iKG}
	\begin{tabular}{!{\vrule width1.0pt}c|!{\vrule width1.0pt}c|!{\vrule width1.0pt}c|!{\vrule width1.0pt}c|!{\vrule width1.0pt}c|}
		\Xhline{1.0pt}
		& points &\tabincell{c}{$u_{tt}=u_{xx}-u-u^3$}&mean(err)$\pm$std(err)\\
		\Xhline{1.0pt}
		\tabincell{c}{Identified PDE\\(no noise)}&10000 &\tabincell{c}{$u_{tt}=0.9994u_{xx} -0.9995u-0.9998u^3$}& 0.0426\%$\pm$ 0.0220\%\\
		\Xhline{1.0pt}
		\tabincell{c}{Identified PDE\\(with 1\% noise)} & 15500&\tabincell{c}{$u_{tt}=0.9987u_{xx} -0.9833u-1.0303u^3$}& 1.61\%$\pm$ 1.4513\% \\
		\Xhline{1.0pt}
	\end{tabular}
\end{table}
lists our identified results for the noise-free and the noisy cases, respectively. In the noise-free case, we randomly sample $10000$ out of a total of $101\times 1001$ data points in the original dataset. The candidate terms consist of $m=20$ basis functions, including the derivatives of $u$ up to the third order. We use the explicit difference scheme to compute the derivative terms in the dictionary matrix, $\Phi$.
The MSE and STD in Table \ref{iKG} confirm the accuracy of the identified results.

In the noisy case, we add $1\%$ Gaussian noise into the original dataset and compare the synthetic noisy data with the original one. The computed RMSE is $Err(y) =0.003$. The noisy experiment demonstrates that Eq.~\eqref{Kg} is sensitive to noise.

We use the Proper Orthogonal Decomposition (POD) method to denoise our data. In Appendix \ref{appendix:POD}, we introduce the POD method in detail. By taking the first $4$ POD modes to construct the POD basis, i.e., $\Psi\in \mathbb{R}^{101\times 4}$, we obtain the de-noised dataset, $\tilde{U}=\Psi A \in \mathbb{R}^{101\times 1000}$, with the POD coefficients, $A$, for each POD mode, $A=\Psi^T U$. With our de-noised dataset, $\tilde{U}$, we estimate the derivative in the candidate terms using the polynomial approximation, which is proven to be more robust than the explicit difference method. We pay particular attention to the region from time $t_1=0.603$ to $t_2=2.1$ with a compact spatial interval, $[0.3, 0.6]$. The results in Table \ref{iKG} confirms the effectiveness of the use de-noising and discovery methods.

\subsubsection{Kuramoto-Sivashinsky equation}
Consider the  Kuramoto-Sivashinsky (KS) equation subject to periodic boundary condition, $u(x+L, t) = u(x, t)$:
\begin{eqnarray}\label{KS-e}
u_t+uu_x+u_{xxxx}+u_{xx}=0,
\end{eqnarray}
which models a small perturbation, $u(x, t)$, of a metastable planar front or interface. Here, $L$ is the size of a typical pattern scale and all other dimensional parameters are eliminated by rescaling. The KS equation can exhibit chaotic behavior \citep{Hym} and is known to have an inertial manifold \citep{Foi,constan}. This complexity is one of the KS equation's interesting features. In the KS equation, the fluctuations generated by the instability is dissipated by the (stabilizing) fourth-order derivative, $u_{xxxx}$. However, estimating the term from the data is challenging and results in large deviation. In this KS equation example, we exhibit the S$^3$d method's ability to handle PDEs with a fourth derivative term.

To obtain the data, we simulate Eq.~\eqref{KS-e} with the Fourier spectral method in~\citep{spec}. We start with the initial condition on interval, $[0, 32\pi]$:
\begin{eqnarray*}
	u_0(x)=\cos(x/L)(1+\sin(x/L)),
\end{eqnarray*}
with $L=16$. Note, the sampling period is crucial for estimating the derivatives of the candidate terms. To estimate the derivatives, we use a time step of $\triangle t=0.1$ and a spatial step of $\triangle x=0.0491$
with $n_x=2048$ spatial points and $n_t=1001$ time points. We carry out the pseudo-code in \citep{spec} to solve the KS equation.
\begin{figure}[h]
	\centering \subfigure{\includegraphics[scale=.4]{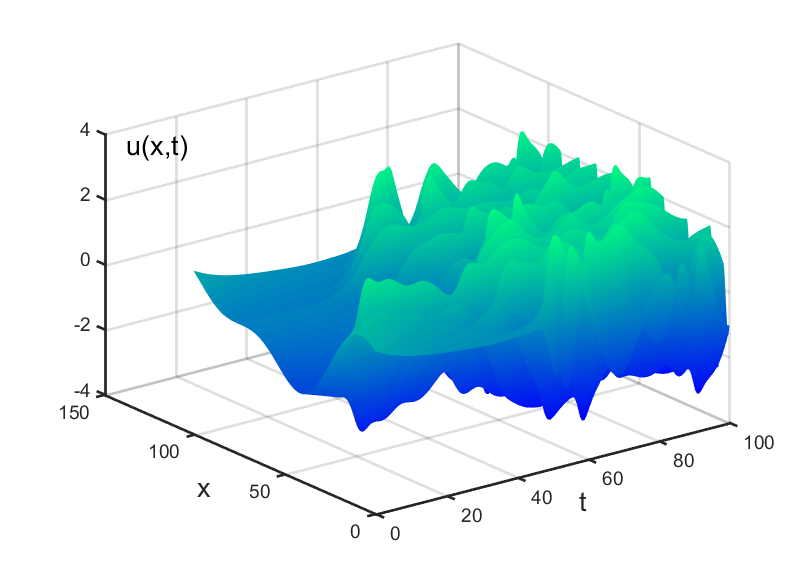}}
	%\captionsetup{justification=centering}
	\caption{Numerical solution for Kuramoto-Sivashinsky equation \eqref{KS-e}.}\label{KS}
\end{figure}
We plot the time evolution of Eq.~\eqref{KS-e} in Fig.~\ref{KS}. The data are stored into the snapshot matrix, $U\in \mathbb{R}^{2048\times 1001}$.

We choose a pool of $m=36$ basis functions, including the derivatives of solution $u$ up to the fifth order. We use the fourth-order compact Pade scheme to approximate the derivative terms. As expected, the resulting derivatives are more accurate than that are obtained using other methods. We further subsample the generated data from time $t_1=40$ to time $t_2=40.6$ using almost all of the spatial points and then we use the sub-sampled $14203$ data points for the sparse identification step. By adjusting the regularization parameter $\lambda$ in our proposed algorithm, our method correctly identifies Eq.~\eqref{KS-e}. We present the discovery results in Table \ref{iKS}.
\begin{table}[ht]
	\centering
	\caption{$\text{S}^3\text{d}$ Method for Kuramoto-Sivashinsky equation.}\label{iKS}
	\begin{tabular}{!{\vrule width1.0pt}c|!{\vrule width1.0pt}c|!{\vrule width1.0pt}c|!{\vrule width1.0pt}c|!{\vrule width1.0pt}c|}
		\Xhline{1.0pt}
		&points &\tabincell{c}{$u_t=-uu_x-u_{xx}-u_{xxxx}$}&mean(err)$\pm$std(err)\\
		\Xhline{1.0pt}
		\tabincell{c}{Identified PDE\\(no noise)}& 14203 &\tabincell{c}{$u_t=-1.0000uu_x-1.0000u_{xx}$\\$-1.0000u_{xxxx}$}&\tabincell{c}{  0.0022\%$\pm$0.0001\%}\\
		\Xhline{1.0pt}
		\tabincell{c}{Identified PDE\\(with 1\% noise)}&59210 &\tabincell{c}{$u_t=-0.9095uu_x-0.9214u_{xx}$\\$-0.9238u_{xxxx}$}&\tabincell{c}{ 8.1781\%$\pm$ 0.7678\%}\\
		\Xhline{1.0pt}
	\end{tabular}
\end{table}

We add $1\%$ Gaussian noise to the original dataset and the measured RMSE is $Err=0.0107$. In this example, we apply the POD method to de-noise our data. By projecting the data onto the $33$ POD modes obtained with the a threshold of $99.99\%$, we generate a new snapshot, $\tilde{U} \in \mathbb{R}^{2048\times 1001}$, from the original snapshot: $\tilde{U}=\Psi A$ with $\Psi$ consisting of $33$ POD modes and the corresponding POD coefficients, $A$. The measured RMSE of $Err=0.0028$ is much smaller than the noisy snapshot. Using the new snapshot, $\tilde{U}$, we use the polynomial interpolation method to estimate the derivatives. We are interested in re-sampling the time domain, $[4, 96]$. We use $59210$ out of the total $2050048$ data points for the sparse identification step. However, with the large dataset, our S$^3$d method suffers from the curse of dimensionality.

To reduce computational time, we apply the SVD technique to the dictionary matrix, $\Phi\in \mathbb{R}^{59210\times36}$, giving us a $72$ dimensional subspace denoted as $\Psi$. We project the dictionary matrix onto the $72$ dimensional subspace such that we obtain the reduced output vector, $\bar{Y}=\Psi^T Y\in \mathbb{R}^{72 \times 1}$, and the reduced dictionary matrix, $\bar{\Phi}=\Psi^T \Phi \in \mathbb{R}^{72 \times 36}$. We show in Table~\ref{iKS} that applying the $\text{S}^3\text{d}$ method to data with de-noising from the POD method and dimensionality reduction works well.

\subsubsection{Navier-Stokes equation}
Consider the incompressible Navier-Stokes (NS) equations for the unsteady two-dimensional flows on torus with vorticity/stream function formulation \citep{NS}:
\begin{eqnarray}\label{NSe}
\begin{aligned}
\omega_t+u\omega_x+v\omega_y&=(\omega_{xx}+\omega_{yy})/Re,\\
\psi_{xx}+\psi_{yy}&=-\omega,
\end{aligned}
\end{eqnarray}
where $u=\psi_y$ represents the horizontal velocity component, $v =-\psi_x $ represents the vertical velocity component, $\omega=u_y-v_x$ represents the vorticity and $Re$ is the Reynolds number (based on the radius of the cylinder and the free-stream velocity, $V_\infty$). The vorticity formulation is attractive for the accurate solution of high Reynolds number planar or axisymmetric NS equations \citep{NS-1}. In this example with the NS equations, the flow field is described by four field quantities, each of which needs to be measured for data.

We use a combination of the Fourier spectral method and the Crank-Nicolson method \citep{NS-1} to solve Eq.~\eqref{NSe} with initial condition in $x\times y\in [0, 2\pi]\times [0, 2\pi]$:
\begin{eqnarray*}
	\omega(x,y)&=&\exp\left(-\frac{1}{5}[x^2+(y+\frac{\pi}{4})^2]\right)
	+\exp\left(-\frac{1}{5}[x^2+(y-\frac{\pi}{4})^2]\right)\cr
	&&-\frac{1}{2}\exp\left(-\frac{2}{5}[(x-\frac{\pi}{4})^2+(y-\frac{\pi}{4})^2]\right).
\end{eqnarray*}
We perform the spatial discretization on a uniform grid with $\triangle x=0.0628$ and $\triangle y=0.0628$. By substituting the Fourier approximation of the solution, $u, v, \omega,$ into Eq.\eqref{NSe}, we further integrate the resultant ODEs using a time step of $\triangle t=0.1$ using the Crank-Nicolson scheme.
In Fig.~\ref{NS},
\begin{figure}[h]
	\centering
	\subfigure{\includegraphics[scale=.5]{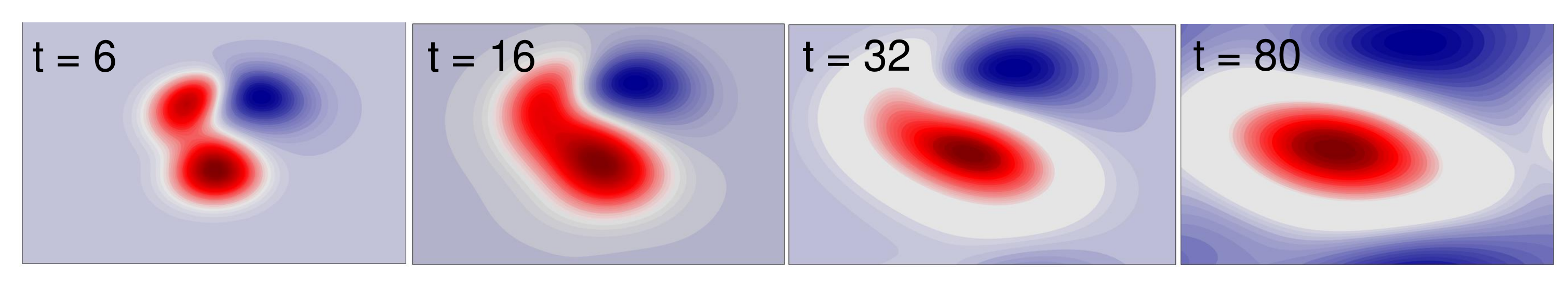}}
	%\captionsetup{justification=centering}
	\caption{Flow field for 2D NS equation \eqref{NSe} on the torus depicted in time $t=6, t=16, t=32$ and $t=80$.}
	\label{NS}
\end{figure}
we plot the 2D NS pseudo-spectral solver on the torus with Reynolds number, $Re=100$. The discrete version of the numerical solution forms the snapshot matrices, $U, V, W\in \mathbb{R}^{100\times 100\times 1001}$.

We report the discovery results in Table \ref{table:ns}.
\begin{table}[h]
	\centering
	\caption{$\text{S}^3\text{d}$ Method for NS equation.}\label{table:ns}
	\begin{tabular}{!{\vrule width1.0pt}c|!{\vrule width1.0pt}c|!{\vrule width1.0pt}c|!{\vrule width1.0pt}c|!{\vrule width1.0pt}c|}
		\Xhline{1.0pt}
		& points &\tabincell{c}{$\omega_t(x,t)=0.01\omega_{xx}+0.01\omega_{yy}$\\$-u\omega_x-v\omega_y$}&mean(err)$\pm$std(err)\\
		\Xhline{1.0pt}
		\tabincell{c}{Identified PDE\\(no noise)}& 10000 &\tabincell{c}{$\omega_t(x,t)=0.01\omega_{xx}+0.01\omega_{yy}$\\$-1.0068u\omega_x-0.9987v\omega_y$}&0.37\%$\pm$0.2305\%\\
		\Xhline{1.0pt}
		\tabincell{c}{Identified PDE\\(with 1\% noise)}& 20000&\tabincell{c}{$\omega_t(x,t)=0.0101\omega_{xx}+0.0097\omega_{yy}$\\$-1.0051u\omega_x-1.0013v\omega_y$}&0.9447\%$\pm$1.3498\%\\
		\Xhline{1.0pt}
	\end{tabular}
\end{table}
The dictionary matrix, $\Phi$, consists of $m=60$ basis functions. The underlying dynamics is described with the three state variables, $u, v, w$. Thus, our dictionary of basis functions include the linear combinations of the variables, $u, v, \omega,u_{x},v_{x},u_{xx}, v_{xx}$. With snapshots, $U, V, W$, we use the explicit difference scheme to approximate the first or second order derivatives in the noise-free case. We see from the discovery results that our method has strong performance with the noise-free dataset, a small MSE, and a small STD. We enumerate the re-sampling region in Table \ref{table:ns}.

In the noisy case, we use the 4th-order compact Pade scheme with the $\text{S}^3\text{d}$ method. We add $1\%$ Gaussian noise to the original snapshot and the resulting RMSE is $Err=0.4823$ for $u$, $Err=0.4338$ for $v$, and $Err=1.7111$ for $w$. We apply the POD method to de-noise our data. We reshape the snapshot $U$ into a $10000\times 1001$ matrix, $\bar{U}$; we do the same shaping for $W$ and $V$. By projecting the data onto 33 POD modes obtained using a threshold of 99.99 fixed, we generate a new snapshot, $\tilde{U} \in \mathbb{R}^{10000\times 1001}$, from the matrix, $\bar{U}$: $\tilde{U}=\Psi A$ with $\Psi$ consisting of 33 POD modes and the corresponding POD  coefficients, $A$. Then, we reverse the process such that snapshots, $\bar{U}, \bar{V}, \bar{W}\in \mathbb{R}^{100\times 100\times 1001}$. The computed RMSE is $0.0056$ for $u$, $0.0051$ for $v$, and $0.0412$ for $w$. We pay special attention to the Pade scheme. With the de-noised dataset, we approximate the first-order derivatives in each grid point and then we compute the second-order derivatives using the Pade scheme. After we compute the dictionary matrix, we correctly discover the NS equation as shown in Table \ref{table:ns}.

\subsection{Discovery of Complex Ginzburg-Landau Equations from Binary-Fluid Convection Experiment}\label{sec:subexperiment}
This section specifies the discovery of the complex Ginzburg-Landau equation (CGLE) from experimental measurements. Traveling-wave (TW) convection in binary fluids is a known, mainstream ansatz technique for studying the physical mechanism of non-equilibrium pattern-forming systems \citep{Cros}. Substantial experiments on TW convection unfolds many different dynamical states of TW. A major challenge in the scientific study of TW is quantitatively understanding and predicting the underlying dynamics. First-principle models based on the CGLE and its variants \citep{CrosPRA,New} have long been used to quantitatively explain experimental observations. We apply the proposed $\text{S}^3\text{d}$ method to see whether we can discover the CGLE  from experimental data alone.

\subsubsection{Data description}
The experimental data comes from an experiment conducted in an annular cell \citep{Voss,paul2}, detailed in Table \ref{Size1}:
\begin{table}[ht!]
	\centering
	\caption{
		Overview of the considered real-world TW data. More experimental details for each dataset can be found in \citep{Voss,paul2}.
		\label{Size1}}
	\begin{tabular}{cccccc}
		\Xhline{1pt}
		Experiment Label &  \tabincell{c}{bifurcation\\ parameter  } & Samples & \tabincell{c}{Sampling \\ interval} & \tabincell{c}{Space \\interval} & Truncation\\
		\Xhline{1pt}
		cal06172/car06172 & $\gamma=9.32\times 10^{-3}$ & $988\times 180$ & 1.5625 & 0.458167 & 228\\
		cal06182/car06182 & $\gamma=4.22\times 10^{-3}$ & $800\times 180 $  & 1.5625 & 0.458167 & 40 \\
		cal06192/car06192 & $\gamma=1.77\times 10^{-3}$ & $1000\times 180 $  & 1.5625 & 0.458167 &  240\\
		cal06212/car06212 & $\gamma=6.38\times 10^{-3}$ & $1000\times 180 $  & 1.5625 & 0.458167 & 240\\
		cal06222/car06222 & $\gamma=12.07\times 10^{-3}$ & $1031\times 180$ & 1.5625 & 0.458167 & 271\\
		cal06242/car06242 & $\gamma=14.03\times 10^{-3}$ & $1000\times 180 $  & 1.5625 & 0.458167 & 240\\
		cal06252/car06252 & $\gamma=16.28\times 10^{-3}$ & $1000\times 180 $  & 1.5625 & 0.458167 & 240\\
		\Xhline{1pt}
	\end{tabular}
\end{table}
\begin{itemize}
	\item The special letter ``cal" and ``car" respectively represent the left-going waves and right-going ones. The files record the experimental data;
	\item Seven bifurcation parameters, $\varepsilon$, scaled by the characteristic time, $\tau_0$ (i.e., $\gamma=\varepsilon\tau_0^{-1}$), represent the same experiments but conducted with different bifurcation parameters;
	\item  All samples used for the analysis of $\text{S}^3\text{d}$ are truncated from the $21^{th}$ sampling point to the $780^{th}$ sampling point, such that a total of $180\times 760$ data points are used for the analysis. Then, we form the snapshot matrices, $U\in \mathbb{C}^{180\times 760}$.
\end{itemize}
We define the measured left-going TW state (e.g. ``cal06172", ``cal06182") as the left-going complex amplitudes, $A_L$, in Eq.~\eqref{GLE}. We do the same for the data for the right-going complex amplitudes $A_R$, such as ``car06172",``cal06182", etc.. Note that the data is being collected periodically. Then, we use the circular Pade scheme to approximate the derivatives of the amplitude, $A_R$ (or $A_L$), from the full dataset (i.e., $180\times 760$).

In Fig.\ref{GLE19}, we show the dynamics of the weakly nonlinear wave packets for $\varepsilon \tau_0=1.77\times 10^{-3}$, which propagate around the system in opposite directions.
\begin{figure}[h!]
	\centering
	\includegraphics[scale=.4]{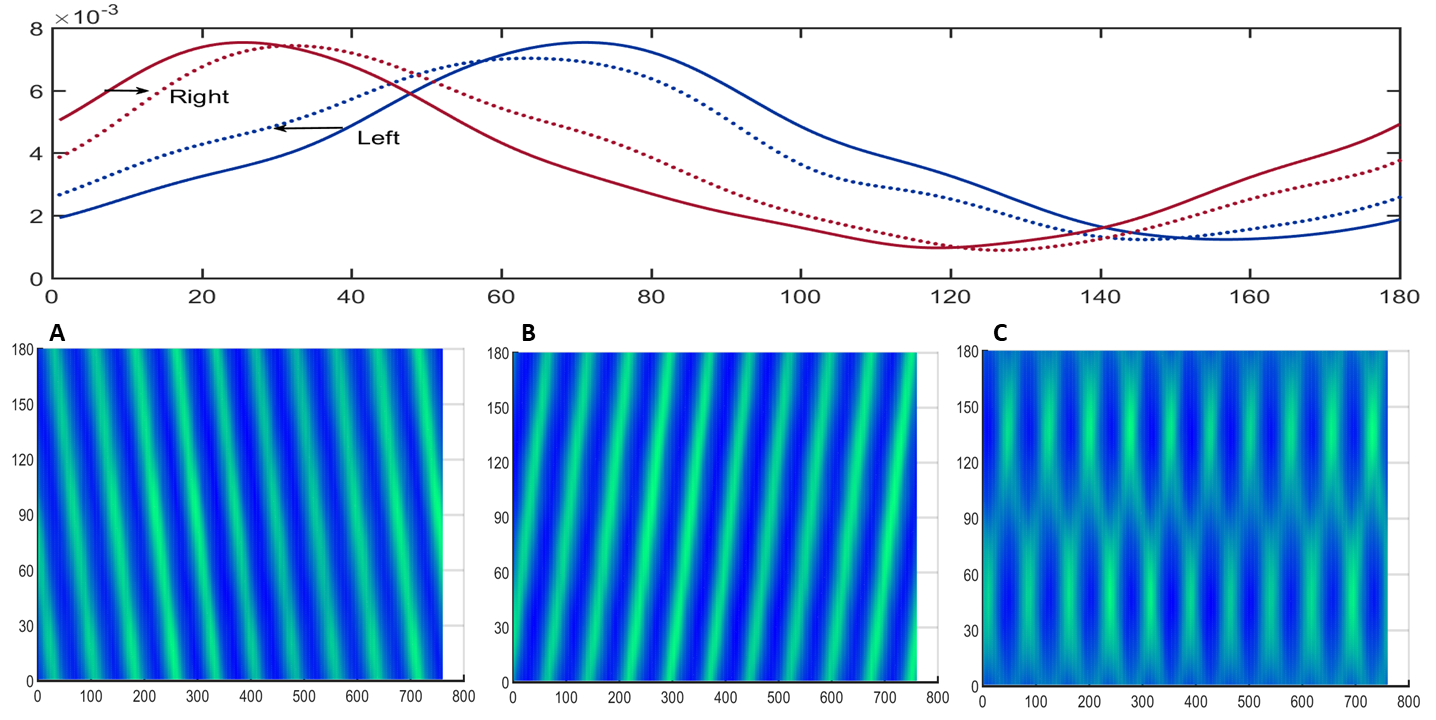}
	\caption{\textbf{The dynamics of the observed TW state for $\varepsilon \tau_0=1.77\times 10^{-3}$.} The top line shows the physical change of the amplitudes of the left-going (blue full line and dashed line) and right-going (red full curve and dashed line) waves components at a particular time and at the next time's data.
		A: The behavior of left-going amplitudes ;
		B: The behavior of right-going amplitudes;
		C: The behavior of the sum of left- and right-going amplitudes.}
	\label{GLE19}
\end{figure}

%\begin{center}
%\textbf{Complex Ginzburg-Landau equation model}
%\end{center}

\subsubsection{Discovery of CGLE}
We summarize the best-fit basis functions and parameters for each database in Table \ref{tab:results}. Here, we illustrate the performance of the $\text{S}^3\text{d}$ method in two steps: first, it determines the key nonlinearities and then recovers the coefficients from experimental data. We introduce $18$ basis functions to construct the dictionary:
$A_{L,R}$,
$|A_{L,R}|$,
$A_{L,R}^2$,
$|A_{L,R}|^2$
$A_{L,R}^3$,
$|A_{L,R}|^3$,
$A_{L,R}^2|A_{L,R}|,$
$A_{L,R}|A_{L,R}|^2$,
$\partial_{x}A_{1,2}$,
$\partial_{xx}A_{1,2}$,
$A_{L,R}\partial_{x}A_{L,R}$,
$A_{L,R}\partial_{xx}A_{L,R}$,
$A_{L,R}^2\partial_{x}A_{L,R}$,
$A_{L,R}^2\partial_{xx}A_{L,R}$,
$A_{L,R}^3\partial_{x}A_{L,R}$,
$A_{L,R}^3\partial_{xx}A_{L,R},$
$|A_{R,L}|^2A_{L,R}$. The above amounts to $3^{rd}$-order Volterra expansions of $A_{L,R}$ and $|A_{L,R}|$. We investigate the performance of the $\text{S}^3\text{d}$ algorithm on the last three sets of data with larger signal-to-noise ratio (e.g. $\varepsilon\tau_0^{-1}\geq 12.07\times 10^{-3}$ ) as discussed in \citep{Voss} in Table \ref{Size1}. We take the maximum iteration number to be $k_{max}=15$ and empirically adjust the regularization parameter, $\lambda$.

In our experiments, we observe that the $\text{S}^3\text{d}$ algorithm yields a stable result (shown in Fig.~\ref{fig-lam}):  for each data set, there exists a critical number of the basis functions before which the fitting-error slowly increases and after which the error begins to grow rapidly. The table in Fig.~\ref{fig-lam} enumerates the critical number for the bifurcation parameter, $\gamma=14.03\times 10^{-3}$; specifically, we discover $5$ key nonlinearities. In addition, the discovered key nonlinearities are consistent for two other datasets with bifurcation parameter $\gamma=12.07\times 10^{-3}$ and bifurcation parameter $\gamma=16.28\times 10^{-3}$.

The following discovered equation has the following form:
\begin{eqnarray}\label{GLE}
\begin{split}
\tau_0(\partial_t+s\partial_x)A_R&=\varepsilon(1+ic_0)A_R+\omega_0^2(1+ic_1)\partial_{x}^2A_R+g(1+ic_2)|A_R|^2A_R\\
&+h(1+ic_3)|A_L|^2A_R,\\
\tau_0(\partial_t-s\partial_x)A_L&=\varepsilon(1+ic_0)A_L+\omega_0^2(1+ic_1)\partial_{x}^2A_L+g(1+ic_2)|A_L|^2A_L\\
&+h(1+ic_3)|A_R|^2A_L.
\end{split}
\end{eqnarray}
In this model (CGLE), $A_R$ (or $A_L$) is the complex amplitude of a right-going (or left-going) wave with group velocity $s$, the parameter $\omega_0$ is a characteristic length scale and $\tau_0$ is a characteristic time which is determined experimentally by measuring the growth rate, $\gamma=\varepsilon \tau_0^{-1}$, at several values of $\varepsilon$ and fitting the slope, $c_0-c_3$ are dispersion coefficients, $g$ is a nonlinear saturation parameter, and $h$ is a nonlinear coupling coefficient which reflects the stabilizing interaction between oppositely propagating traveling-wave components \citep{Voss}. CGLE is frequently used to model non-equilibrium patterning-forming systems.

\begin{figure}[!ht]
	
	\begin{minipage}{0.25\linewidth}
		\centerline{\includegraphics[width=3.8cm]{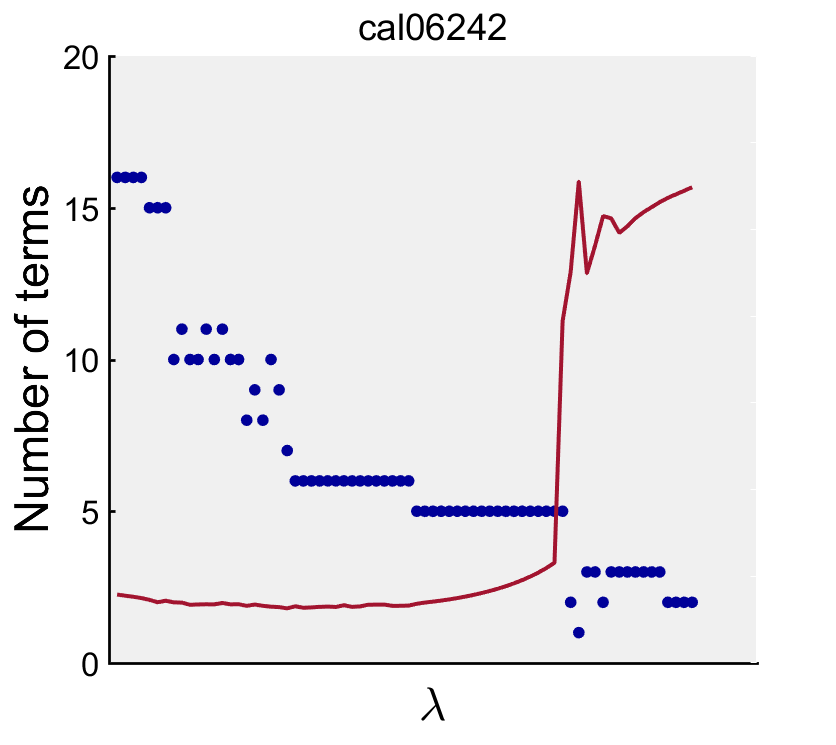}}
	\end{minipage}
	\hfill
	\begin{minipage}{0.9\linewidth}
		\begin{tiny}
			\renewcommand\arraystretch{1.1}
			\begin{tabular}{
					|p{0.0001cm}<{\centering}
					|p{0.01cm}<{\raggedleft}
					|p{0.2cm}<{\centering}
					%|p{0.1cm}<{\raggedleft}
					%|p{0.2cm}<{\centering}
					%|p{0.2cm}<{\centering}
					|p{0.4cm}<{\centering}
					|p{0.4cm}<{\centering}
					|p{0.2cm}<{\centering}
					|p{0.1cm}<{\centering}
					|p{0.2cm}<{\centering}
					|p{0.3cm}<{\centering}
					|p{0.4cm}<{\centering}
					|p{0.4cm}<{\centering}
					|p{0.5cm}<{\centering}
					|p{0.4cm}<{\centering}
					|p{0.5cm}<{\centering}
					|p{0.5cm}<{\centering}
					|p{0.001cm}<{\centering}}
				
				\Xhline{1pt}
				
				1
				& $A$
				& $\cdots$
				%& $|A|$
				%& $A^2$
				%& $|A|^2$
				%& $A^3$
				& $A^2|A|$
				& $A|A|^2$
				& $|A|^3$
				& $A_{x}$
				&$A_{xx}$
				& $AA_{x}$
				& $AA_{xx}$
				&$A^2A_{x}$
				& $A^2A_{xx}$
				&  $A^3A_{x}$
				& $A^3A_{xx}$
				& $|A_R|^2A$
				& \\ \Xhline{1pt}
				
				& \multicolumn{1}{>{\columncolor{bluebell}}l|}{} & &\multicolumn{1}{>{\columncolor{bluebell}}l|}{}&&&\multicolumn{1}{>{\columncolor{bluebell}}l|}{}&\multicolumn{1}{>{\columncolor{bluebell}}l|}{}
				&\multicolumn{1}{>{\columncolor{bluebell}}l|}{}&&\multicolumn{1}{>{\columncolor{bluebell}}l|}{}&&&&
				\multicolumn{1}{>{\columncolor{bluebell}}l|}{} & $7$ \\ \Xhline{1pt}
				
				& \multicolumn{1}{>{\columncolor{bluebell}}l|}{} & &&\multicolumn{1}{>{\columncolor{bluebell}}l|}{}&&\multicolumn{1}{>{\columncolor{bluebell}}l|}{}&\multicolumn{1}{>{\columncolor{bluebell}}l|}{}&&&&&&
				\multicolumn{1}{>{\columncolor{bluebell}}l|}{}&
				\multicolumn{1}{>{\columncolor{bluebell}}l|}{} & $6$\\ \Xhline{1pt}
				
				& \multicolumn{1}{>{\columncolor{frenchlilac}}l|}{} & &&\multicolumn{1}{>{\columncolor{frenchlilac}}l|}{}&&\multicolumn{1}{>{\columncolor{frenchlilac}}l|}{}&\multicolumn{1}{>{\columncolor{frenchlilac}}l|}{}&&&&&&&
				\multicolumn{1}{>{\columncolor{frenchlilac}}l|}{} & $5$\\ \Xhline{1pt}
				
				& \multicolumn{1}{>{\columncolor{frenchlilac}}l|}{} & &&\multicolumn{1}{>{\columncolor{frenchlilac}}l|}{}&&\multicolumn{1}{>{\columncolor{frenchlilac}}l|}{}&\multicolumn{1}{>{\columncolor{frenchlilac}}l|}{}&&&&&&&
				\multicolumn{1}{>{\columncolor{frenchlilac}}l|}{} & $\vdots$ \\ \Xhline{1pt}
				
				& \multicolumn{1}{>{\columncolor{frenchlilac}}l|}{} & &&\multicolumn{1}{>{\columncolor{frenchlilac}}l|}{}&&\multicolumn{1}{>{\columncolor{frenchlilac}}l|}{}&\multicolumn{1}{>{\columncolor{frenchlilac}}l|}{}&&&&&&&
				\multicolumn{1}{>{\columncolor{frenchlilac}}l|}{} & $5$ \\ \Xhline{1pt}
				
				&  & &&\multicolumn{1}{>{\columncolor{bluebell}}l|}{}&&\multicolumn{1}{>{\columncolor{bluebell}}l|}{}&&&&&&&&
				\multicolumn{1}{>{\columncolor{bluebell}}l|}{} & $3$ \\ \Xhline{1pt}
				
				& &&&&&\multicolumn{1}{>{\columncolor{bluebell}}l|}{}&&&&&&&&
				\multicolumn{1}{>{\columncolor{bluebell}}l|}{} & $2$ \\ \Xhline{1pt}
				
			\end{tabular}
		\end{tiny}
	\end{minipage}

	\begin{minipage}{0.25\linewidth}
		\centerline{\includegraphics[width=3.8cm]{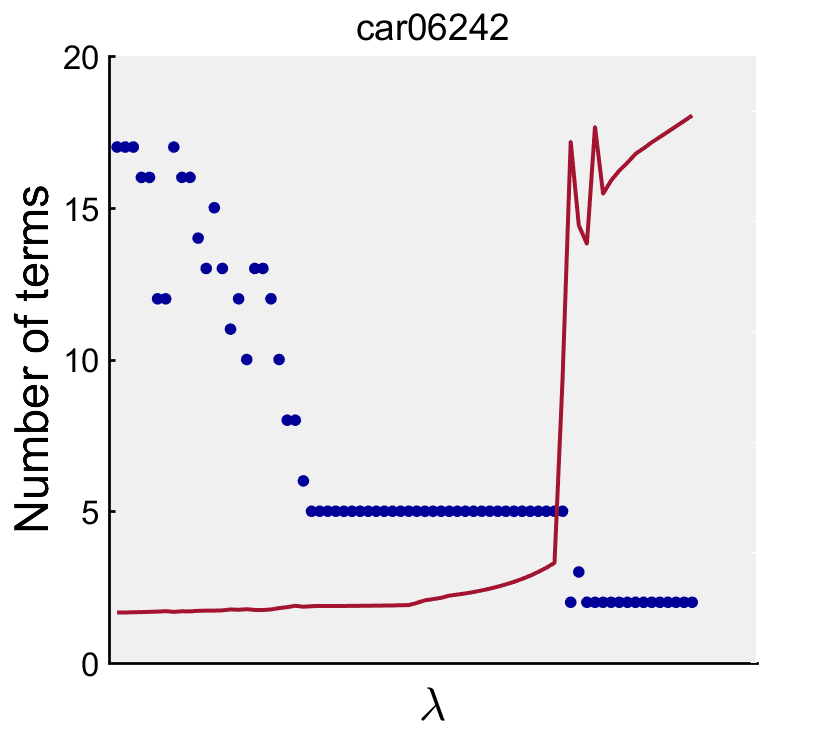}}
	\end{minipage}
	\hfill
	\begin{minipage}{.9\linewidth}
		\begin{tiny}
			\renewcommand\arraystretch{1.1}
			\begin{tabular}{
					|p{0.0001cm}
					|p{0.0001cm}<{\raggedleft}
					|p{0.1cm}<{\raggedleft}
					%|p{0.2cm}<{\centering}
					%|p{0.1cm}<{\raggedleft}
					%|p{0.2cm}<{\centering}
					%|p{0.2cm}<{\centering}
					|p{0.4cm}<{\centering}
					|p{0.4cm}<{\centering}
					|p{0.2cm}<{\centering}
					|p{0.1cm}<{\centering}
					|p{0.2cm}<{\centering}
					|p{0.3cm}<{\centering}
					|p{0.4cm}<{\centering}
					|p{0.4cm}<{\centering}
					|p{0.5cm}<{\centering}
					|p{0.4cm}<{\centering}
					|p{0.5cm}<{\centering}
					|p{0.5cm}<{\centering}
					|p{0.001cm}<{\raggedleft}}
				
				\Xhline{1pt}
				
				1
				& $A$
				& $\cdots$
				%& $|A|$
				%& $A^2$
				%& $|A|^2$
				%& $A^3$
				& $A^2|A|$
				& $A|A|^2$
				& $|A|^3$
				& $A_{x}$
				&$A_{xx}$
				& $AA_{x}$
				& $AA_{xx}$
				&$A^2A_{x}$
				& $A^2A_{xx}$
				&  $A^3A_{x}$
				& $A^3A_{xx}$
				& $|A_L|^2A$
				& \\ \Xhline{1pt}

				& \multicolumn{1}{>{\columncolor{bluebell}}l|}{} & &\multicolumn{1}{>{\columncolor{bluebell}}l|}{}&\multicolumn{1}{>{\columncolor{bluebell}}l|}{}&&\multicolumn{1}{>{\columncolor{bluebell}}l|}{}
				&\multicolumn{1}{>{\columncolor{bluebell}}l|}{}&&\multicolumn{1}{>{\columncolor{bluebell}}l|}{}&&&& \multicolumn{1}{>{\columncolor{bluebell}}l|}{}&
				\multicolumn{1}{>{\columncolor{bluebell}}l|}{} & $8$ \\ \Xhline{1pt}
				
				& \multicolumn{1}{>{\columncolor{bluebell}}l|}{} & &&\multicolumn{1}{>{\columncolor{bluebell}}l|}{}&&\multicolumn{1}{>{\columncolor{bluebell}}l|}{}&\multicolumn{1}{>{\columncolor{bluebell}}l|}{}&
				&\multicolumn{1}{>{\columncolor{bluebell}}l|}{}&&&&&
				\multicolumn{1}{>{\columncolor{bluebell}}l|}{} & $6$\\ \Xhline{1pt}
				
				& \multicolumn{1}{>{\columncolor{frenchlilac}}l|}{} & &&\multicolumn{1}{>{\columncolor{frenchlilac}}l|}{}&&\multicolumn{1}{>{\columncolor{frenchlilac}}l|}{}&\multicolumn{1}{>{\columncolor{frenchlilac}}l|}{}&&&&&&&
				\multicolumn{1}{>{\columncolor{frenchlilac}}l|}{} & $5$\\ \Xhline{1pt}
				
				& \multicolumn{1}{>{\columncolor{frenchlilac}}l|}{} & &&\multicolumn{1}{>{\columncolor{frenchlilac}}l|}{}&&\multicolumn{1}{>{\columncolor{frenchlilac}}l|}{}&\multicolumn{1}{>{\columncolor{frenchlilac}}l|}{}&&&&&&&
				\multicolumn{1}{>{\columncolor{frenchlilac}}l|}{} & $\vdots$ \\ \Xhline{1pt}
				
				& \multicolumn{1}{>{\columncolor{frenchlilac}}l|}{} & &&\multicolumn{1}{>{\columncolor{frenchlilac}}l|}{}&&\multicolumn{1}{>{\columncolor{frenchlilac}}l|}{}&\multicolumn{1}{>{\columncolor{frenchlilac}}l|}{}&&&&&&&
				\multicolumn{1}{>{\columncolor{frenchlilac}}l|}{} & $5$ \\ \Xhline{1pt}
				
				&  & &&&&\multicolumn{1}{>{\columncolor{bluebell}}l|}{}&\multicolumn{1}{>{\columncolor{bluebell}}l|}{}&&&&&&&
				\multicolumn{1}{>{\columncolor{bluebell}}l|}{} & $3$ \\ \Xhline{1pt}
				
				& &&&&&\multicolumn{1}{>{\columncolor{bluebell}}l|}{}&&&&&&&&
				\multicolumn{1}{>{\columncolor{bluebell}}l|}{} & $2$ \\ \Xhline{1pt}
				
			\end{tabular}
		\end{tiny}

	\end{minipage}

	\begin{minipage}{0.52\linewidth}
		\centerline{\includegraphics[width=8.8cm]{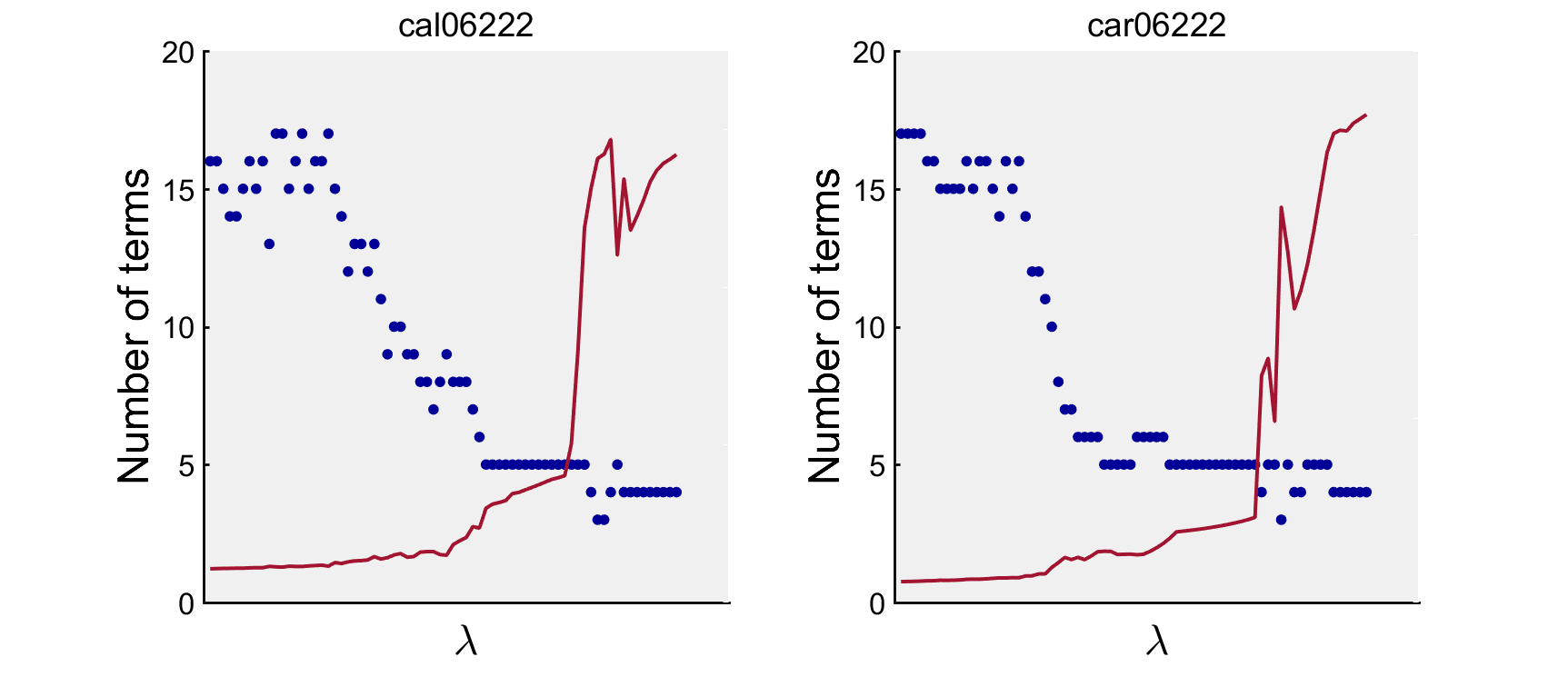}}
	\end{minipage}
	\hfill
	\begin{minipage}{.5\linewidth}
		\centerline{\includegraphics[width=8.8cm]{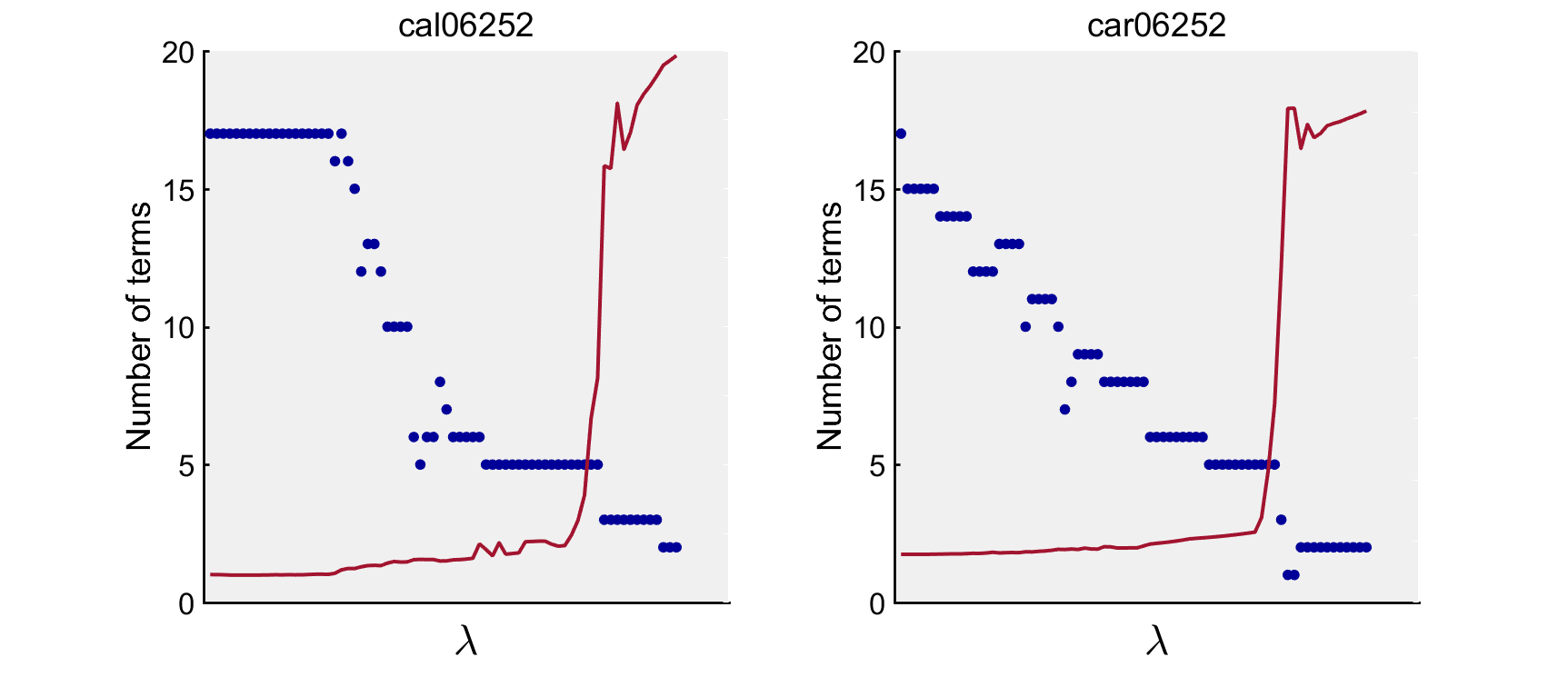}}
	\end{minipage}
	\caption{\textbf{The accelerations of fitting error due to change in the number of feature terms.} We empirically adjust the regulation parameter, $\lambda$, while observing the change of the fitting error. Three sets of data were chosen to test $\text{S}^3\text{d}$: car/cal06242, car/cal06222, and car/cal06252 from top to bottom. For each dataset, the fitting error slowly increased and tended to a relatively stable value (5 features in Eq.~\eqref{GLE} emerge), then that error began  to rapidly grow. In the table above, we show the selected basis functions for different $\lambda$.}
	\label{fig-lam}
\end{figure}

Thus, the structural analysis of the CGLE model further allows us to quantitatively explain the weakly nonlinear dynamics of TW convection. To further fine-tune the coefficients of CGLE, we apply the $\text{S}^3\text{d}$ algorithm to each data set in Table \ref{tab:results} with the discovered nonlinearities. To show our identified results in an accessible form, we introduce the concept of the leave-one-out in machine learning to discuss the best method to choose the regularization parameters $\lambda$ and obtain the best-fit coefficients. We accept the parameters that lead to the smallest fitting error on the test sets:
$err= \frac{\|Y-\Phi \omega\|^2 }{\|Y\|^2},$ by which, the identified coefficients are reported in Table~\ref{tab:results}. We can see that the identified results resemble the theoretical values and experimental ones \citep{Voss,paul2,paul3} and reference therein, except for some obscure parameters in theory and experiment.

\begin{table}[!h]
	\renewcommand\arraystretch{1.8}
	\caption{\textbf{Summary of the identified parameters of seven sets of real data:}
		The data are labeled according to ``right"(car) and ``left"(cal). Each row represents the identified coefficients
		by $\text{S}^3\text{d}$ .}\label{tab:results}
	\begin{center}
		\def\temptablewidth{1\textwidth}
		{\rule{\temptablewidth}{1pt}}
		\begin{tabular*}{\temptablewidth}{ccccccccccccccc}
			&&&&&\multicolumn{5}{c}{\bf Coefficients in the discovered complex Ginzburg-Landau equation}
		\end{tabular*}
		{\rule{\temptablewidth}{1pt}}
		\begin{scriptsize}
			\begin{tabular*}{\temptablewidth}{@{\extracolsep{\fill}}c||cccccccccccccc}
				\centering
				\tabincell{c}{\bf Experiment\\ \bf Label }&
				%\tabincell{c}{\bf fitting error}&
				\bm{$s$}&
				\bm{$ \varepsilon\tau_0^{-1} $} &
				\bm{$\xi_0^2\tau_0^{-1} $} &
				\bm{$g\tau_0^{-1}$}&
				\bm{$h\tau_0^{-1}   $}&
				\bm{$ \varepsilon\tau_0^{-1}c_0$}&
				\bm{$\xi_0^2c_1\tau_0^{-1}$}&
				\bm{$gc_2\tau_0^{-1}$}&
				\bm{$hc_3\tau_0^{-1}$}&
				\\
				\Xhline{0.8pt}
				\tabincell{c}{\bf cal06172cb}&
				
				0.5895&
				0.0065&
				0.3392&
				-17.1619&
				-6.4595&
				0.0443&
				-0.2185&
				-89.0851&
				-139.7500&
				
				\\
				\hline
				\tabincell{c}{\bf car06172cb}&
				
				0.5982&
				0.0079&
				0.9210&
				17.7800&
				-25.2965&
				0.0640&
				-0.1243&
				-206.3442&
				-150.6041&
				
				\\
				\hline
				\tabincell{c}{\bf cal06182cb}&
				
				0.6268&
				0.0057&
				0.0000&
				-78.0836&
				-0.8076&
				-0.0094&
				-0.3774&
				0.0000&
				-154.4844&

				\\
				\hline
				\tabincell{c}{\bf car06182cb}&
				
				0.6308&
				0.0081&
				0.0935&
				-90.2362&
				-35.0632&
				-0.0059&
				-0.4397&
				-55.8157&
				-157.3323&

				\\
				\hline
				\tabincell{c}{\bf cal06192cb}&
				
				0.5930&
				0.0054&
				0.0148&
				-139.4494&
				-2.5230&
				-0.0066&
				-0.6457&
				-37.6346&
				-152.7833&
				
				\\
				\hline
				\tabincell{c}{\bf car06192cb}&
				
				0.5916&
				0.0087&
				0.0158&
				-198.6799&
				-34.9115&
				-0.0066&
				-0.9189&
				-68.1182&
				-141.3105&
				
				\\
				\hline
				\tabincell{c}{\bf cal06212cb}&
				
				0.6063&
				0.0057&
				0.0639&
				-45.0095&
				-7.9602&
				0.0033&
				-0.2732&
				-35.3580&
				-154.5495&
				
				\\
				\hline
				\tabincell{c}{\bf car06212cb}&
				
				0.6131&
				0.0042&
				0.0202&
				-26.1553&
				-33.0555&
				0.0019&
				-0.1489&
				-17.7879&
				-168.7832&
				
				\\
				\hline
				\tabincell{c}{\bf cal06222cb}&
				
				0.5741&
				0.0173&
				1.7492&
				40.1657&
				-15.1882&
				0.0281&
				-0.0424&
				-379.5853&
				-136.5903&
				
				\\
				\hline
				\tabincell{c}{\bf car06222cb}&
				
				0.5771&
				0.0127&
				1.3000&
				29.9928&
				-16.6945&
				0.0114&
				-0.0894&
				-272.3181&
				-161.5083&
				
				\\
				\hline
				\tabincell{c}{\bf cal06242cb}&
				
				0.5509&
				0.0138&
				1.0701&
				16.2302&
				-17.3846&
				0.0555&
				-0.1505&
				-234.8376&
				-144.0718&
				
				\\
				\hline
				\tabincell{c}{\bf car06242cb} &
				
				0.5597&
				0.0150&
				1.1368&
				13.1359&
				-4.5687&
				0.0567&
				-0.2108&
				-237.9784&
				-149.1559&

				\\
				\hline
				\tabincell{c}{\bf cal06252cb}&
				
				0.5634&
				0.0199&
				1.1690&
				2.1061&
				-5.2084&
				0.0867&
				-0.2690&
				-267.3739&
				-169.0035&
				
				\\
				\hline
				\tabincell{c}{\bf car06252cb} &
				
				0.5481&
				0.0131&
				1.1075&
				22.3429&
				-3.9619&
				0.0794&
				-0.1245&
				-234.6182&
				-179.6777&

			\end{tabular*}
			{\rule{\temptablewidth}{1pt}}
		\end{scriptsize}
	\end{center}
\end{table}

Previous works \citep{Voss} developed nonlinear regression analysis using prior knowledge in physics to infer the CGLE model for the description of binary-fluid convection near onset. We employ the proposed $\text{S}^3\text{d}$ to discover CGLE from experimental data and lead to comparable parameters of the oscillatory behavior in the experimental system. The proposed method advances the theory without any prior knowledge in physics and can consistently discover the underlying mechanistic PDE from multiple datasets with different bifurcation parameters.

\subsubsection{Validation}
To further validate our method, we suggest simulating the identified CGLE to reconstruct the experimental chaos observed in the one-dimensional TW convection. We employ the Fourier spectral method \citep{spec} to simulate the identified CGLE model
\begin{eqnarray}\label{GLEi}
\begin{split}
\partial_t A_R+s\partial_x A_R&=\varepsilon\tau_0^{-1}(1+ic_0)A_R+\xi_0^2\tau_0^{-1}(1+ic_1)\partial_{x}^2A_R\\
&+g\tau_0^{-1}(1+ic_2)|A_R|^2A_R+h\tau_0^{-1}(1+ic_3)|A_L|^2A_R,~~~0<x<L,t>0,\\
\partial_t A_L-s\partial_x A_L&=\varepsilon\tau_0^{-1}(1+ic_0)A_L+\xi_0^2\tau_0^{-1}(1+ic_1)\partial_{x}^2A_L\\
&+g\tau_0^{-1}(1+ic_2)|A_L|^2A_L+h\tau_0^{-1}(1+ic_3)|A_R|^2A_L,~~~0<x<L,t>0,
\end{split}
\end{eqnarray}
where parameter values are given in Table~\ref{tab:results}. We consider system length of $L=82.0119$ and computational time of $T=625$ ($t_0=0, t_f=625$), which are determined based on a sampling interval with spatial step, $\triangle x=0.458167$, and time step, $\triangle t=1.5625$. We discretize the left-going complex wave amplitude, $A_L$, and the right-going complex wave amplitude, $A_R$, in space on a uniform $180$ mesh with spectral approximations of the spatial derivatives, and integrated in time using an explicit Runge-Kutta formula with $401$ time points.

Note, we conducted the experiments in an annular container where the TW system comprises periodic boundary conditions \citep{Voss}. This corresponds to the boundary condition:
$$
A_L(x,t)=A_L(x+L,t),A_R(x,t)=A_R(x+L,t).
$$
Here, we take any column of the snapshot matrix, $U$, as initial conditions. Unequivocally, we take the $622^{th}$ column of the snapshot matrix, $U$, as initial values for data labelled ``cal/car06172",  the $600^{th}$ column for  data labelled ``cal/car06192", the $600^{th}$ column for  data labelled ``cal/car06212", and the $1^{st}$ column for data labelled ``cal/car06182". Given the initial and boundary conditions, we obtain the discretized version of the solution, $A_L, A_R$, to Eq.~\eqref{GLEi} and plot their time evolution.

Fig.~\ref{fig-r} presents the comparison between the simulation and the experiments on TW convection. The simulations in the first three row of Fig.~\ref{fig-r} have a resolution of $180\times 400$ samples and the forth row has a resolution of $180\times 200$. Both resolutions correspond to reconstruction times $T=625$ and $T=312.5$, respectively. It is observed that the simulations agree well with the experiments. In addition, we also find that the reconstruction results are different under different initial values. In particular, we take  the $658^{th}$ column taken from data labelled ``cal/car06192" as the initial value. The regular bursts of left- and right-going TW are seen with the simulation time up to $T=1187$, and the amplitudes of wave approach the real values, as shown in Fig.~\ref{fig-all}.
Thus, the reconstruction from our identified CGLE model bears a striking resemblance to the experimentally observed nonlinear states of TW convection on binary-fluid convection.
\begin{figure}[!]
	\begin{minipage}{0.4\linewidth}
		\rightline{\includegraphics[width=4.6cm]{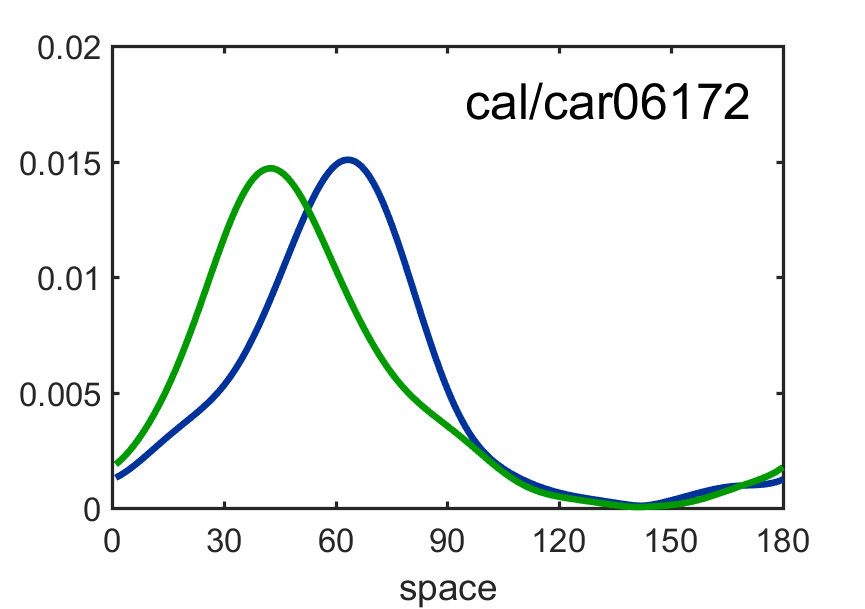}}
	\end{minipage}
	\hfill
	\begin{minipage}{.6\linewidth}
		\leftline{\includegraphics[width=9cm]{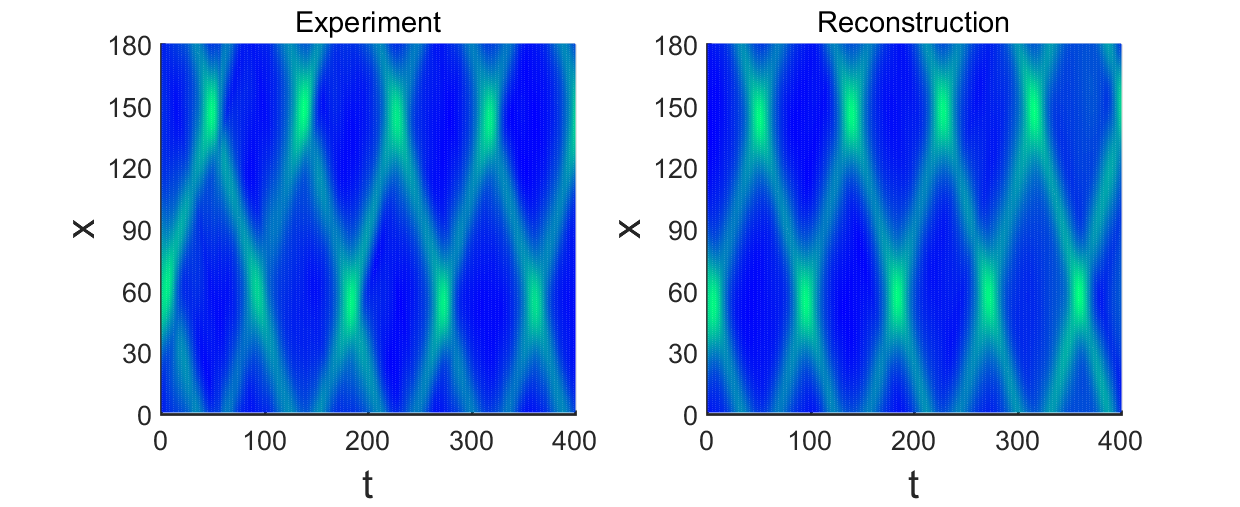}}
	\end{minipage}

	\begin{minipage}{0.4\linewidth}
		\rightline{\includegraphics[width=4.6cm]{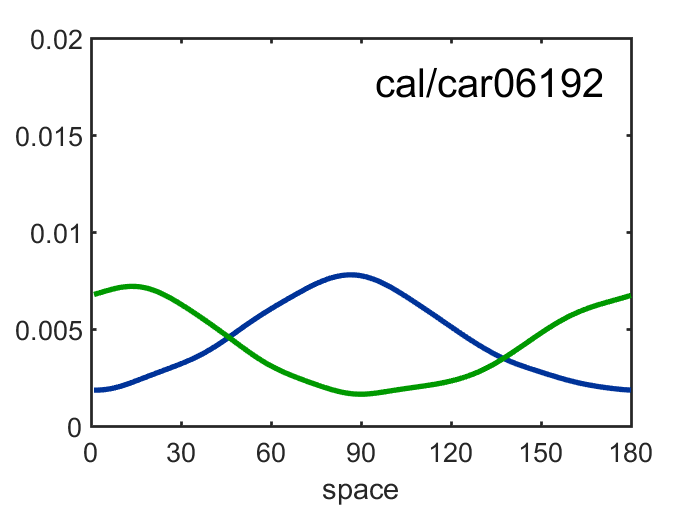}}
	\end{minipage}
	\hfill
	\begin{minipage}{.6\linewidth}
		\leftline{\includegraphics[width=9cm]{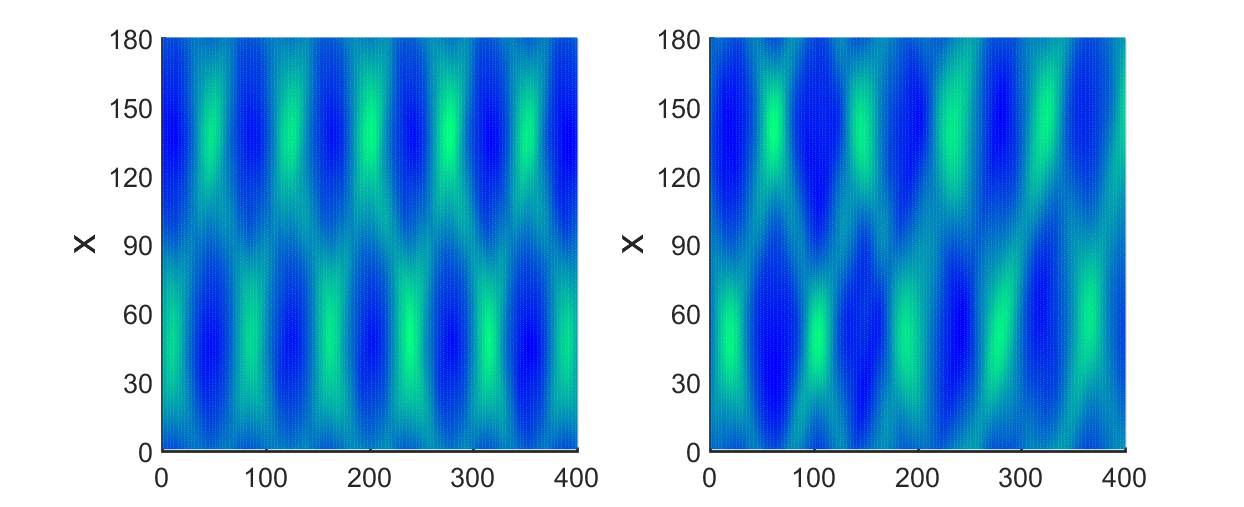}}
	\end{minipage}
	
	\begin{minipage}{0.4\linewidth}
		\rightline{\includegraphics[width=4.6cm]{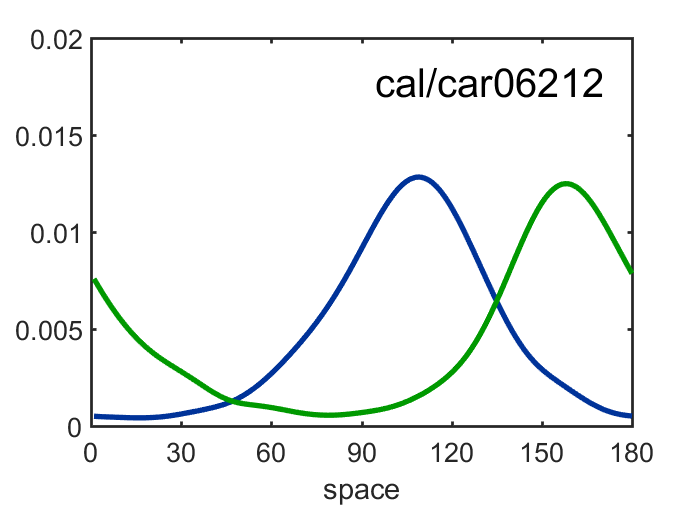}}
	\end{minipage}
	\hfill
	\begin{minipage}{.6\linewidth}
		\leftline{\includegraphics[width=9cm]{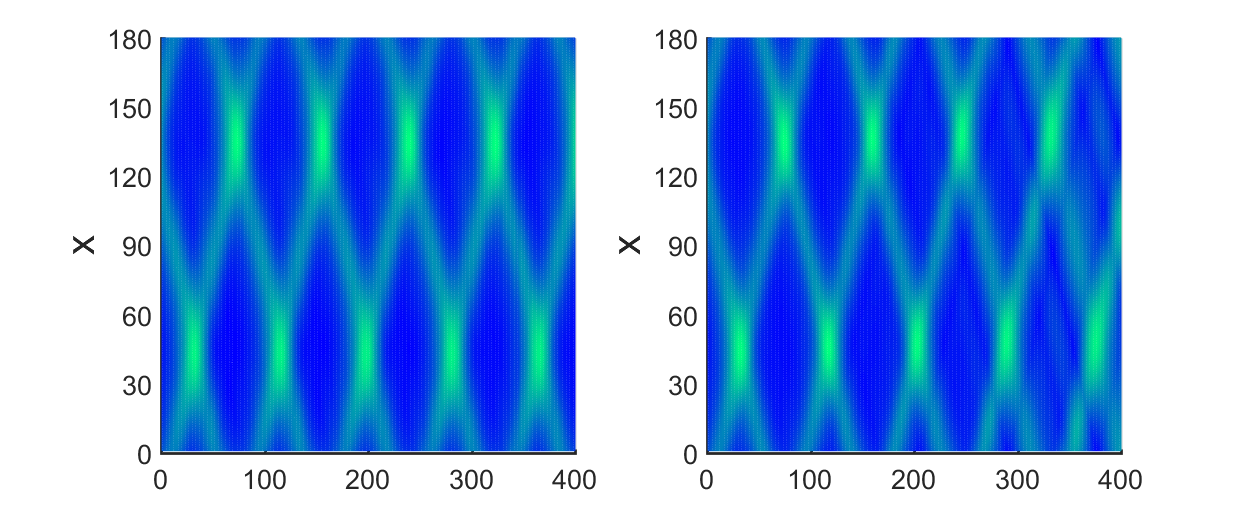}}
	\end{minipage}

	\begin{minipage}{0.4\linewidth}
		\rightline{\includegraphics[width=4.6cm]{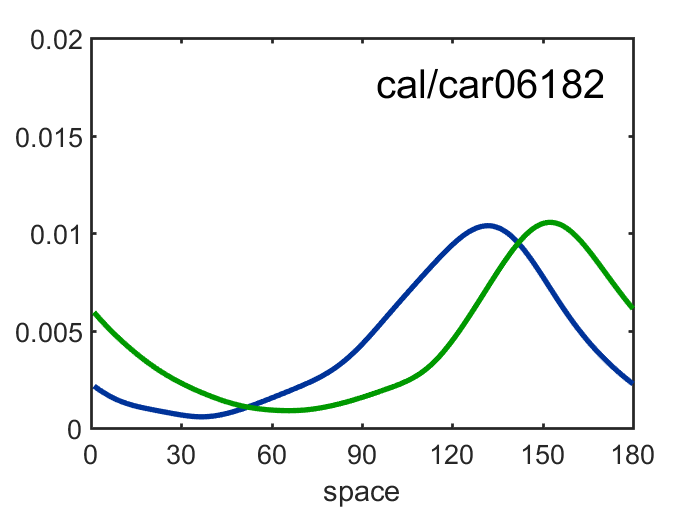}}
	\end{minipage}
	\hfill
	\begin{minipage}{.6\linewidth}
		\leftline{\includegraphics[width=9cm]{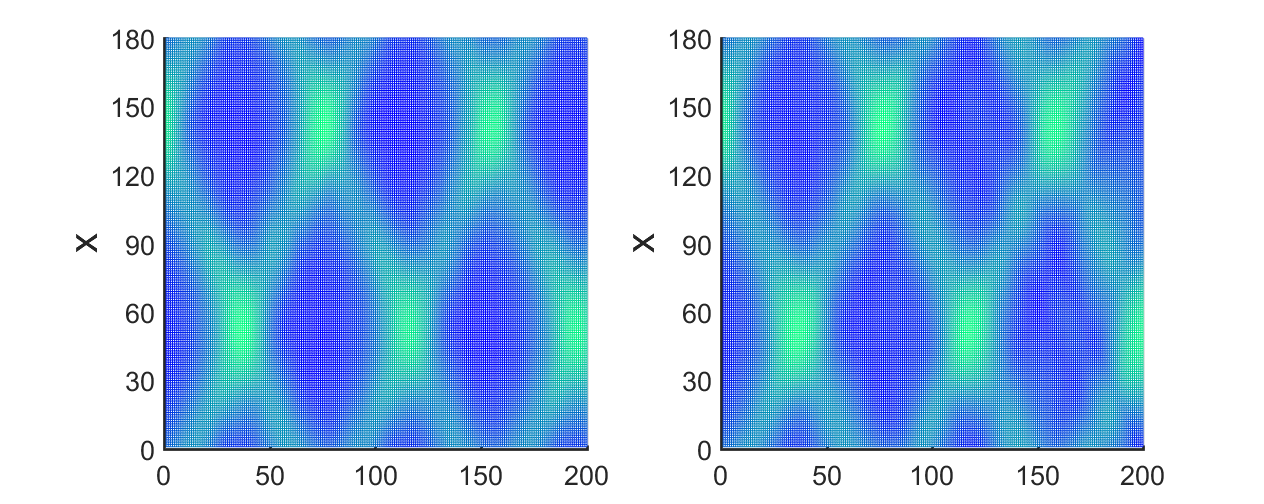}}
	\end{minipage}

	\caption{\textbf{Comparison of spatiotemporal evolution of the superposition of the left- and right- TW amplitude between the simulation and the experiment.}
		The different initial values: the $622^{th}$ column of the snapshot matrix $U$ for data labelled ``cal/car06172",  the $600^{th}$ column for data labelled ``cal/car06192", the $600^{th}$ column for  data labelled ``cal/car06212" and the $1^{st}$ column for data labelled ``cal/car06182", are selected for simulation (the first column). The first 400 (or 200) snapshots of the real data and the simulation data are shown with a resolution of $180\times 400$ (column 2 and column 3). }
	\label{fig-r}
\end{figure}

\begin{figure}[!]
	%\begin{minipage}{0.4\linewidth}
	% \rightline{\includegraphics[width=4.6cm]{182i.png}}
	%\end{minipage}
	%\hfill
	%\begin{minipage}{.6\linewidth}
	%\leftline{\includegraphics[width=9cm]{182r.png}}
	%\end{minipage}
	
	%\begin{minipage}{0.4\linewidth}
	% \rightline{\includegraphics[width=4.6cm]{192i658.png}}
	%\end{minipage}
	%\hfill
	\begin{minipage}{.6\linewidth}
		\leftline{\includegraphics[width=15cm]{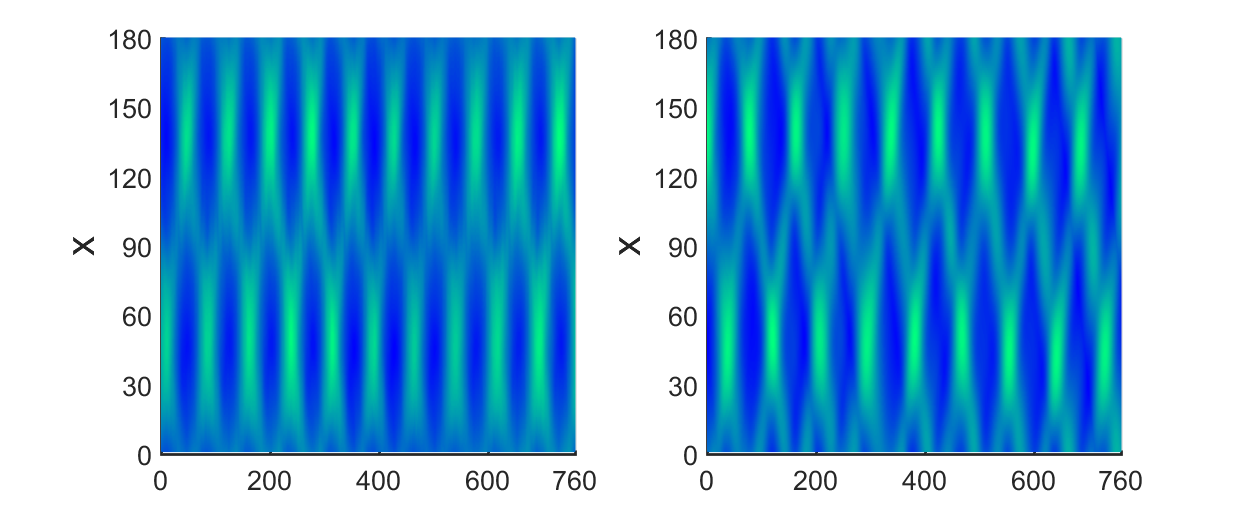}}
	\end{minipage}
	
	\caption{\textbf{Comparison of spatiotemporal evolution of the superposition of the left- and right- TW amplitude between the simulation and the experiment.}
		Using the extended time $T=1187$ and initial value (i.e., $658^{th}$column of snapshot matrix), a regular busts of TW is seen . }
	\label{fig-all}
\end{figure}

\subsection{Comparison of $\text{S}^3\text{d}$, STRidge and  Douglas-Rachford}\label{sec:comparison}
We first compare the $\text{S}^3\text{d}$ method with the STRidge regression algorithm or PDE functional identification of nonlinear dynamics (PDE-FIND) proposed in \citep{Rudy}. We consider the original dataset in \citep{Rudy}. These dataset are generated by simulating the KdV equation with single- and double- soliton solution, the Burger's equation, the Quantum Harmonic Oscillator and the nonlinear Schr\"{o}dinger equation. Here we neglect the simulation details and refer readers to \citep{Rudy}.  Table \ref{Tab:com-1} shows the identified results of applying the  $\text{S}^3\text{d}$ method to the dataset. We use a much smaller number of samples for a better discovery of these PDEs in terms of smaller parametric error.
We further investigate the impact of noise on both methods, $\text{S}^3\text{d}$ and PDE-FIND, using data generated from the Quantum Harmonic Oscillator  and the nonlinear Schr\"{o}dinger equation \citep{Rudy}. In particular, we consider data with  1\% Gaussian noise, as shown in Table~\ref{Tab:com-2}. $\text{S}^3\text{d}$ is able to discover both equations from such datasets.
\begin{table}[!hb]
	\caption{Comparison of identified results  using
		PDE-FIND  method and $\text{S}^3\text{d}$ method with the original database in \citep{Rudy}.}\label{Tab:com-1}
	\centering
	\begin{footnotesize}
		\begin{tabular}{lllllllllllllllll}
			
			\Xhline{0.5pt}
			\hline
			\multicolumn{3}{c}{\bf Example}  &\multicolumn{2}{c}{ \bf $\text{S}^3\text{d} $ \bf method }&\multicolumn{2}{c}{\bf PDE-FIND method }  \\
			\hline
			\bf Terms & \multicolumn{2}{c}{ \bf Identified results } & \bf  \tabincell{c}{No. of\\ samples} & \tabincell{c}{\bf MSE(err)-STD(err)} &\bf  \tabincell{c}{No. of\\ samples} & \tabincell{c}{\bf MSE(err)-STD(err)} \\
			\hline
			
			\multicolumn{7}{l}{\bf Korteweg-de ~Vries (Single-soliton)}\\
			
			\multicolumn{3}{c}{\tabincell{c}{$u_t=-4.9892u_{x}(c=5)$\\$ u_t=-1.0004u_{x}(c=1)$}} & 2560 & \tabincell{c}{ 0.2159 \%$\pm$0.0\%\\0.0375\%$\pm$0.0\%}
			& 12800 &\tabincell{c}{ 0.3745 \%$\pm$0.0\%\\0.0820\%$\pm$0.0\%}\\
			\multicolumn{3}{c}{\tabincell{c}{$u_t=-6.0552uu_{x}-1.0312u_{xxx}$}} &5120  & \tabincell{c}{  2.0221\%$\pm$1.558\%}  & 25600 &\tabincell{c}{  2.5931\%$\pm$1.3601\%}\\
			
			\multicolumn{7}{l}{\bf Korteweg-de ~Vries (Double-soliton)}\\
			
			\multicolumn{3}{c}{\tabincell{c}{$u_t=-5.9853uu_{x}-0.9978u_{xxx}$}} & 9000 & 0.2339\%$\pm$0.0151\% & 102912  & 0.9572\%$\pm$0.2322\% \\
			\multicolumn{3}{c}{\tabincell{c}{$u_t=-5.9250uu_{x}-0.9867u_{xxx}$}} & 16600  & 1.2894\%$\pm$0.0548\%  & 102912 & 7.4727\%$\pm$4.9306\%\\
			
			\multicolumn{7}{l}{\bf Burgers'~ equation}\\
			
			\multicolumn{3}{c}{\tabincell{c}{$u_t=-0.9999uu_x+ 0.1000u_{xx}$}} & 2000  & 0.0051\%$\pm$0.0054\% & 25856  & 0.1595\%$\pm$0.0608\%\\
			\multicolumn{3}{c}{\tabincell{c}{$u_t=-1.0010uu_x+0.1001u_{xx}$}} & 5000  & 0.0760\%$\pm$ 0.0366\%& 25856  & 1.9655\%$\pm$1.0000\%\\
			
			\multicolumn{7}{l}{\bf Quantum~ Harmonic~ Oscillator}\\
			
			\multicolumn{3}{c}{\tabincell{c}{$u_t=0.5001iu_{xx}-0.9999i\frac{x^2}{2}u$}}  & 2000  & 0.0117\%$\pm$0.0055\%  & 205312  & 0.2486\%$\pm$0.0128\%\\
			\multicolumn{3}{c}{\tabincell{c}{$u_t=0.4996iu_{xx}-1.0005i\frac{x^2}{2}u$}} & 2000 & 0.0685\%$\pm$0.0273\%  & 205312   & 9.6889\%$\pm$6.9705\%\\
			
			\multicolumn{7}{l}{\bf Nonlinear~ Schr\"{o}dinger~ equation}\\
			
			\multicolumn{3}{c}{\tabincell{c}{$u_t=0.4999iu_{xx}+0.9999i|u|^2u$}} & 7000  & 0.0199\%$\pm$0.0090\% & 256512&  0.0473\%$\pm$0.0147\%\\
			\multicolumn{3}{c}{\tabincell{c}{$u_t=0.4997iu_{xx}+0.9961i|u|^2u$}} & 7200 & 0.2250\%$\pm$0.2383\%& 256512  &3.0546\%$\pm$1.2193\%\\
			
			\hline
			\Xhline{0.5pt}
		\end{tabular}
	\end{footnotesize}
	
\end{table}

\begin{table}[htbp]
	\begin{center}
		\caption[MAhis]{Discovering the Schr\"{o}dinger~ equation and the nonlinear Schr\"{o}dinger equation from  local dynamics using $\text{S}^3\text{d}$.}\label{Tab:com-2}
		\begin{tabular}{@{}>{\sf }lllll@{}}
			\midrule
			\multicolumn{1}{>{\columncolor{mypink}}l}{Data~1(QH)} & \mgape{\includegraphics[scale=.15]{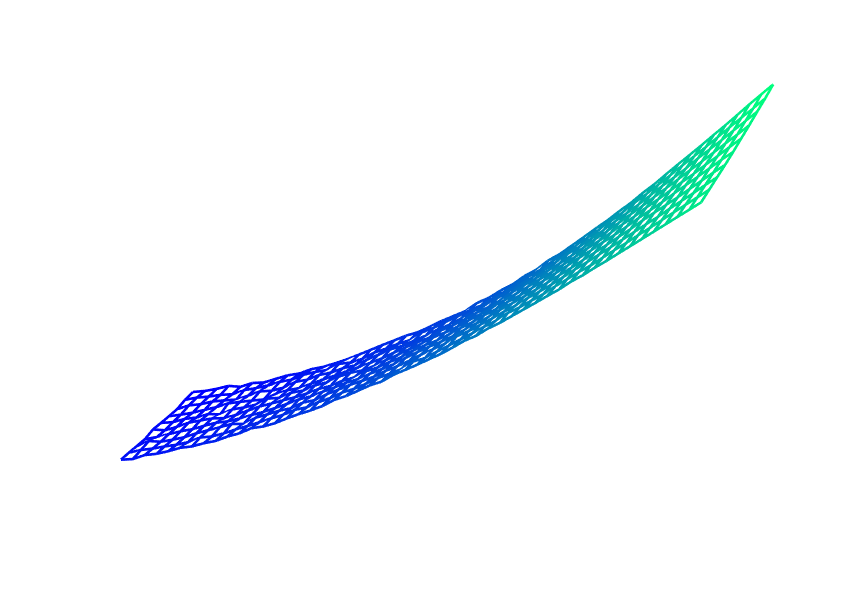}} & 500
			& \footnotesize \tabincell{c}{$u_t(x,t)=0.4732iu_{xx}(x,t)-0.9451i\frac{x^2}{2}u(x,t)$} \\
			\multicolumn{1}{>{\columncolor{mycyan}}l}{Data~2(NLS)} & \mgape{\includegraphics[scale=.15]{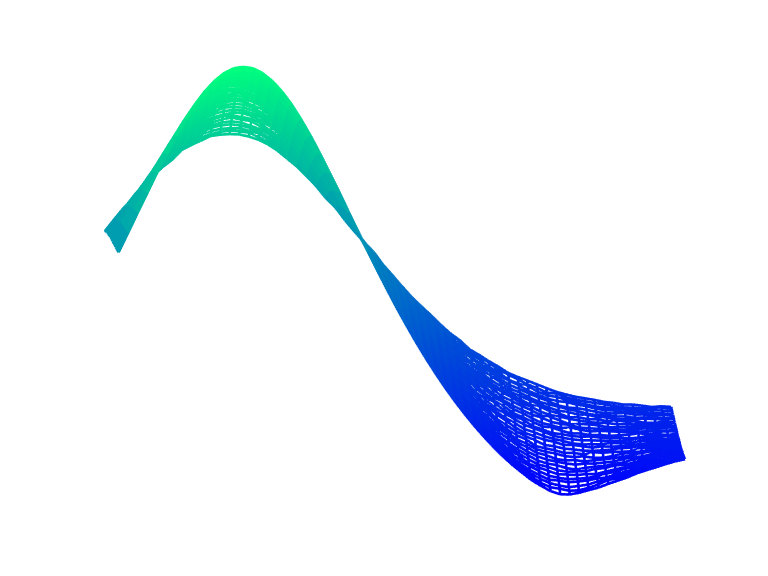}} & 1800
			& \footnotesize  \tabincell{c}{$u_t(x,t)=0.5005iu_{xx}(x,t)+0.9958iu|u|^2(x,t)$} \\
			\bottomrule
		\end{tabular}
	\end{center}
\end{table}

We then apply the PDE-FIND algorithm to Data 1(QH) in Table \ref{Tab:com-2}. For PDE-FIND, we perform a two-stage grid search for the
regularization parameter $\lambda$ and tolerance value $d_{tol}$. In the first
stage, we search through a coarse grid from $\lambda=10^{-8}$ to $10^{-1}$ and $d_{tol}=10^{-5}$ to $8.5\times 10^3$, such that
a total of $1152$ sets of parameters are tested for the PDE identification in order to make a fair comparison. However, the PDE-FIND is not able to discover the QH equation for any combinations: either it discovers the true terms $u_{xx}$ and $\frac{x^2}{2}u$  together with a few redundant terms or fails to discover the true terms. To further demonstrate the results, in the second stage, we perform a more fine-grained
search in the vicinity of that tuple ($\lambda,d_{tol}$) that gives
the most closest results. More specifically, the regularization parameter $\lambda$ varies now from $10^{-4}$ to $0.0864$ with the step $10^{-4}$, such that a total of $124,352$ sets of parameters are picked. By this, the PDE-FIND method still fail in identifying the QH equation.

For nonlinear Schr\"{o}dinger equation, the same two-stage grid search are used for the identification of the PDE using PDE-FIND. A total of $14,256$ sets of parameters were exhaustively tested for the discovery. However, the PDE-FIND fails in discovering the nonlinear Schr\"{o}dinger equation (the code used for comparison was in the Github repository).

Now, we compare our method with the work in \citep{Schae2017}, which applies the Douglas-Rachford algorithm \citep{Lions1979,Combettes2011}  to learn PDEs from the data generated by the viscous Burgers' equation, inviscid Burgers' equation, the Swift-Hohenberg equation, the Cahn-Hilliard equation, et al. However, the method makes an unrealistic assumption that only time derivatives are corrupted by additive Gaussian noise (see the section ``Simulation and numerical experiments" in \citep{Schae2017}).  Because the measured variables are always associated with noise, it will not only lead to variations in the time-derivatives but also to higher-order derivatives. Such variations will further propagate through the later identification process, which significantly increases the difficulty in identifying PDEs. The proposed  $\text{S}^3\text{d}$ method weakens their assumption by allowing adding Gaussian noise to the measured variables. As a result, both time and spatial derivatives are corrupted by noise. Indeed, this is a major challenge in the discovery process.

We compared the identification performance to the datasets generated by the NLS equation and the QH equation.
All these equations were discovered by our proposed method.
Note that there exist four tunable parameters in the Douglas-Rachford algorithm (ours has one tuning parameter): $\gamma$, $\mu$, the maximum number of iterations (MaxIt), balancing parameter $\lambda$.
Starting with $\gamma=0.1$ , $\mu=0.1$, $\lambda=10^{-3}$ and  $\text{MaxIt}=5000$ (which are picked according to examples in their papers), a grid search method was used to find its best set of parameters at reasonable ranges of these parameters. In total, we checked a total of 26244 sets of parameters for the PDE identification, which took more than 24 hours to complete. The algorithm failed to identify the correct features of the NLS equation, and gave many redundant terms across all searches. Similarly, we observed a similar failure of PDE discovery for the QH equation. In summary, the proposed Douglas-Rachford algorithm failed to discover those equations with the same limited and noisy data that was used (successfully) in $\text{S}^3\text{d}$.

\section{Conclusion and Discussion}\label{sec:conclusion}

In summary, our proposed $\text{S}^3\text{d}$ algorithm discovers the dynamics underlying a state of traveling-wave convection and many canonical PDEs from measured spatiotemporal data. The merit of the proposed method is its ability to freely construct a model class with candidate functions and automatically select the key ones that reproduce the observed spatiotemporal patterns. Benefiting from sparsity, the inferred PDEs are parsimonious
and accurate, enabling interpretability. One of the most important details of applying the $\text{S}^3\text{d}$ method is the approach for estimating the derivatives. In our work, we use the finite difference method, the polynomial
interpolation method and the spectral method. In the various examples, we observe that we are additionally able to robustly handle noise in the measurement data with polynomial
interpolation and the Pade scheme.

Our proposed $\text{S}^3\text{d}$ method exhibits the ability to extract the governing amplitude equations solely from high-quality experimental data.  Although the dynamics we refer to in this work are Eulerian dynamics described by PDEs, $\text{S}^3\text{d}$, as a general method, is also able to infer ODEs and static functional relations using datasets in \citep{Schmidt,Rudy}. $\text{S}^3\text{d}$ unifies results for the discovery of natural laws \citep{Schmidt,Bon2007}. Future work will focus on two important
aspects of extending the method for a wider range of practical applications. First, the current proposed method relies on estimating temporal and spatial derivatives of the measured state variables from data; there are cases that the system variables are not necessarily observable. Second, the current method can reconstruct equations that are linear with respect to the parameters. We are extending this proposed method to cases that are nonlinear with the parameters.

We expect the $\text{S}^3\text{d}$ method to be useful for the modeling of spatiotemporal dynamics from experimental data. This framework, as demonstrated through numerous examples, could potentially accelerate the discovery of new laws and furthermore stimulate physical explanations for the discovered equations, which lead to the discovery of the underlying mechanisms.

\section*{Acknowledgments}

We wish to thank Dr. Paul Kolodner for useful discussion and allowing us to use the experimental data. All synthetic data and codes used in this manuscript are publicly available on GitHub https://github.com/HAIRLAB/S3d.

\appendix

\section{Polynomial approximation}\label{sec:Polyapp}
For noise-contaminated data, polynomial approximation is a better choice to alleviate effects due to noise. With sampled data, $u(x_j, t_i)$, at time, $t_i$, with $j=1, \ldots, n_x$, we construct an approximation of the $q$-order derivative $\frac{\partial^q u(x,t)}{\partial x^q}$ by selecting the following sequence of polynomials of degree $p\in \mathbb{N}^+$ with $q<p$:
\begin{eqnarray*}
	L_p(x)=a_0+a_1x+a_2x^2+\cdots+a_px^p,
\end{eqnarray*}
subject to $L_p(x_j)=u(x_j, t_i)$. For example,
\begin{eqnarray}\label{Vande}
&&a_0+a_1x_1+a_2x_1^2+\cdots+a_px_1^p=u(x_1,t_i),\cr
&&a_0+a_1x_2+a_2x_2^2+\cdots+a_px_2^p=u(x_2,t_i),\cr
&&\cdots\cdots\\
&&a_0+a_1x_{n_x}+a_2x_{n_x}^2+\cdots+a_px_{n_x}^p=u(x_{n_x},t_i).\nonumber
\end{eqnarray}
Further, we write the above Eq.~\eqref{Vande} into matrix form and solve for parameters using a QR factorization. Then, with $u(x, t_i) = L_p(x)$, we compute the $q$th-order derivative, $\frac{\partial^q u(x,t)}{\partial x^q}$. We demonstrate in our experiment that the error values caused by noise can be removed using polynomial approximation, leading a closer-to-real estimation of the derivative.

\section{Proper Orthogonal Decomposition} \label{appendix:POD}
Data preprocessing is a key sub-step in $\text{S}^3\text{d}$ that removes noise in our measured datasets. There are many established approaches for rejecting noise in a set of data. For example, one can employ damping or thresholding \citep{Don, Ela, Bua}. The Proper Orthogonal Decomposition (POD) \citep{Bek, Sir}, originally a compression method, is a data processing algorithm that extracts coherent structures with a single temporal frequency from a numerical or experimental data-sequence. In this work, we instead choose to use POD for de-noising our datasets since the extracted structures, called the POD modes, do not contain directions with small variance including small information signals or signals related to noise.

We determine the POD modes from the snapshot matrix, $U$, by minimizing the Frobenius norm of the difference between $U$ and $\Psi A$,
\begin{eqnarray*}
	\underset{\Psi, A}{\min}\frac{1}{2}\|U-\Psi A\|^2_{F} ~~ s.t.~ <\psi_i, \psi_j>=\delta_{ij}
\end{eqnarray*}
where each column of $\Psi$ is a POD mode and $A$ is the coefficient matrix. We solve this minimization by using the singular value decomposition (SVD) on the empirical correlation matrix $C=UU^T$, yielding:
\begin{itemize}
	\item the $n_x\times p$ matrix comprised of the POD modes: $\Psi$;
	\item the $p\times p$ diagonal matrix with $p$ eigenvalues, $\{\lambda_i\}_{i=1}^p$: $\Sigma$; and
	\item the $p\times n_t$ matrix whose rows represent the POD coefficients for each POD mode: $A$;
\end{itemize}
Then, we capture the optimal POD modes with a sufficiently high threshold. In our work, we fix our threshold at $0.9999$, allowing us to choose integer, $r\leq p$, with
\begin{eqnarray*}
	\sum\limits_{i=1}^{r}\lambda_i^2/\sum\limits_{i=1}^{p}\lambda_i^2\geq 0.9999.
\end{eqnarray*}
where $\lambda_{i+1}\leq \lambda_{i}$. We obtain the new, filtered snapshot matrix, $\tilde{\mathbf{U}}$, with $\tilde{\mathbf{U}}=\Psi A$, with $\Psi$ taking on the first $r$ columns and $A$ taking on the first $r$ rows. In the previously stated minimization problem, we compare the error between the new snapshot matrix,$\tilde{\mathbf{U}}$, and the original snapshot matrix, $U$, and choose to use $\tilde{\mathbf{U}}\in \mathbb{R}^{n_x\times n_t}$ instead of $U$ in the remaining steps.

\section{Reduce Computational Burden}\label{appendix:reduce}
Our first strategy is to identify the dynamics from subsampled data instead of the full data. Let $\mathbf{U}_s$ be a small portion of the total snapshot matrix $\mathbf{U}$. For example, we can choose the value of solution $u(x,t)$ at all space points and in time interval $[t_k, t_{k+p}]$ such that $\mathbf{U}_s\in \mathbb{R}^{n_x\times p}$. This subsampling is similar to constructing a measurement matrix $C\in \mathbb{R}^{n_xp\times n_xn_t}$ for $y$ in Eq.~\eqref{regre}, which results in linear equations similar to Eq.~\eqref{regre}
\begin{eqnarray}\label{Re1}
Cy=C\Phi \theta\rightarrow y=\Phi \theta
\end{eqnarray}
where $y\in \mathbb{R}^{n_xp\times 1}$, $\Phi \in \mathbb{R}^{n_xp\times m}$, and $\theta \in \mathbb{R}^{m\times 1}$ with $m$ representing the total number of candidate terms. Then, we solve the linear equations with the proposed algorithm in the case $mp \ll mn$.

Our next strategy is to reduce the dimension of our problem. Whether $\Phi\in \mathbb{R}^{n_xn_t\times m}$ in Eq.~\eqref{regre} or $\Phi\in \mathbb{R}^{n_xp\times m}$ in Eq.~\eqref{Re1}, we can use SVD to obtain a set of transform basis. Specifically, using the economy-size SVD of $\Phi$ in Eq.~\eqref{Re1}, we get
\begin{eqnarray*}
	\Phi
	&=&\underset{\text{Transform ~~basis}}{\underbrace{\left[ \begin{array}{cccccccc}
				\psi_1&\psi_2&\cdots&\psi_r
			\end{array}
			\right]}}
	\underset{\text{Truncated~singular~value}}{\underbrace{ \left[ \begin{array}{cccccccc}
				\lambda_1&&\\
				&\lambda_2&&\\
				&&\ddots&\\
				&&&\lambda_r\\
			\end{array} \right]}}
	\left[ \begin{array}{cccccccc}
		A_1\\
		A_2\\
		\vdots\\
		A_r\\
	\end{array}
	\right].
\end{eqnarray*}
Then, we write the optimal low-dimensional representation as
\begin{eqnarray*}
	\bar{y}=\bar{\Psi}\theta,
\end{eqnarray*}
where $\bar{y}=\Psi^T y\in \mathbb{R}^{r\times 1}$, $\bar{\Psi}=\Psi^T \Phi\in \mathbb{R}^{r\times m}$ with $r\ll n_xp$. The proposed dimensionality reduction method is especially useful in applications where there is a large number of measurements for the sparse identification algorithm.

Depending on the size of the data, either or both the above methods are selected for reducing computational burden from high dimensionality.

\section{Systems with Complex Variables}\label{appendix:COM}
In certain applications, we consider a standard matrix form, $\tilde{y}=\tilde{\Phi}\tilde{\theta}$, where $\tilde{y}, \tilde{\theta}$ are complex vectors and $\tilde{\Phi}$ is a complex matrix. Given a complex vector $\tilde{y}$, we consider the following one-to-one mapping to a real vector $y$:
\begin{equation}\label{eq:y}
y=\begin{bmatrix} \text{Re}~\tilde{y}  \cr
\text{Im}~\tilde{y} \nonumber
\end{bmatrix}.
\end{equation}
Given complex matrix, $\tilde{\Phi}$, we consider the following construction of real matrix $\Phi$:
\begin{equation}\label{eq:M}
\Phi=\begin{bmatrix} \text{Re}(\tilde{\Phi}) & -\text{Im}(\tilde{\Phi}) \cr
\text{Im}(\tilde{\Phi}) & \text{Re}(\tilde{\Phi})\nonumber
\end{bmatrix}.
\end{equation}
With the above mapping and matrix construction, we are able to write that
\begin{equation}
\tilde{y}=\tilde{\Phi}\tilde{\theta} \Leftrightarrow  y=\Phi\theta.\nonumber
\end{equation}
Once $\theta$ is determined, we can obtain $\tilde{\theta}$. Thus, we can relate the discovery of systems with complex variables with the discovery of systems with real variables and are able to use the solution presented for the discovery of systems with real variables.

\section{Proofs of Lemmas}\label{appendix:proof}
{\begin{proof}[Proof of Lemma~\ref{lem2}]
By the definition of $\widehat{\mathcal{L}}$, and the items (1), (2), (3), (4), $\lim \limits_{k\rightarrow \infty}\widehat{\mathcal{L}}(\bar{\theta}^{(k)},$ $\bar{\gamma}^{(k)};\theta^{(k)},\gamma^{(k)})$ exists, and therefore the sequence $\big{\{}\widehat{\mathcal{L}}(\bar{\theta}^{(k)},\bar{\gamma}^{(k)};\theta^{(k)},\gamma^{(k)})\big{\}}_{k=0}^{\infty}$ is bounded.

 We first show that
\begin{equation}\label{l2}
  \lim \limits_{k\rightarrow \infty}\widehat{\mathcal{L}}(\bar{\theta}^{(k)},\bar{\gamma}^{(k)};\theta^{(k)},\gamma^{(k)})=\widehat{\mathcal{L}}(\theta^{\ast\ast},\gamma^{\ast\ast};\theta^{\ast},\gamma^{\ast}).
\end{equation}
To this end, it suffices to show that
\begin{equation}\label{l3}
 \lim \limits_{k\rightarrow \infty}\widehat{\mathcal{L}}(\bar{\theta}^{(k)},\bar{\gamma}^{(k)};\theta^{\ast},\gamma^{\ast})=\widehat{\mathcal{L}}(\theta^{\ast\ast},\gamma^{\ast\ast};\theta^{\ast},\gamma^{\ast}).
\end{equation}
% Note though that, $\widehat{L}(\theta,\gamma;\theta^{\ast},\gamma^{\ast})$ is undefined on $\Omega_2$ and hence we supplement its definition just as the supplementary definition of $\widehat{L}(\theta,\gamma;\theta^{\ast},\gamma^{\ast})$.
 If $(\theta^{\ast\ast},\gamma^{\ast\ast})\in \mathrm{int}(\Omega)$, then one has from the continuity of $\widehat{\mathcal{L}}(\theta,\gamma;\theta^{\ast},\gamma^{\ast})$ on $\mathrm{int}(\Omega)$ that \eqref{l3} holds; If $(\theta^{\ast\ast},\gamma^{\ast\ast})\in \Omega_2$, then by the supplementary definition of $\widehat{\mathcal{L}}(\theta,\gamma;\theta^{\ast},\gamma^{\ast})$ on $\Omega_2$, \eqref{l3} holds. Thus, \eqref{l2} holds.

  Next, we show \eqref{l1}. To complete the proof let us assume that $(\theta^{\ast\ast},\gamma^{\ast\ast})\not \in \mathcal{A}(\theta^{\ast},\gamma^{\ast})$ and establish a contradiction. Since $\widehat{\mathcal{L}}(\theta,\gamma;\theta^{\ast},\gamma^*)$ is convex,
  $\mathcal{A}(\theta^\ast,\gamma^\ast)$ is nonempty. Let $(\tilde{\theta},\tilde{\gamma})$ be any point in $\mathcal{A}(\theta^\ast,\gamma^\ast)$. Then, one has $$\widehat{\mathcal{L}}(\tilde{\theta},\tilde{\gamma};\theta^{\ast},\gamma^{\ast}) <\widehat{\mathcal{L}}(\theta^{\ast\ast},\gamma^{\ast\ast};\theta^{\ast},\gamma^{\ast}).$$
Select a sufficiently small positive number $\varepsilon$ such that $$2\varepsilon<\widehat{\mathcal{L}}(\theta^{\ast\ast},\gamma^{\ast\ast};\theta^{\ast},\gamma^{\ast})-\widehat{\mathcal{L}}(\tilde{\theta},\tilde{\gamma};\theta^{\ast},\gamma^{\ast}).$$
By \eqref{l2}, there exists a positive number $k_1$ such that
 \[
\widehat{\mathcal{L}}(\bar{\theta}^{(k)},\bar{\gamma}^{(k)};\theta^{(k)},\gamma^{(k)})>  \widehat{\mathcal{L}}(\theta^{\ast\ast},\gamma^{\ast\ast};\theta^{\ast},\gamma^{\ast})-\varepsilon,
 \]
 for any $k>k_1$. By the definition of $\widehat{\mathcal{L}}(\tilde{\theta},\tilde{\gamma};\theta,\gamma)$, we know that $\widehat{\mathcal{L}}(\tilde{\theta},\tilde{\gamma};\theta,\gamma)$ with respect to $(\theta,\gamma)$ is continuous on $\Omega$. Then, there exists a positive number $k_2$ such that
 \[
 \widehat{\mathcal{L}}(\tilde{\theta},\tilde{\gamma};\theta^{(k)},\gamma^{(k)})<  \widehat{\mathcal{L}}(\tilde{\theta},\tilde{\gamma};\theta^{\ast},\gamma^{\ast})+\varepsilon,
 \]
  for any $k>k_2$. Putting together these pieces above yields
  \[
\widehat{\mathcal{L}}(\bar{\theta}^{(k)},\bar{\gamma}^{(k)};\theta^{(k)},\gamma^{(k)})>\widehat{\mathcal{L}}(\tilde{\theta},\tilde{\gamma};\theta^{(k)},\gamma^{(k)}),
  \]
 for any $k>\max\{k_1,k_2\}$. However, we have from the item (5) that
 \[
 \widehat{\mathcal{L}}(\bar{\theta}^{(k)},\bar{\gamma}^{(k)};\theta^{(k)},\gamma^{(k)})\leq \widehat{\mathcal{L}}(\theta,\gamma;\theta^{(k)},\gamma^{(k)}),
 \]
 for any $(\theta,\gamma)\in \Omega$. It is a contradiction, and hence, \eqref{l1} holds.
\end{proof}
}

\bibliography{xtref}

\end{document}